\documentclass{article} 
\usepackage{iclr2025_conference,times}

\usepackage{amsmath,amsfonts,bm}

\def\eqref#1{equation~\ref{#1}}

\def\1{\bm{1}}

\def\eps{{\epsilon}}

\def\rmA{{\mathbf{A}}}
\def\rmB{{\mathbf{B}}}
\def\rmC{{\mathbf{C}}}

\def\rmF{{\mathbf{F}}}
\def\rmG{{\mathbf{G}}}
\def\rmH{{\mathbf{H}}}
\def\rmI{{\mathbf{I}}}

\def\rmO{{\mathbf{O}}}

\def\rmQ{{\mathbf{Q}}}

\def\rmS{{\mathbf{S}}}

\def\rmU{{\mathbf{U}}}
\def\rmV{{\mathbf{V}}}
\def\rmW{{\mathbf{W}}}

\def\rmY{{\mathbf{Y}}}
\def\rmZ{{\mathbf{Z}}}

\def\vmu{{\bm{\mu}}}

\def\vnu{{\bm{\nu}}}

\def\vb{{\bm{b}}}

\def\vf{{\bm{f}}}
\def\vg{{\bm{g}}}
\def\vh{{\bm{h}}}

\def\vm{{\bm{m}}}

\def\vs{{\bm{s}}}

\def\vu{{\bm{u}}}

\def\vw{{\bm{w}}}
\def\vx{{\bm{x}}}

\def\vz{{\bm{z}}}

\def\mepsilon{{\bm{\varepsilon}}}

\def\mxi{{\bm{\xi}}}

\DeclareMathAlphabet{\mathsfit}{\encodingdefault}{\sfdefault}{m}{sl}
\SetMathAlphabet{\mathsfit}{bold}{\encodingdefault}{\sfdefault}{bx}{n}

\def\gL{{\mathcal{L}}}

\def\gX{{\mathcal{X}}}

\newcommand{\E}{\mathbb{E}}

\newcommand{\R}{\mathbb{R}}

\newcommand{\Cov}{\mathrm{Cov}}

\usepackage{tocloft}

\usepackage[utf8]{inputenc} 
\usepackage[T1]{fontenc}    
\usepackage{hyperref}       
\usepackage{url}            
\usepackage{booktabs}       
\usepackage{nicefrac}       
\usepackage{microtype}      
\usepackage{cleveref}
\usepackage{soul}

\usepackage{graphicx}
\usepackage{caption}
\usepackage{subcaption}

\usepackage{algorithm}
\usepackage{algorithmic}
\usepackage{tablefootnote}

\usepackage{latexsym}
\usepackage{tabularx}
\usepackage{multirow}
\usepackage{enumitem}
\usepackage{pifont}
\usepackage{color, colortbl}
\usepackage{float}
\usepackage{appendix}
\usepackage{bbold}
\usepackage{xspace}
\usepackage{adjustbox}
\usepackage{changepage}
\definecolor{DarkPink}{rgb}{0.5,0.0,0.18}
\definecolor{DarkGreen}{rgb}{0.1,0.5,0.1}
\definecolor{DarkRed}{rgb}{0.5,0.1,0.1}
\definecolor{DarkBlue}{rgb}{0.1,0.1,0.7}
\definecolor{DarkYellow}{rgb}{.79,.79,0}
\hypersetup{
    unicode=false,          
    pdftoolbar=true,        
    pdfmenubar=true,        
    pdffitwindow=false,      
    pdfnewwindow=true,      
    colorlinks=true,       
    linkcolor=DarkBlue,          
    citecolor=DarkRed,        
    filecolor=DarkRed,      
    urlcolor=DarkBlue,          
    pdftitle={},
    pdfauthor={},
}

\newcommand{\name}{SWAN} 
\newcommand{\gradnorm}{\mathtt{GradNorm}}
\newcommand{\gradwhite}{\mathtt{GradWhitening}}

\definecolor{commentcolor}{RGB}{110,154,155}   %
\usepackage{listings}
\definecolor{codegreen}{rgb}{0,0.6,0}
\definecolor{codegray}{rgb}{0.5,0.5,0.5}

\definecolor{backcolour}{RGB}{245,248,250}
\definecolor{emph}{RGB}{166,88,53}
\definecolor{nightblue}{RGB}{9,49,105}
\definecolor{keywords}{RGB}{207,33,46}
\definecolor{lightpurple}{RGB}{130,81,223}

\lstdefinestyle{mystyle}{
    backgroundcolor=\color{backcolour},   
    commentstyle=\color{codegreen},
    keywordstyle=\color{keywords},
    stringstyle=\color{nightblue},
    basicstyle=\fontsize{6.5}{6.5}\ttfamily,
    breakatwhitespace=true,         
    breaklines=true,                 
    captionpos=b,                    
    keepspaces=true,                 
    numberstyle=\tiny\color{codegray},
    numbersep=2pt,                  
    showspaces=false,                
    showstringspaces=false,
    showtabs=false,                  
    tabsize=2,
    captionpos=t,
    emph={},
    emphstyle={\color{lightpurple}},
    linewidth=0.98\columnwidth,
    frame=tb,    
    xrightmargin=0pt,
    xleftmargin=0.23cm,
    numbers=left,
    aboveskip=0.4cm,
    belowskip=0.4cm,
}

\usepackage{tcolorbox}  
  
\definecolor{lightorange}{RGB}{255, 229, 204}  
\definecolor{lightgreen}{RGB}{204, 255, 204}  
\definecolor{lightblue}{RGB}{	193, 230, 255}  
  
\newtcolorbox{keyquestionorange}{  
  colback=lightorange,  
  colframe=orange,  
  boxrule=1pt,  
  arc=4pt,  
  left=6pt,  
  right=6pt,  
  top=6pt,  
  bottom=6pt,  
  fontupper=\color{black}
}  
  
\newtcolorbox{keyquestiongreen}{  
  colback=lightgreen,  
  colframe=green,  
  boxrule=1pt,  
  arc=4pt,  
  left=6pt,  
  right=6pt,  
  top=6pt,  
  bottom=6pt,  
  fontupper=\color{black}
}  

\newtcolorbox{keyquestionblue}{  
  colback=lightblue,  
  colframe=blue,  
  boxrule=1pt,  
  arc=4pt,  
  left=6pt,  
  right=6pt,  
  top=6pt,  
  bottom=6pt,  
  fontupper=\color{black}
} 

\usepackage{amssymb}
\usepackage{mathtools}
\usepackage{amsthm}

\makeatletter
\renewcommand\thanks[1]{ 
  \footnotemark[2]
  \protected@xdef\@thanks{\@thanks
    \protect\footnotetext[2]{#1}}
}
\let\oldmaketitle\maketitle 
\renewcommand{\maketitle}{%
  \begingroup
  \renewcommand*{\thefootnote}{\fnsymbol{footnote}}%
  \oldmaketitle
  \endgroup
}
\g@addto@macro\@maketitle{\setcounter{footnote}{0}} 
\makeatother

\title{\name{}: SGD with Normalization and Whitening Enables Stateless LLM Training}

\iclrfinalcopy

\theoremstyle{plain}
\newtheorem{assumption}{Assumption}
\newtheorem{theorem}{Theorem}
\newtheorem*{theorem*}{Statement}
\newtheorem{proposition}{Proposition}
\newtheorem*{proposition*}{Statement}

\newtheorem{corollary}{Corollary}

\newtheorem{definition}{Definition}

\theoremstyle{remark}

\newenvironment{focusquestion}
  {\begin{center}\itshape\bfseries}
  {\end{center}}

\begin{document}
\enlargethispage{2\baselineskip}

\author{    
  Chao Ma$^{\ast}$ \\     
  Microsoft Research\\  
  \And    
  Wenbo Gong$^{\ast}$ \\    
  Microsoft Research\\ 
  \And    
  Meyer Scetbon$^{\ast}$ \\    
  Microsoft Research\\ 
  \And  
  Edward Meeds \\    
  Microsoft Research\\   
  \\  
  $^{\ast}$ These authors contributed equally to this work. 
}    


\newtcolorbox{focusquestionbox}{
    colback=blue!5!white, 
    colframe=blue!75!black, 
    fonttitle=\bfseries,
    coltitle=black,
    title=Focus Question, 
    boxrule=0.75mm,       
    arc=4mm,              
    width=\textwidth,     
    left=1mm,             
    right=1mm,            
    top=2mm,              
    bottom=2mm,           
}

\maketitle
\newcommand{\outidx}{j}
\newcommand{\inidx}{i}
\newcommand{\outdim}{n}
\newcommand{\indim}{m}
\newcommand{\batchidx}{b}
\newcommand{\batchdim}{B}
\begin{abstract}
Adaptive optimizers such as Adam~\citep{adam} have been central to the success of large language models. However, they often require maintaining optimizer states throughout training, which can result in memory requirements several times greater than the model footprint. This overhead imposes constraints on scalability and computational efficiency.
Stochastic Gradient Descent (SGD), in contrast, is a \emph{stateless} optimizer, as it does not track state variables during training. Consequently, it achieves optimal memory efficiency. However, its capability in LLM training is limited \citep{zhao2024deconstructing}. 
In this work, we show that pre-processing SGD using normalization and whitening in a stateless manner can achieve the same performance as the Adam optimizer for LLM training, while maintaining the same memory footprint of SGD. Specifically, we show that normalization stabilizes gradient distributions, and whitening counteracts the local curvature of the loss landscape. This results in SWAN (\ul{S}GD with \ul{W}hitening \ul{A}nd \ul{N}ormalization), a stochastic optimizer that eliminates the need to store any optimizer states. Empirically, SWAN achieves $\approx 50\%$ reduction on total end-to-end memory compared to Adam. In language modeling tasks, SWAN demonstrates comparable or even better performance than Adam: when pre-training the LLaMA model with 350M and 1.3B parameters, SWAN achieves a 2× speedup by reaching the same evaluation perplexity using half as many tokens.

\end{abstract}
\section{Introduction}\label{sec:introduction}

Adaptive optimizers, such as Adam and its variants~\citep{adam, adamw, Adafactor, pagliardini2024ademamix, sophia, Zhao2024GaLoreML}, have been central to the success of training large language models (LLMs) \citep{gpt3, llama2, llama3, Bi2024DeepSeekLS, bai2023qwen, zhang2022opt}. However, most adaptive optimizers for LLMs are \emph{stateful}, meaning they require tracking and maintaining internal states. While achieving remarkable empirical success, these states introduce significant memory overhead. For instance, Adam \citep{adam} -- the de facto optimizer for LLM training -- involves the tracking of exponential moving averages (EMAs), effectively doubling  memory requirements. AdEMAMix \citep{pagliardini2024ademamix}, an extension of Adam that achieves significant convergence speed boost, requires storing even more states, tripling the memory requirements. This overhead can be significant especially in distributed settings, where the optimizer states could consume a significant amount of the GPU memory \citep{llama3, korthikanti2023reducing}.  On the other hand, while  stochastic gradient descent (SGD) is optimal in terms of memory efficiency (i.e., it is \emph{stateless}), their capability to train LLMs is limited \citep{zhao2024deconstructing, zhang2020adaptive, kunstner2023noise, kunstner2024heavy}.
Therefore, a natural question arises:

\begin{focusquestion}
Can LLMs be trained efficiently, without storing any optimizer states?
\end{focusquestion}



There is a growing body of research that has contributed to answering this question positively by developing novel optimizers that reduce the memory requirements associated with tracking internal state variables during the training of LLMs while achieving similar or even speedup boost performance compared to Adam. For instance, some methods rely solely on tracking the first moment of gradients~\citep{xu2024adamlearningratescaling, jordan2024muon}, while others introduce an additional one-dimensional tracking variable on top of first moments~\citep{zhang2024adam, zhao2024deconstructingmakesgoodoptimizer}. Alternatively, approaches focusing exclusively on pre-conditioner tracking have also been proposed~\citep{pooladzandi2024curvature, li2017preconditioned}. Another line of work focuses on using low-rank approximations to store the first and second moments, thereby reducing the memory cost associated with tracking optimizer states~\citep{Lialin2023ReLoRAHT, Hao2024FloraLA, Zhao2024GaLoreML}. Finally, the concurrent work of ~\citep{zhu2024apollo} proposed an approximate gradient scaling method, that tracks the channel-wise or tensor-wise surrogate learning rates, further improving the memory efficiency.

In this work, we address this question by proposing to simply pre-process the instantaneous stochastic gradient in a stateless manner. The result is \name{} (SGD with Whitening And Normalization), a novel stochastic optimizer that eliminates \emph{all} internal optimizer states and empirically achieves comparable or even better performance compared to Adam on several LLM pre-training tasks.

Our optimizer consists of combining two well-known operators to pre-process the raw gradients: $\gradnorm$ and $\gradwhite$. $\gradnorm$ applies a row-wise normalization on the gradient matrix, while $\gradwhite$ orthogonalizes the normalized gradient matrix. We show that these operators aims at \emph{stabilizing} the stochasticity of gradient distributions during training, and \emph{neutralizing} the local geometry of the loss landscape, respectively. When applied together, both operators enables \name{} to rely solely on the statistics of the current gradient matrices. This approach eliminates the need to track state variables, thereby matching the memory footprint of SGD. In addition to memory savings, \name{} also demonstrates significant computational efficiency: our empirical evaluations on pre-training LLaMA \citep{touvron2023llama} models on the C4 dataset with multiple model sizes show consistently the same or better performance than Adam and other low-rank optimizers. Remarkably, at the 350M and 1.3B scale, our method achieves up to 2X faster convergence in terms of tokens seen compared to Adam. Our contributions are summarized below:

\begin{itemize}[leftmargin=1em]
    \item \textbf{A practical, stateless, adaptive optimizer.}  \name{} (\Cref{alg: no_ema}), is a novel optimizer based on pre-processing instantaneous stochastic gradients with two stateless operators---$\gradnorm$ and  $\gradwhite$ (\Cref{fig: \name{}}). 
    They perform gradient stabilization and loss landscape whitening, respectively, using information solely from the current gradient. \name{} offers:
    \begin{enumerate}
        \item Memory efficiency: \name{} only requires the memory footprint of SGD (w/o momentum), that is: $\approx 50\%$ reduction on total memory, and $\approx 100\%$ reduction on optimizer states compared to Adam. This can be further reduced with LOMO technique \citep{lv2023full} (\Cref{fig: preview} (c)).
        \item Sample efficiency: through experiments (\Cref{sec: exp}) on LLM pretraining tasks, we demonstrate that \name{} not only reduces memory overhead but also achieves similar or even better performance compared to Adam. Notably, \name{} achieves convergence speed-ups of over 2$\times$ in terms of the number of tokens used for both 350M and 1.3B models (\Cref{fig: preview}).  
        \item Computational efficiency: we propose a efficient scheme to accelerate the computation of \name{} (\Cref{sec: practical}), achieving a similar throughput to Adam without the need for distributed computation (\Cref{sec: memory_exp}) required by recent optimizers such as shampoo \cite{gupta2018shampoopreconditionedstochastictensor}.
        \item Robustness to hyperparameters: the hyperparameters of \name{} are lazily tuned via heuristics; and they are shared across all model sizes. Even when compared to tuned-baselines obtained with learning rate sweeps, \name{} still delivers strong performances on LLM pretraining. This is a promising indication of applicability to real-world problems.
    \end{enumerate}
    \item \textbf{Theoretical consistency with LLM dynamnics.} We show that 
    (1) $\gradnorm$ can stabilize the heterogeneous covariance of LLM gradients, leveraging the redundancies in LLM gradient flows (\Cref{thm: stability} in \Cref{app: discussion}); and (2) $\gradwhite$ can be derived as a non-diagonal second-order update under a specific structural assumption of the Hessian (\Cref{sec: gradwhite_discuss}). Additionally, we highlight that in the quadratic case, $\gradwhite$ leads to convergence rates that are robust to the condition number of the local curvature (\Cref{thm_sve_optimal}, \Cref{sec: analysis2}). 
    
\end{itemize}

\begin{table}[t]

\centering
\caption{The summary of memory consumption breakdown of optimizers, and the corresponding learning rate expressiveness. We assume all optimizers are applied on a 2D tensor of $m$ by $n$ ($m < n$). The optimizer state is defined as all intermediate variables maintained over time; while expressiveness is defined as the number of adaptive learning rates that a optimizer can model per tensor. }
\label{tab: summary table}
\resizebox{\textwidth}{!}{
\begin{tabular}{l|lllllll}
\hline
& Adam   & SGD  & SGD-Sal  &  Apollo (rank $r$) & Apollo (mini) & GaLore (rank $r$)  & \textcolor{blue}{\name{}}    \\ \hline
Optimizer States $\downarrow$  & $2mn$ & $0$ &$mn$  &  $2nr + 2 + (mr) $ \footnotemark    & $2n + 2 + (m)$  & $2nr+mr$  & $0$    \\ 
Model weights $\downarrow$  & $mn$ & $mn$ &$mn$  &  $mn$ & $mn$  & $mn$  & $mn$    \\
\hline
Expressiveness $\uparrow$ & $mn$ & $1$ & blockwise & $n$ & $1$ & $mn$ & $mn$ \\
\hline
\end{tabular}}
\label{tab: savings}
\end{table}

\begin{figure*}[t]
\captionsetup[subfigure]{labelformat=empty} 
    \centering
    \begin{minipage}{0.45\textwidth}
        \centering
        \includegraphics[width=\textwidth]{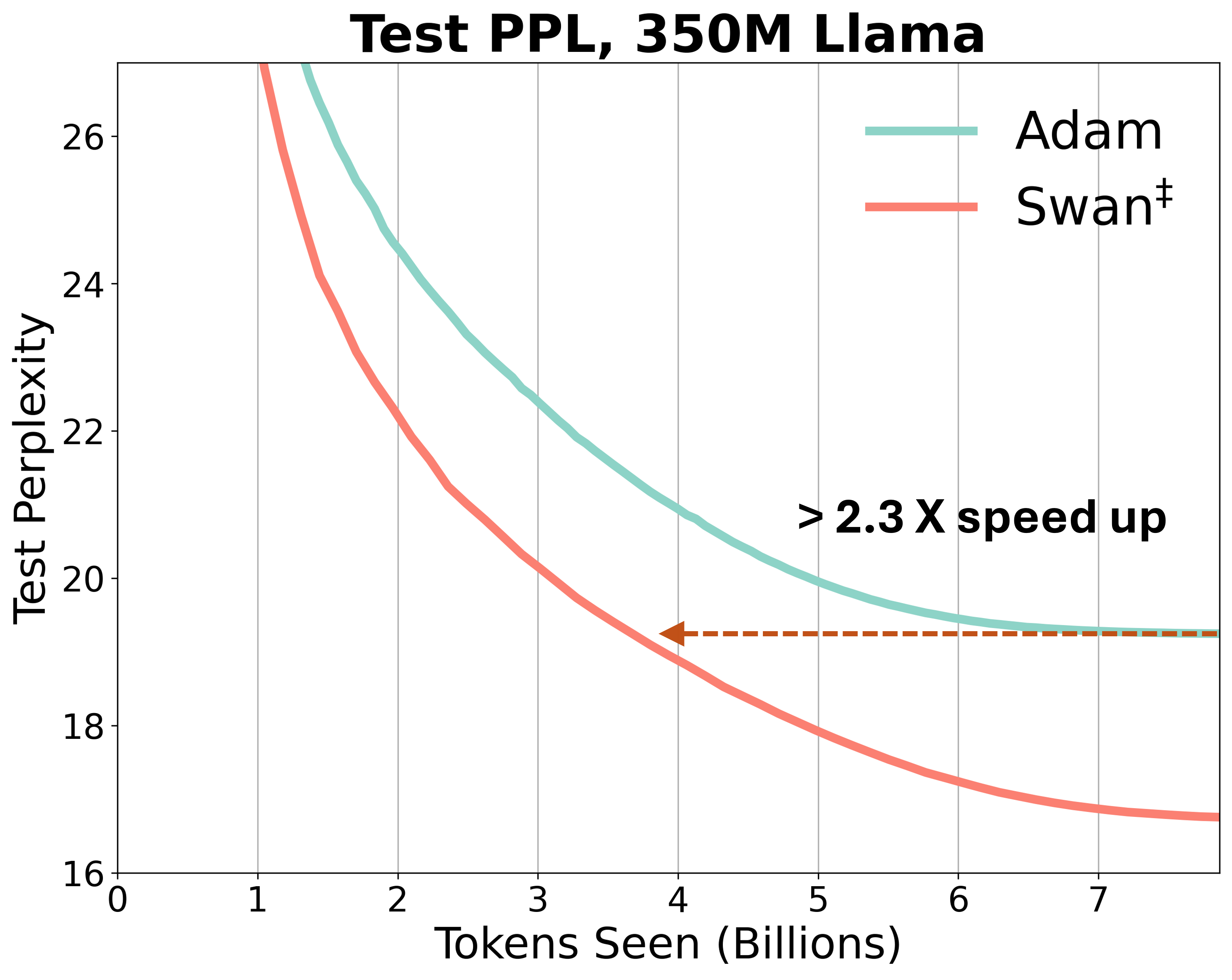}
        \caption*{(a) 350M model pretraining}
    \end{minipage}\hfill
    \begin{minipage}{0.45\textwidth}
        \centering
        \includegraphics[width=\textwidth]{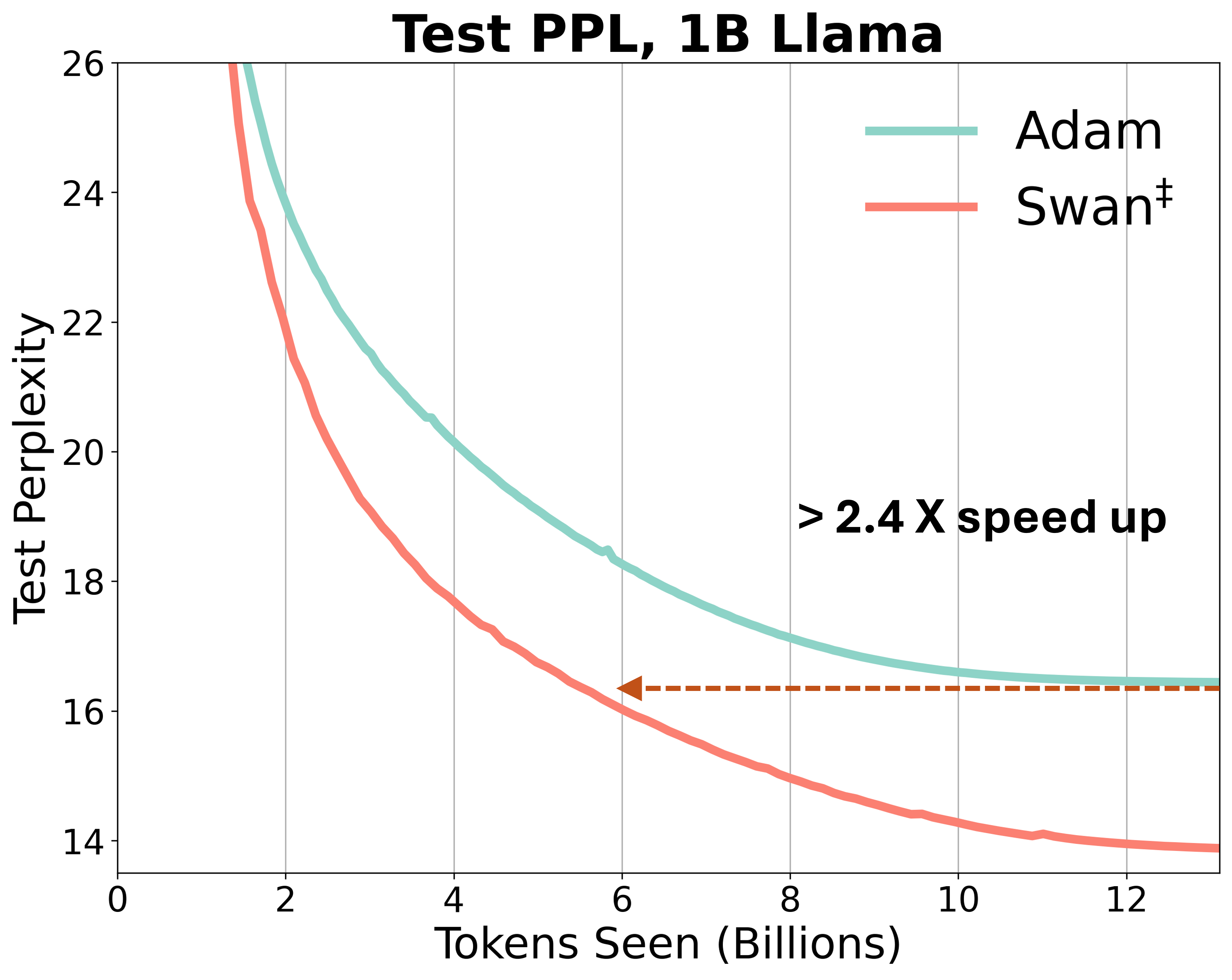}
        \caption*{(b) 1.3B model Pretraining}
    \end{minipage}\hfill
    
    \centering
    \begin{minipage}{0.45\textwidth}
        \centering
        \includegraphics[width=\textwidth]{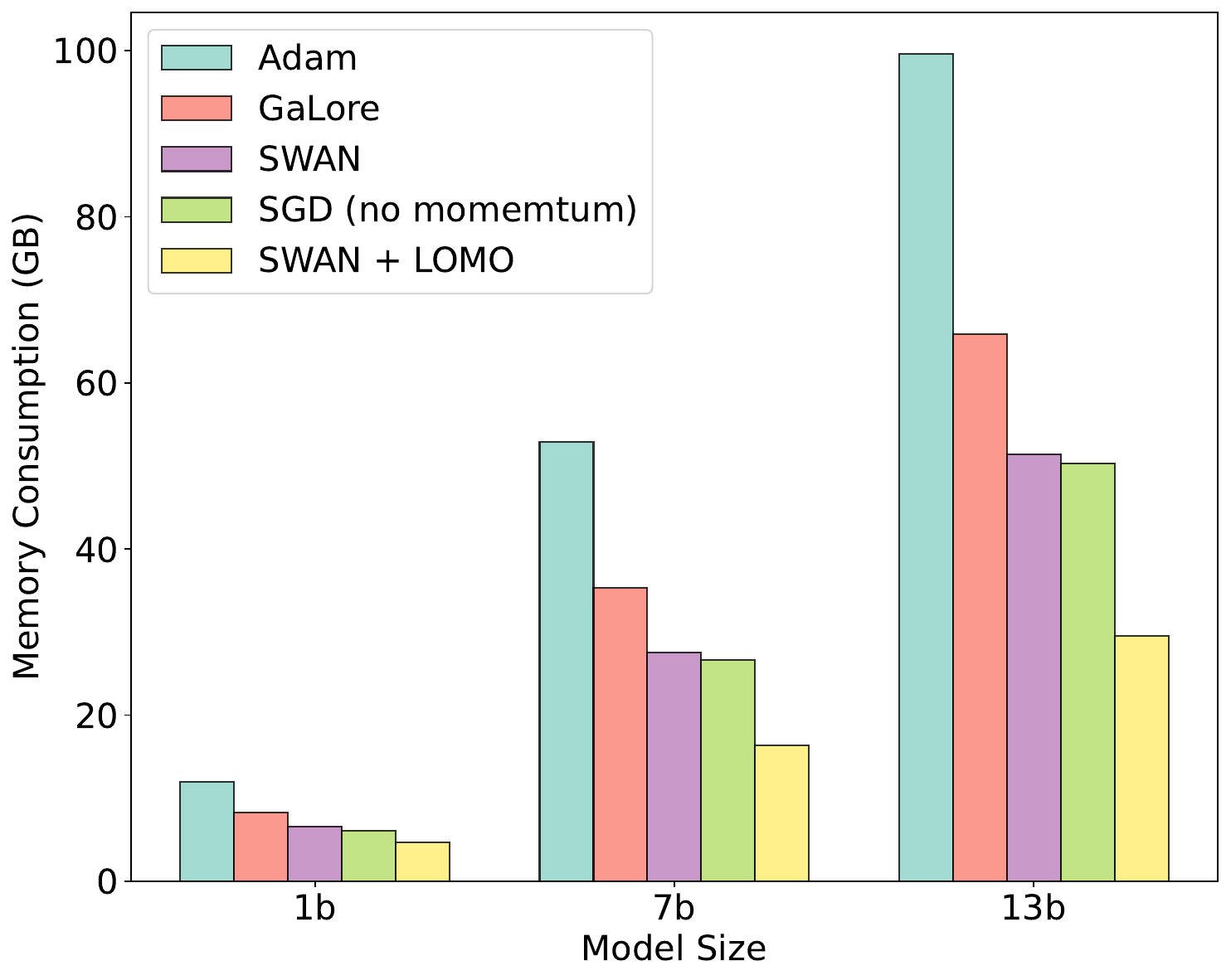}
        \caption*{(c) End-to-end memory footprint}
    \end{minipage}\hfill
     \begin{minipage}{0.45\textwidth}
        \centering
        \includegraphics[width=\textwidth]{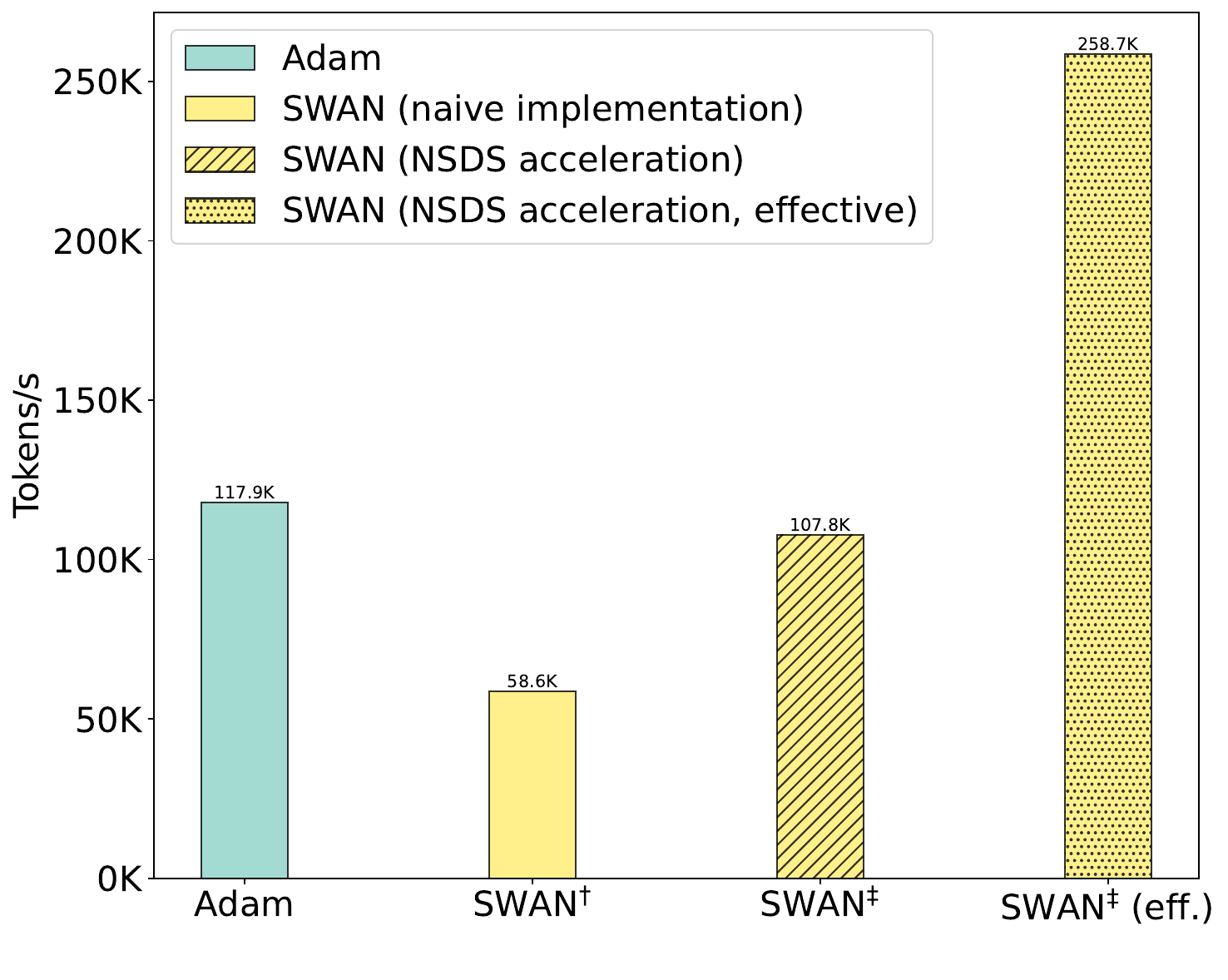}
        \caption*{(d) Throughput w/o distributed processing}
    \end{minipage}\hfill
\caption{\textbf{\name{} Performance on LLM Pretraining.} \textbf{(a)} and \textbf{(b)}: On both 350M and 1.3B LLama architectures, \name{} achieves over 2X speed-up compared to Adam in terms of tokens seen. \textbf{(c)}: Memory footprint. We measure end-to-end memory usage during full-model training with a batch size of 1 sequence. \name{} achieves nearly 100\% reduction in optimizer states memory, and 50\% total memory reduction (up to 70\% when combined with per-layer training technique \cite{Zhao2024GaLoreML, lv2023full}), and  \textbf{(d)}: Training throughput of a 1.3B model without distributed gradient pre-processing. With our acceleration scheme (denoteed by \textbf{\name{}$^\ddag$}), we closely match the throughput of Adam without the need for distributed computation \citep{shi2023distributed}. The rightmost bar shows the effective throughput, adjusted for the token efficiency of the optimizer relative to Adam.}
\label{fig: preview}
\end{figure*}

\section{Related works} \label{sec: related}

\paragraph{Towards Stateless LLM Training.}
Adaptive optimizers generally rely on tracking internal state variables to perform weight updates, which can substantially increase memory consumption when training large models. 
Several recent works have successfully managed to reduce the memory requirements associated with storing additional state variables for training LLMs.
Muon \citep{jordan2024muon, bernstein2024old}, a newly proposed optimizer, has demonstrated strong acceleration and memory saving for LLM training by \footnotetext{The factor $mr$ comes from the storage of random projection tensors. It can be reduced by only storing the random seeds; however, this might cause challenges in a distributed setting.} simplifying shampoo-like optimizers \citep{gupta2018shampoopreconditionedstochastictensor, anil2021scalablesecondorderoptimization, shi2023distributed, wang20244, vyas2024soap, peirson2022fishy, lin2024can} and requiring only the tracking of a first-moment estimate.  SGD-Sal~\citep{xu2024no} only stores the first moment with a learning rate estimated at the beginning of training. Sign-based methods such as \citep{lion} have also demonstrated success on training transformer-based models by only tracking first moments. There are also several works that aim to enhance the memory efficiency of Adam by reducing the memory cost associated with second moments.  Adam-mini~\citep{zhang2024adam} significantly reduces memory usage by storing only scalar values for each parameter block, while Adalayer~\citep{zhao2024deconstructingmakesgoodoptimizer} retains only the scalar average of the second moment for each layer. Alternatively,~PSGD~\citep{pooladzandi2024curvature, li2017preconditioned} focuses on exclusively tracking a pre-conditioner, eliminating the need to track a first moment estimate. Finally, ~\citep{zhu2024apollo} proposes a optimizer based on approximate gradient scaling to train LLMs, which requires tracking channel-wise or tensor-wise approximation to element-wise adaptive learning rates. However, all the aforementioned optimizers still require the storage of state variables. In contrast, \name{} completely eliminates the need to store internal states for both the first and second moments by employing a combination of $\mathtt{GradNorm}$ and $\mathtt{GradWhitening}$ steps, which is discussed next.

\paragraph{Pre-processing Gradients.} Gradient pre-processing is a common technique to improve optimizer performance. Various pre-processing procedures have been proposed in the literature, such as signed gradient \citep{bernstein2018signsgd, crawshaw2022robustness, lion, kunstner2023noise}, gradient clipping \citep{zhang2020adaptive}, normalization \citep{zhang2020adaptive, you2019lamb}, and whitening~\citep{yang2008principal, adam, hwang2024fadam, jordan2024muon, bernstein2024oldoptimizernewnorm, bernstein2024modulardualitydeeplearning, carlson2015preconditioned}. In this work, we focus on normalization and whitening. We apply normalization row-wise on gradient matrices, together with gradient whitening under a specific structural assumption of the Fisher Information \citep{hwang2024fadam, martens2018kronecker}, recovering the orthogonalization step used in~\citep{jordan2024muon, tuddenham2022orthogonalising}. Our key result is that {\em composing} normalization and whitening on stochastic gradients can be sufficient to enable the efficient training of LLMs in a stateless manner. In our setting, we show empirically that both removing any one of these two pre-processing steps from \name{} results in significant performance degradation (\Cref{sec: ablation}). Compared to Lamb~\citep{you2019lamb}, our normalization operation is applied on raw gradients row-wise; while LAMB applied layer-wise global normalization on Adam states instead. Compared to Muon~\citep{jordan2024muon}, \name{} removes first-moment tracking and instead uses row-wise normalization, making the optimizer fully statless. Moreover, our proposed efficient heuristic scheme for computing square root inverse requires $\mathcal{O}(m^2)$ instead of $\mathcal{O}(m^3)$ of the standard NS scheme \citep{song2022fast, li2018towards, huang2019iterative}.

\paragraph{Low-rank methods.} Low-rank optimization techniques have been explored in the context of large language model (LLM) training as a means to reduce memory consumption. These methods focus on applying low-rank approximations to model weights, gradients, and/or optimizer state variables. A seminal work in this domain is LORA~\citep{hu2021lora} to fine-tune pre-trained models using additional low-rank weight matrices at each layer, thereby significantly reducing memory usage to update the weights. More recently, methods such as ReLoRA~\citep{Lialin2023ReLoRAHT}, FLORA~\citep{Hao2024FloraLA}, and Galore~\citep{Zhao2024GaLoreML} have advanced low-rank optimization techniques for memory-efficient LLM pre-training. These approaches leverage low-rank gradient projections to enable full-rank learning, thereby achieving memory savings without compromising model capacity. Notably, they have achieved substantial reductions in the memory consumption of optimizer states, with only minimal impact on model performance. While these approaches effectively reduce the memory footprint of LLM training, they still necessitate storing internal states, resulting in higher memory consumption compared to SWAN.

\section{Preliminaries} \label{sec: adam}

Adam \citep{adam} is the current standard choice for adaptive optimizers across a multitude of machine learning training tasks, including LLM pre-training. Adam is an example of a \emph{stateful} optimizer, which means it accumulates and stores internal states throughout training. It combines the advantages of two earlier methods: AdaGrad \citep{Duchi2011AdaptiveSM}, which adapts learning rates based on the historical gradients' magnitudes, and RMSProp \citep{tieleman2012lecture}, which mitigates the aggressive decrease in learning rates by using a decaying average of squared gradients. 

Consider a loss function $\gL_{\rmW}: \gX \rightarrow \R$, parameterized by weight matrices $\rmW \in \R^{m \times n}$, and denote $\vx^{(t)}$ a mini-batch of inputs provided at the $t$-th training step that is sampled from data distribution $p_{\text{data}}(\vx)$. Let $\rmG^{(t)}$ be the stochastic gradient of $\gL_{\rmW}$ (i.e., a random variable induced by sampling $\vx^{(t)}$). Then, Adam can be broken down into the following steps:
\begin{align*}
        \rmG^{(t)} &= \nabla_{\rmW} \gL_{\rmW}(\vx^{(t)}), \quad    \vx^{(t)} \sim p_{\text{data}}(\vx) \hspace{1cm} &  \text{(stochastic gradient)}\\
        \vm^{(t)} &= \beta_1 \vm^{(t-1)} + (1-\beta_1) \rmG^{(t)} , 
        \hspace{1cm} & \text{(EMA first moment)}\\
        \vnu^{(t)} &= \beta_2 \vnu^{(t-1)} + (1-\beta_2) {\rmG^{(t)}}^2 , \quad  
        %
        %
        \hspace{1cm} & \text{(EMA second moment)} \\
        \rmW^{(t + 1)} &= \rmW^{(t)} - \eta \big( \frac{\hat{\vm}^{(t)}}{\sqrt{\hat{\vnu}^{(t)}}+\eps} \big)\hspace{1cm}& \text{(weight update)}
    \end{align*}
where $\vm^{(t)}$ and $\vnu^{(t)}$ are EMAs of the first and second moments of the gradients; and $\eta$ is a global step size. Intuitively, Adam estimates the signal-to-noise (SNR) ratios and use it to adjust learning rates element-wise. Tracking and storing these two EMA estimates triple the total memory consumption required to train a LLM model. For example, for LLaMA 405B model, storing model weights requires 810 GB of memory, while the Adam optimizer states requires an additional 1.6TB of memory.


\paragraph{Desired properties of Adam.} There is a rich literature on understanding adaptive methods’ inner workings and unreasonable effectiveness. Notably, the key desired properties of Adam can be categorized as \emph{gradient whitening}, \emph{gradient smoothing}, and \emph{gradient invariance}.

\begin{itemize}[leftmargin=1em]
    \item \emph{Gradient whitening}: it is known that the inverse second moment $ \frac{1} {\sqrt{\hat{\vnu}^{(t)}}+\eps} $ of Adam 
    performs gradient whitening by approximating the square root inverse of the diagonal of Fisher information matrix \citep{adam, hwang2024fadam}. This step
    biases the optimization trajectories towards well-conditioned regions \citep{jiang2024does} and provides a better approximation to the geodesic flow when compared with the natural gradient update \citep{yang2008principal}.
    \item \emph{Gradient smoothing}: the EMA operations in Adam naturally  reduce the influence of mini-batch noise \citep{cutkosky2020momentum,crawshaw2022robustness};
    \item \emph{Gradient invariance}: recent work suggest the performance gap between SGD and Adam might lie in Adam's \emph{sign-descent}-like nature \citep{bernstein2018signsgd, crawshaw2022robustness, lion}. For example, Adam is robust to the rescaling of gradient diagonals \citep{adam}; and is invariant to any sign-preserving scalings (under $\beta_1 = \beta_2 = 0$)~\citep{bernstein2018signsgd}.
\end{itemize}

For a more comprehensive discussion of these properties, please refer to \cref{app: desired property}.

\paragraph{Adam as SGD pre-processing.} Adam can be viewed as history-dependent pre-processing of the gradients ($\{ \rmG^{(0)}, \rmG^{(1)}, ..., 
 \rmG^{(t)}\} \rightarrow \frac{\hat{\vm}^{(t)}}{\sqrt{\hat{\vnu}^{(t)}}+\eps} $) to achieve the desired properties described above. A key observation is that all of these properties are achieved through element-wise operations, where each element of the gradient matrix is independently pre-processed and re-scaled. This approach does not take into account the interactions and structures between different variables, and we hypothesize that this is the reason why additional history information is necessary to bridge this gap, ultimately leading to the requirement for EMA states. Thus, we believe that designing stateless adaptive optimizers is possible if we can achieve similar properties by applying matrix-level operations that pre-process the instantaneous stochastic gradients of SGD ($\rmG^{(t)} \rightarrow \tilde{\rmG}^{(t)} $).

\section{The \name{} Optimizer: Preprocessing SGD with Normalization and Whitening}

\begin{figure}[h]
\captionsetup[subfigure]{labelformat=empty} 
    \centering
        \centering\includegraphics[width=\textwidth]{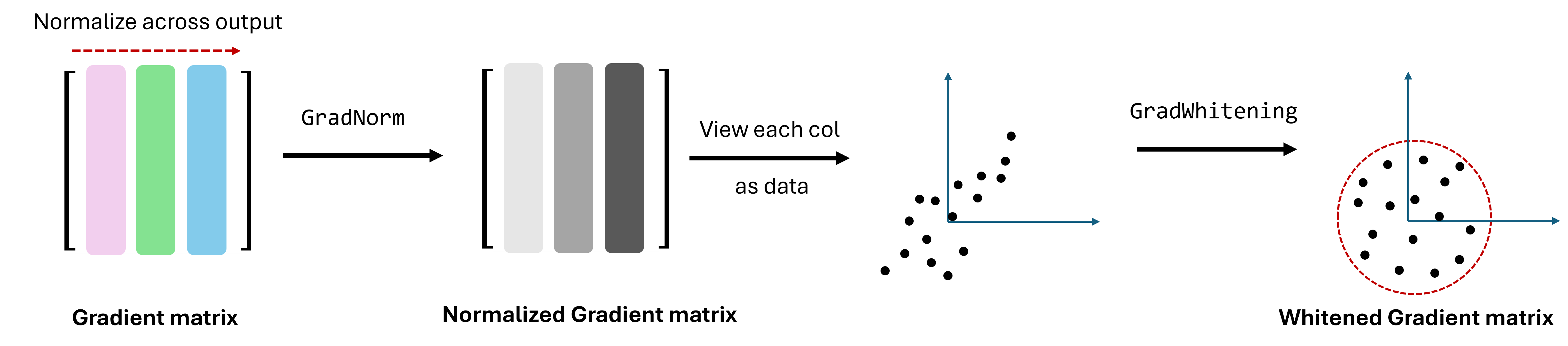}
\caption{Illustration of $\gradnorm$ and $\gradwhite$ operators. In $\gradnorm$ operator, we perform standardization across the output dimensions (columns), using statistics computed row-wise. In $\gradwhite$ operator (illustration adapted from \citet{huang2019iterative}), we treat each column of the gradient matrix $\rmG$ as a separate data sample. Then, $\gradwhite$ can be seen as stretching/squeezing the data such that the covariance matrix is the identity across all eigen directions.}
\label{fig: \name{}}
\end{figure}

\begin{figure}[h!]
    \centering
    \begin{minipage}[t]{0.49\textwidth}
        \centering
        \begin{algorithm}[H]
            \caption{\name{} Optimizer}
            \label{alg: no_ema}
         \begin{algorithmic}
           \STATE  {\bfseries Input:} weight matrix $\rmW \in \mathbb{R}^{m \times n}$ with $m \leq n$. Step size $\eta$. Number of $\gradwhite$ iteration $k$ (default = 10).
           \STATE Initialize step $t \gets 0$
           \REPEAT
           \STATE $\rmG^{(t)} \in \mathbb{R}^{m \times n} \gets \nabla_{\rmW^{(t)}} \mathcal{L}^{(t)}(\rmW^{(t)})$ 
           \STATE $\tilde{\rmG}^{(t)} \gets \gradnorm(\rmG^{(t)})$ (Eq (\ref{eq: gradnorm}))   \hfill
           \STATE $\Delta \rmW^{(t)} \gets \gradwhite(\tilde{\rmG}^{(t)}, k)  \hfill $
           \STATE (optional) $ \Delta \rmW^{(t)} \gets \frac{ \sqrt{mn}
         \Delta \rmW^{(t)}}{\|\Delta \rmW^{(t)}\|}\hfill$
           \STATE $\rmW^{(t)} \gets \rmW^{(t-1)} - \eta \Delta \rmW ^{(t-1)} $
           \STATE $t \gets t + 1$
           \UNTIL{convergence criteria met}
           \RETURN $\rmW^{(t)}$
         \end{algorithmic}
\end{algorithm}
    \end{minipage}
    \hfill
    \begin{minipage}[t]{0.49\textwidth}
        \centering
        { 
        \begin{algorithm}[H]
            \caption{$\gradwhite$ Operator}
            \label{alg:optimizer2}
            \begin{algorithmic}
                \STATE \textbf{Input:} $\rmG^{m \times n}$ with $m \leq n$. Number of iterations $K$. Step size $\beta$. Boolean $\texttt{diag}$ indicating if use diagonal substitution scheme.
                \STATE Initialize $\rmY \gets \rmG \rmG^\top $, $\rmZ \gets \rmI$
                \FOR{i = 1,...,K}
                    \IF{\texttt{diag}}
                      \STATE   $\rmY \gets \beta \rmY \text{Diag}(3 \rmI - \rmZ \text{Diag}(\rmY))$, 
                      \STATE  $\rmZ \gets \beta (3 \rmI - \text{Diag}(\rmZ) \rmY) \text{Diag}(\rmZ)$
                  \ELSE 
                  \STATE   $\rmY \gets \beta \rmY (3 \rmI - \rmZ \rmY)$, 
                      \STATE  $\rmZ \gets \beta (3 \rmI - \rmZ \rmY) \rmZ$
                    
                  \ENDIF
              \ENDFOR
              \RETURN $\rmZ \rmG$
            \end{algorithmic}
        \end{algorithm}
        \vspace{-1.3em}
        }
        
    \end{minipage}
    \caption{\name{} Optimizer. When $\mathtt{diag}$ is set to $\mathtt{True}$, we cover the fast variant denoted by \name{}$^\ddag$.}
\end{figure}

\label{sec: \name{}}
As discussed in \Cref{sec: adam}, we believe the key to designing stateless, adaptive, and effective optimizers lies in the incorporation of \emph{matrix-level} operations that exploit rich information contained in the gradient matrix. To this end, we compose two well-known matrix operators, namely normalization and whitening. When applied in tandem, they achieve similar desirable properties of adaptive optimizers, without the need to store historical gradient moments. The result is \name{}  ({S}GD with {W}hitening {A}nd {N}ormalization), a new stateless optimizer which we describe next.

{\subsection{\name{} Update Rules}} \label{sec: update_rules}
In \name{} (\Cref{alg: no_ema}), the raw SGD gradient $\rmG^{t}$ is processed by the operations below\footnote{Here $\rmG$ is assumed to be the gradient matrix for some parameter block in the model (e.g. a linear layer); enabling us to take advantage of the matrix structure to achieve our smoothing and whitening goals.}:

\begin{align*}
  \begin{cases}
    \tilde{\rmG}^{(t)} \leftarrow \gradnorm(\rmG^{(t)}) \\
    \Delta \rmW^{(t)} \leftarrow \gradwhite(\tilde{\rmG}^{(t)})
  \end{cases}\tag{\name{}} \label{eq:noema}
\end{align*}

The weight is then updated by $\rmW^{(t + 1)} = \rmW^{(t)} - \eta \Delta \rmW ^{(t)}$. The $\gradnorm$ operator (\Cref{eq: gradnorm}) denotes the normalization of the gradient matrix row-wise (\Cref{fig: \name{}}); and the $\gradwhite$ operator denotes the whitening of the gradients~(\Cref{fig: \name{}}). Both operators have been extensively applied in different contexts of optimization and/or architecture. Related applications include the processing of neural network forward pass activations (in the form of RMS layer norm \citep{zhang2019root} and decorrelated batch norm \citep{huang2018decorrelated, song2022fast}), as well as the processing of backward gradients~\citep{ tuddenham2022orthogonalising, jackson2023isometric, jordan2024muon}. 

\paragraph{On the $\gradnorm$ Step.}
 Consider the gradient matrix $\rmG \in \mathbb{R}^{m \times n}$ with rows and columns corresponding to input and output dimensions, respectively, of some block of model parameters.  Let $1 \leq \inidx \leq \indim$ represent the input indices and $1 \leq \outidx \leq \outdim$ represent the output indices. Instead of performing element-wise EMA to stabilize and normalize the noisy gradients, as done in Adam, we propose to standardize across the output dimensions at each time step $t$:
\begin{equation}
   \gradnorm(\rmG) := \frac{ \rmG  }{ \vs \mathbf{1}_{\outdim}^\top  } \label{eq: gradnorm}
\end{equation}
where $ \vs := \sqrt{ \frac{1}{\outdim} \sum_{\outidx=1}^{\outdim} (\rmG_{:, \outidx}}) $ is ${\indim}$-dimensional the root mean square across dimension ( ${\indim}$-dimensional column vector); $ \mathbf{1}_{\outdim} $ is a ${\outdim}$-dimensional column vector of ones.

$\gradnorm$ is the forward pass operator RMS-Norm \citep{zhang2019root} applied on backward gradients. $\gradnorm$ allows the optimizer to be invariant under matrix-wise and row-wise rescaling. In \Cref{sec: analysis} we show that this is the key to stabilizing the LLM gradient distribution. 
Meanwhile, compared to the sign operation ($\text{sign}(\rmG)$) and matrix-wise normalization ($\frac{\rmG}{\|\rmG\|}$) alternatives \citep{zhang2020adaptive, you2019lamb}, $\gradnorm$ preserves richer information of gradient scaling while offering invariance properties.

%

\paragraph{On the $\gradwhite$ Step.}


As discussed in \Cref{sec: adam}, Adam relies on a second moment estimate to perform \emph{element-wise} gradient whitening. More formally, its second moment estimate a diagonal approximation to the Fisher information matrix (FIM): ~\citep{adam, hwang2024fadam}:  
\begin{equation*}
     \mathbb{E}(\text{Vec}(\rmG)\text{Vec}(\rmG)^\top) \approx \text{Diag}[\text{Vec}[\mathbb{E}(\rmG^{\odot 2})]] \label{eq: exact_whitening}
\end{equation*}
where $\text{Vec}(\cdot)$ denotes the vectorized operation, $\text{Diag}(\cdot)$ denotes the operation that produces a diagonal matrix given a vector, and $\odot$ denotes the Hadamard product. Hence, whitening using the inverse square root of this diagonal FIM is equivalent to element-wise rescaling with $\frac{1}{\sqrt{\mathbb{E}(\rmG^{\odot 2})}}$. Here, we consider the following non-diagonal approximation:
\begin{equation*}
   \mathbb{E}(\text{Vec}(\rmG)\text{Vec}(\rmG)^\top)\approx \rmI_{n \times n}  \otimes \rmG\rmG^\top \;,
\end{equation*}
where $\otimes$ denotes Kronecker product. This leads to the following whitening step:
\begin{align}
 \label{eq: gradwhite_def}
    \gradwhite(\rmG):=(\rmG\rmG^\top)^{-1/2}\rmG
\end{align}
where the exponent $-\frac{1}{2}$ stands for matrix inverse square root. $(\rmG\rmG^\top)^{-\frac{1}{2}} \rmG$ is simply the orthogonalization of $\rmG$, i.e., the closest orthogonal matrix to $\rmG$ (w.r.t the Frobenius norm). The derivation of \Cref{eq: gradwhite_def} as structured FIM/Hessian approximation is discussed in \Cref{sec: gradwhite_discuss}. Apart from the FIM/Hessian approximation view,  \Cref{eq: gradwhite_def} can also be interpreted as \emph{minimizing the condition number of $\rmG$} (defined as the ratio between the largest and smallest singular values). It can be shown that after whitening, its condition number achieves the minimum value which is 1. 

Similar to $\gradnorm$,  $\gradwhite$ (\Cref{eq: gradwhite_def}) as a matrix operation has been widely used as a forward-pass operator in the form of decorrelated batch normalization \citep{huang2018decorrelated, li2018towards}; and it has also shown great success in processing backward gradients \citep{tuddenham2022orthogonalising, jordan2024muon, gupta2018shampoopreconditionedstochastictensor}. As shown in \Cref{fig: \name{}}, by treating each column of $\rmG$ as i.i.d. vector-valued data samples $\rmG = \{\vg_1, ..., \vg_{\outidx}, ..., \vg_{\outdim}\}$, $\mathtt{GradWhitening}$ can be seen as effectively stretching/squeezing this data matrix along the eigenvectors to whiten its covariance matrix. This essentially forces the gradients to \emph{traverse all eigen-directions at the same rate}. 

\subsection{Practical Considerations: Fast and Stable Implementation} \label{sec: practical}

\paragraph{Fast $\gradwhite$} Computing $\gradwhite$ exactly can be expensive, as it involves solving the matrix square-root inverse. One option is to directly apply the Newton-Schulz variant of decorrelated batch normalization \citep{song2022fast, li2018towards, huang2019iterative}, which allows a more GPU-friendly estimation. This is given by \citep{song2022fast, li2018towards}:
%
\begin{align}
  \begin{cases}
   \rmY_{k+1} = \frac{1}{2} \rmY_{k} (3 \rmI - \rmZ_k \rmY_k) \\
   \rmZ_{k+1} = \frac{1}{2} (3 \rmI - \rmZ_k \rmY_k) \rmZ_k
  \end{cases}
  \label{eq: NS}
\end{align}
where $\rmY_0 = \rmG \rmG^\top $ 
, $\rmZ_0 = \rmI$. At convergence, $\mathtt{GradWhitening}(\rmG) = \rmZ \rmG$ (\Cref{alg:optimizer2}). See \Cref{app: code} for implementation details.  Note that the same N-S procedure has been discussed in \citep{mei2023kradagrad} and ~\citep{jackson2023isometric} for preconditioned optimizers. Another N-S procedure was been introduced in \citep{bernstein2024old}, and optimized in~\citep{jordan2024muon}.

However, estimating $(\rmG\rmG^\top)^{-1/2}$ with NS requires $\mathcal{O}(m^3)$ (assuming $m < n$) complexity, making its scalability for larger models, especially under distributed settings unclear \cite{jordan2024muon}. Here, we propose a heuristic scheme that has $\mathcal{O}(m^2)$ complexity to estimate square-root inverse (comparable to the $\mathcal{O}(mn)$ complexity of Adam second moment):
\begin{align}
  \begin{cases}
   \rmY_{k+1} = \frac{1}{2} \rmY_{k} \text{Diag}(3 \rmI - \rmZ_k \text{Diag}(\rmY_k)) \\
   \rmZ_{k+1} = \frac{1}{2} (3 \rmI - \text{Diag}(\rmZ_k) \rmY_k) \text{Diag}(\rmZ_k) 
  \end{cases}
  \label{eq: GNS}
\end{align}
where under a slight abuse of notation, $\text{Diag}(\cdot)$ returns a diagonal matrix that has the same diagonals as the input matrix. Intuitively, whenever we encounter matrix multiplication in NS iterations, we replace one of the matrices with its diagonal approximation. The particular choice of where to apply $\text{Diag}(\cdot)$ is determined by performance on synthetic datasets. We refer to this as the \emph{NS with diagonal substitution} (NSDS) scheme. In \Cref{app: NS}, we empirically demonstrate that NSDS performs well in minimizing the matrix condition number of gradients. As shown in \Cref{fig: condition_number}, on the gradient distribution induced by Llama model training, the performance of NSDS (when combined with $\gradnorm$) is comparable to or better than the standard NS scheme in reducing the condition number.


We found that NSDS does have the side effect of slowing down the early stage convergence. However, it offers stronger tail convergence in return. For LLM pretraining NSDS helps \name{} achieve similar or better final validation loss performance than the original NS scheme (\Cref{sec: llm_exp}), enabling Adam-level training throughput without the help of distributed computation (\Cref{sec: memory_exp}).


\paragraph{Rescaling for robustness.} The operator $\gradwhite$ maps the normalized gradient $ \tilde{\rmG}^{(t)}$ onto the closest orthogonal matrix, and as such might drastically change its effective learning rate. In practice, we propose the following re-scaling before updating the weights:
\begin{equation}
    \Delta \rmW^{(t)} \gets \frac{ \|\tilde{\rmG}^{(t)}\|
 \Delta \rmW^{(t)}}{\|\Delta \rmW^{(t)}\|} =  \frac{ \sqrt{mn}
 \Delta \rmW^{(t)}}{\|\Delta \rmW^{(t)}\|} \label{eq: optional}
\end{equation}

This helps rescale the norm of the whitened gradient back to the norm of $ \tilde{\rmG}^{(t)}$, that is $ \|\tilde{\rmG}^{(t)}\|= \sqrt{mn}$. Notice that this is exactly the norm of signed gradient descent ($=$ Adam without EMAs) \citep{bernstein2018signsgd}. This relationship allows us to directly adapt the learning rates from Adam without tuning specifically for \name{}. This can be throught as a variation of LR grafting \citep{agarwallearning}. 

Finally, the practical \name{} algorithm ($\gradnorm$ + NSDS-$\gradwhite$ + rescaling) is presented in \Cref{alg: no_ema} (under $\mathtt{diag} = \text{True}$). We denote it as \textbf{\name{}$^\ddag$}.

\subsection{Saving analysis: balancing memory efficiency and expressiveness}

\Cref{tab: savings} shows the theoretical memory savings of \name{}. When applied on a 2D tensor of $m$ by $n$ ($m > n$), \name{} requires storing $0$ optimizer states (i.e., intermediate variables maintained over time), compared with other state-of-the-art memory efficient optimizers. Notably, this is achieved without significantly compromising its capabilities to model the adaptive learning rates. \name{} can model \emph{element-wise} adaptive learning rates, while the concurrent work (SGD-sal and Apollo) needs to resort to block-wise or even tensor-wise approximations. In our experiments in \Cref{sec: llm_exp}, we show that these approximations indeed has an negative impact on LLM training performance.

\section{Analysis: A LLM Learning Dynamics Perspective} \label{sec: analysis}

As a new stateless adaptive optimizer, the complete theoretical properties of \name{} is an open question which we leave for future work. However, as a first analysis, we consider \name{} from a {\em learning dynamics} perspective, specifically the dynamics of an LLM based upon a simplified transformer block.  It is this analysis that led to the design choices for \name{}.

\subsection{Setup}
We consider the simplified transformer block (STB) architecture proposed in \citep{tian2023joma}.

\begin{definition}[Simplified Transformer Block (STB)] \label{def: STB}

Given the input activation $\vx \in \mathbb{R}^{M_C \times 1}$, query token index $q$, context embedding matrix $\rmU_C \in \mathbb{R}^{d \times M_C}$, and the query embedding $\vu_q \in \mathbb{R}^{d \times 1}$, the STB computes the output $\vh \in \mathbb{R}^{n \times 1}$ as $ 
\vh = \phi \left( \rmW^\top  \left( \rmU_C \left( \exp(\vz_q) \odot \vx \right) + \vu_q \right) \right)$, 
where $M_C$ is the context length, the attention logits $\vz_q \in \mathbb{R}^{M_C \times 1}$ are given by
$z_{ql} = \vu_q^\top  \rmW_Q^\top  \rmW_K u_l$, 
with $\rmW_Q, \rmW_K \in \mathbb{R}^{d \times d}$ being weight matrices for the queries and keys, respectively, $\rmW \in \mathbb{R}^{d \times n}$ is the weight matrix for the feedforward network, and $\phi$ is a nonlinearity function such as the ReLU.
\end{definition}

Given a STB, we consider a loss function $\gL_{\rmW, \vz}(\vx^{(t)})$, where $\vx^{(t)}$ is a mini-batch of inputs provided at the $t$-th training step sampled from data distribution $p_{\text{data}}(\vx)$. Standard mini-batch learning dynamics is then given by 
$$\dot{\rmW}^{(t)} = \frac{ \partial \gL_{\rmW, \vz_q}(\vx^{(t)}) }{\partial \rmW}, \quad \dot{\vz_q}^{(t)} = \frac{ \partial \gL_{\rmW, \vz_q}(\vx^{(t)}) }{\partial \vz_q}\; .$$ In this case, both $\dot{\rmW}^{(t)}$ and $\dot{\vz_q}$ are viewed as random variables induced by random mini-batch $\vx^{(t)}$.  For example, for each row $i$, $\dot{\rmW}^{(t)}[i, :]$ can be re-written as $\dot{\rmW}^{(t)}[i, :] =  \E[\dot{\rmW}^{(t)}[i, :]] + \mepsilon_{\rmW}^{(t)}[i, :]$, where $\mepsilon^{(t)}[i, :]$ is zero mean random variable with covariance $\Cov[\dot{\rmW}^{(t)}[i, :]]$. 

\subsection{$\gradnorm$: stabilizing gradient distributions of STB}

Below we show that, based on the dynamics of the STB, $\gradnorm$ stabilizes $\mepsilon_{\rmW}^{(t)}$.

\begin{theorem}[\textbf{$\gradnorm$ stabilizes gradient distributions across time} for the STB]\label{thm: stability}
    Consider the STB (\Cref{def: STB}). Assuming we inherit the assumptions in Theorem 1 of \citet{tian2023joma}, as described in \Cref{app: discussion}. Then consider $\rmU_{C}^\top\rmW$, the composition of the MLP project-up matrix and the embedding matrix as a whole. Then, its standardized stochastic gradients $ \tilde{\rmG}_{\rmU_{C}^\top\rmW}^{(t)} := \gradnorm( \frac{ \partial \gL_{\rmW, \vz}(\vx^{(t)}) } {\rmU_C^ \rmW}) $ satisfy:
    $$
    \Cov[\tilde{\rmG}_{\rmU_{C}^\top\rmW}[i, :]^{(t_1)} ] = \Cov[\tilde{\rmG}_{\rmU_{C}^\top\rmW}[i, :]^{(t_2)}] \quad \text{for all } t_1, t_2, \text{and} \, i.
    $$
    In other words, the covariance structure of $ \tilde{{\rmG}}$ is identical across all time steps $ t $, achieving distributional stability across time. The same relationship also holds for the gradient of attention score  $ \tilde{\rmG}_{\vz_q}^{(t)} := \gradnorm( \frac{ \partial \gL_{\rmW, \vz_q}(\vx^{(t)}) } {\partial \vz_q}) $.

\end{theorem}

\Cref{thm: stability} suggests that $\gradnorm$ implicitly aligns with the dynamics of transformer architectures and removes the time-heterogeneity in gradient covariance structures.

\subsection{$\gradwhite$: An Efficient Non-diagonal Second-Order Update} \label{sec: gradwhite_discuss}

Here, we show that $\gradwhite$ is equivalent to a non-diagonal second-order method under a specific Kronecker factorization assumption of the Hessian/FIM. The assumption is as below:

\begin{assumption}[Assumption of $\gradwhite$]
\label{assumption:struct-hessian}
At time $t$, the local Hessian $\rmH$ of the loss has shared block-diagonal structure, such that $ \rmH = \rmI_{n \times n} \otimes \tilde{\rmH}$, where $\tilde{\rmH} \in \mathbb{R}^{m \times m}$.
\end{assumption}

Approximating Hessian with a Kronecker factorization is not new and has been extensively studied in the literature~\citep{martens2015optimizing,george2018fast, gao2023eigenvalue,gao2021trace, koroko2022efficient, eschenhagen2024kronecker, gupta2018shampoopreconditionedstochastictensor}. Here, this specific structure is useful in our context as 1) it enables estimation of Hessian/FIM without temporal EMAs, nor mini-batch statistics, and 2) it aligns with the statistical property of STB, as shown later in Proposition~\ref{prop_llm_hessian_main}.

By leveraging assumption~\ref{assumption:struct-hessian}, we can now effectively estimate $\rmH$ by only using one single gradient matrix sample $\bm{G}:=[\vg_1,\dots,\vg_n]\in\mathbb{R}^{m\times n}$. Recall that the Fisher information formulation of Hessian is defined as $\rmH =\mathbb{E}[\text{Vec}(\rmG)\text{Vec}(\rmG)^\top]$ where $\text{Vec}(\cdot)$ denotes the vectorized operation. Under assumption~\ref{assumption:struct-hessian}, we can estimate $\rmH=\rmI_{n \times n} \otimes \tilde{\rmH}$ by computing the following simple estimate $\frac{1}{n}\sum_{i=1}^n \vg_i\vg_i^\top = \bm{G}\bm{G}^\top$, which approximate the current $\tilde{\rmH}$. 
Hence, $\gradwhite$ can be seen as applying a second order update under our structural assumption:
$$
    \mathtt{GradWhitening}(\rmG) = (\rmG\rmG^\top)^{-\frac{1}{2}} \rmG = \tilde{\rmH}^{-\frac{1}{2}} \rmG
$$
 
In the following Proposition, we show that the assumption~\ref{assumption:struct-hessian} of $\gradwhite$ aligns with the equilibrium Hessian structure in the STB regime.

\begin{proposition}[\textbf{Shared structures in the block-diagonal of Hessians at transformer equilibrium}] \label{prop_llm_hessian_main} Consider a STB (\ref{def: STB}), trained with full-batch gradient descent.  Next, assume we inherit all the assumptions from Theorem 1 of \cite{tian2023joma}.  
Then, as $t \xrightarrow[]{} \infty$, we have the following shared Hessian structure along the diagonal blocks:
\begin{equation}
    \frac{\rmH_{sk,s'k}}{\sum_{s, s'} \rmH_{sk,s'k}} \xrightarrow[]{} \frac{\rmH_{sk',s'k'}}{\sum_{s, s'} \rmH_{sk',s'k'}}, \quad \forall 1 \leq s, s' \leq d, 1 \leq k, k' \leq n
\end{equation}
Where $\rmH(\rmW)_{sk,s'k'} = \frac{\partial \mathcal{L}}{\partial{w_{sk}} \partial{w_{s'k'}}}$.
\end{proposition}

Proposition \ref{prop_llm_hessian_main} shows that, under a simplified setting of the transformer, the Hessian will also converge to an equilibrium solution where the $M_C \times M_C$ blocks over the diagonal direction of Hessian shares an identical structure, which supports the assumption of $\gradwhite$. This result is verified in our numerical experiment (Appendix, Figure \ref{fig_stb_hessian}). Finally, \Cref{thm_sve_optimal} presented in Appendix will also show how $\gradwhite$ helps to make the convergence rate of SGD more robust to the condition number of local curvatures, and outperforms both SGD and Adam in the ill-conditioned regime.

\subsection{Verifying the Theoretical Insights} 


We empirically validate the theoretical benefits of $\gradnorm$ and $\gradwhite$. As detailed in \Cref{app: verification}, we demonstrate that $\gradnorm$ effectively stabilizes stochastic gradient distributions during SGD training of a scaled-down LLaMA model on the C4 dataset, evidenced by significantly reduced KL divergence fluctuations compared to standard training. Additionally, $\gradwhite$ significantly enhances optimization performance across high-dimensional and ill-conditioned classic optimization functions, consistently outperforming both SGD and Adam. These findings confirm that $\gradnorm$ promotes gradient stability and $\gradwhite$ addresses local curvature challenges, enabling faster and more reliable performance.

\section{Experiments} \label{sec: exp}

In this section, we report empirical results for \name{}. All experiments run on NVIDIA A100 GPUs. 

    
    


\subsection{Memory-Efficient Pre-training of Llama Models} \label{sec: llm_exp}

\paragraph{Setup} We evaluate \name{} on memory-efficient Llama ~\citep{touvron2023llama} pre-training tasks, using the standardized settings of \citep{Zhao2024GaLoreML}, which has been adopted by many recent works \citep{Zhao2024GaLoreML, chen2024fira, zhang2024q, huang2024galore, zhu2024apollo}. We consider models with 60M, 130M, 350M, and 1.3B parameters, all trained on the C4 dataset~\cite{2020t5} using an effective batch size of 130K tokens. Following the setup of \cite{Zhao2024GaLoreML}, \name{} is applied to all linear modules in both attention and MLP blocks. Training uses BF16 by default unless specified, see \Cref{app: code}. The other evaluation settings follows \citet{Zhao2024GaLoreML}.

\begin{table*}[h]
    \centering
    \caption{\small{Comparison with Adam and its memory-efficient low-rank variants on pre-training various sizes of LLaMA models on C4 dataset. Validation perplexity is reported, along with a memory estimate of the total of parameters and optimizer states based on BF16 format. Baseline results are directly taken from the official numbers reported in \citet{Zhao2024GaLoreML, zhu2024apollo}, as they shares exactly the same setup. Note that for Adam we report both the official results from \citet{Zhao2024GaLoreML} and our reproduced result. } }
    \label{tab: llm}
    \resizebox{\textwidth}{!}{\begin{tabular}{lcccc}
    \toprule
               & \textbf{60M} & \textbf{130M} & \textbf{350M} & \textbf{1.3 B} \\
    \midrule
    Adam & 33.02 (0.32G) & 24.44 (0.75G) & 19.24 (2.05G) & 16.44 (7.48G) \\
    Adam (cited) &  34.06 (0.32G)   & 25.08 (0.75G) & 18.80 (2.05G) & 15.56 (7.48G) \\
    \midrule 
    \name{}$^\ddag$ & 30.59 (0.23G) & \textbf{22.61 (0.43G)} & \textbf{16.63 (0.93G)} & \textbf{13.56 (2.98G)} \\
    \name{}$^\dag$ & \textbf{30.00 (0.23G)} & 22.83 (0.43G) & 17.14 (0.93G) & 14.42 (2.98G) \\
    \name{}-0 & 32.28 (0.23G) & 24.13 (0.43G) & 18.22 (0.93G) & 15.13 (2.98G) \\
    Momentum$+$$\gradwhite$ & 31.6 (0.27G) & 24.59 (0.59G) & 19.30 (1.49G) & 16.08 (5.23G) \\
    \midrule
    Galore & 34.88 (0.26G) & 25.36  (0.57G) &18.95 (1.29G) & 15.64 (4.43G)\\
    Apollo-mini &  31.93 (0.23G) & 23.53 (0.43G)& 17.18 (0.93G)& 14.17 (2.98G) \\
    Apollo &  31.55 (0.26G) & 22.94 (0.57G) & 16.85 (1.29G) & 14.20 (4.43G) \\
    Low-Rank & 78.18 (0.26G) & 45.51 (0.57G) & 37.41 (1.29G) & 142.53 (4.43G) \\
    LoRA & 34.99 (0.36G) & 33.92 (0.80G) & 25.58 (1.76G) & 19.21 (6.17G) \\
    ReLoRA & 37.04 (0.36G) & 29.37 (0.80G) & 29.08 (1.76G) & 18.33 (6.17G) \\
    \bottomrule
    \name{}$^\ddag$ speed up vs Adam & 1.52 X & 1.6 X & > 2.3 X & > 2.4 X \\
    $r$ of low-rank methods & 128 & 256 & 256 & 512 \\
    Training Steps & 10K & 20K & 60K & 100K \\ %
    \bottomrule
    \end{tabular}}
\end{table*}

\begin{figure}[H]
    \centering
    \begin{subfigure}[b]{0.45\textwidth}
        \centering
        \includegraphics[width=\textwidth]{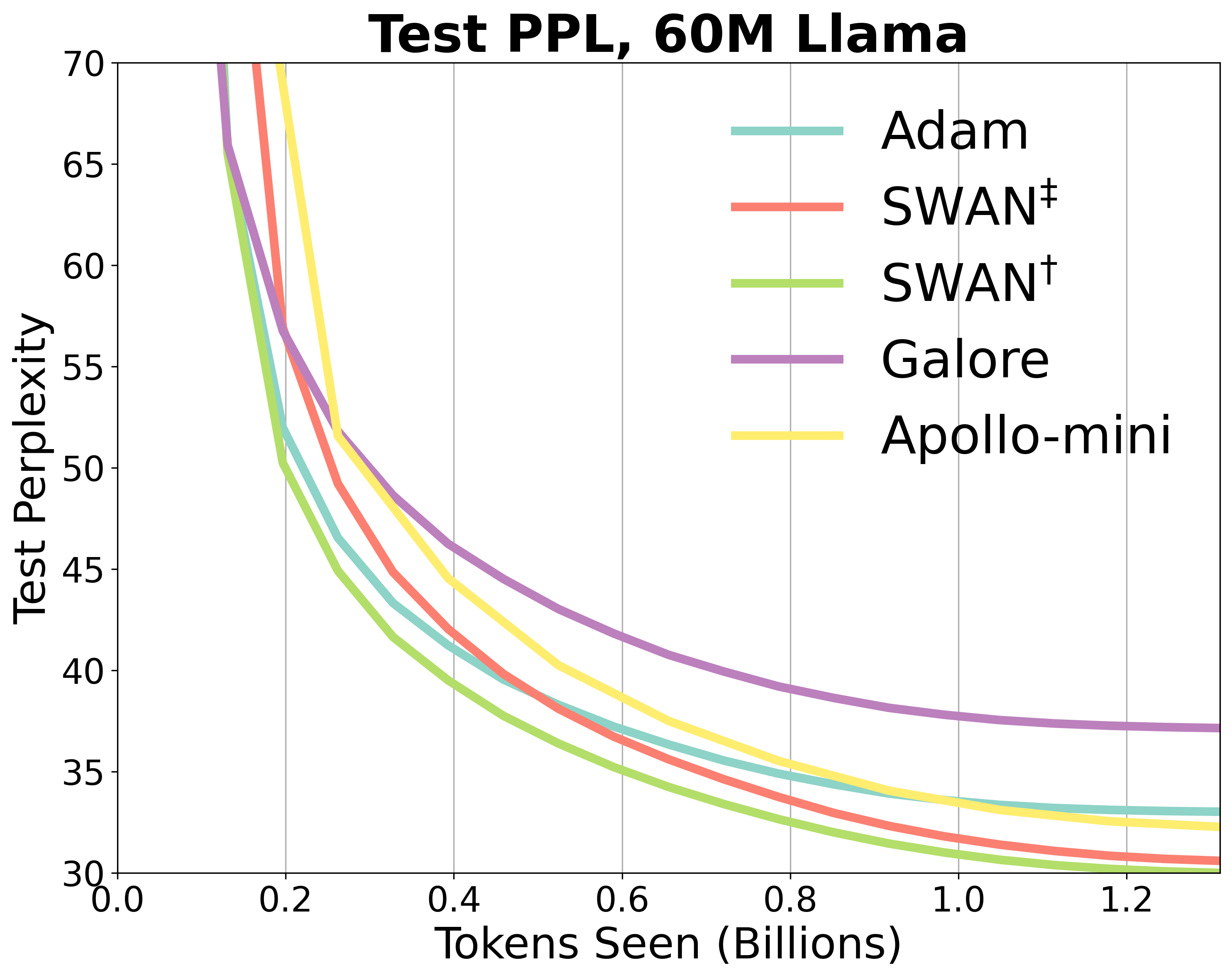} 
        \caption{60M}
    \end{subfigure}
    \hfill
    \begin{subfigure}[b]{0.45\textwidth}
        \centering
        \includegraphics[width=\textwidth]{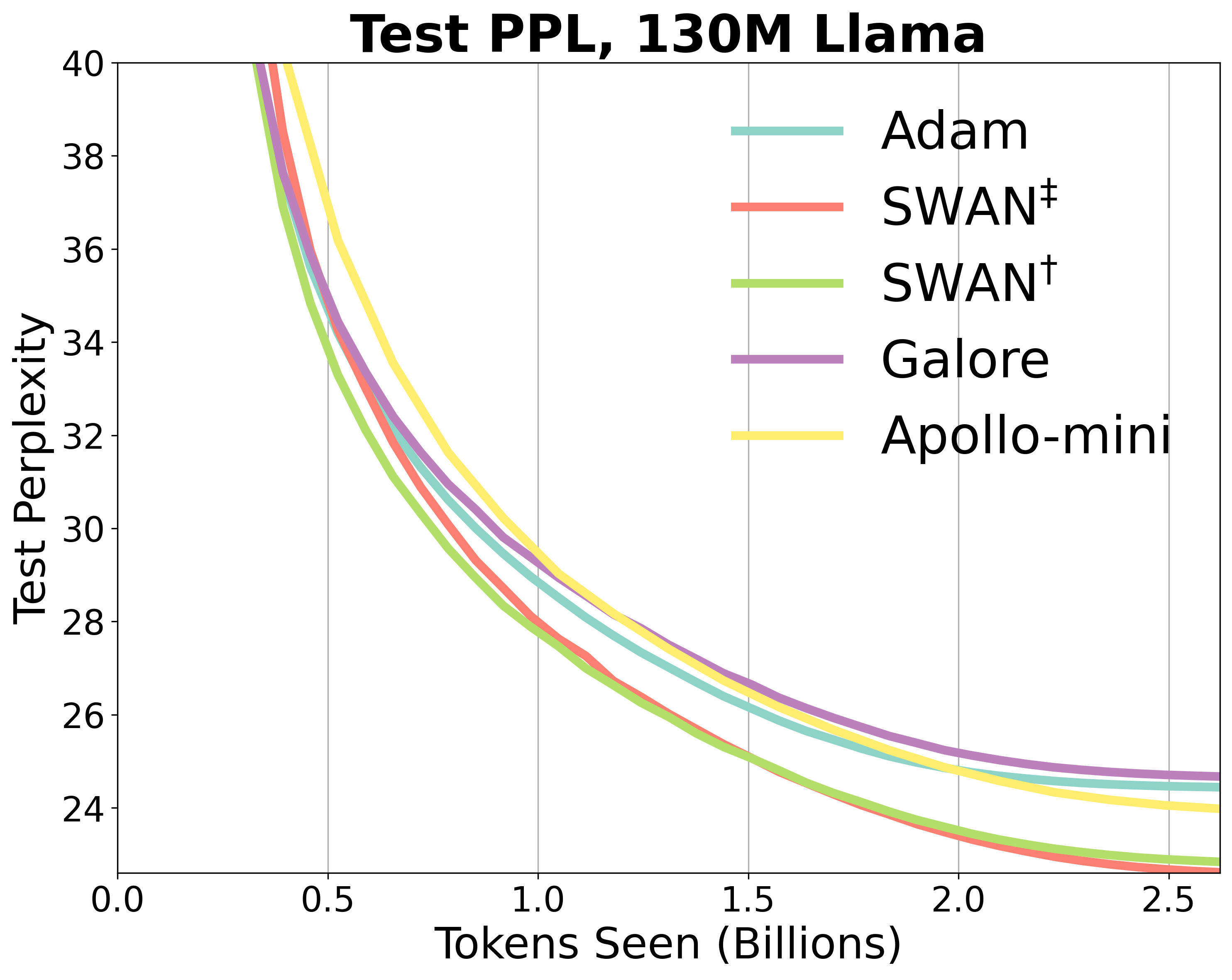} 
        \caption{130M}
    \end{subfigure}
    \hfill
    \vskip\baselineskip
    \begin{subfigure}[b]{0.45\textwidth}
        \centering
        \includegraphics[width=\textwidth]{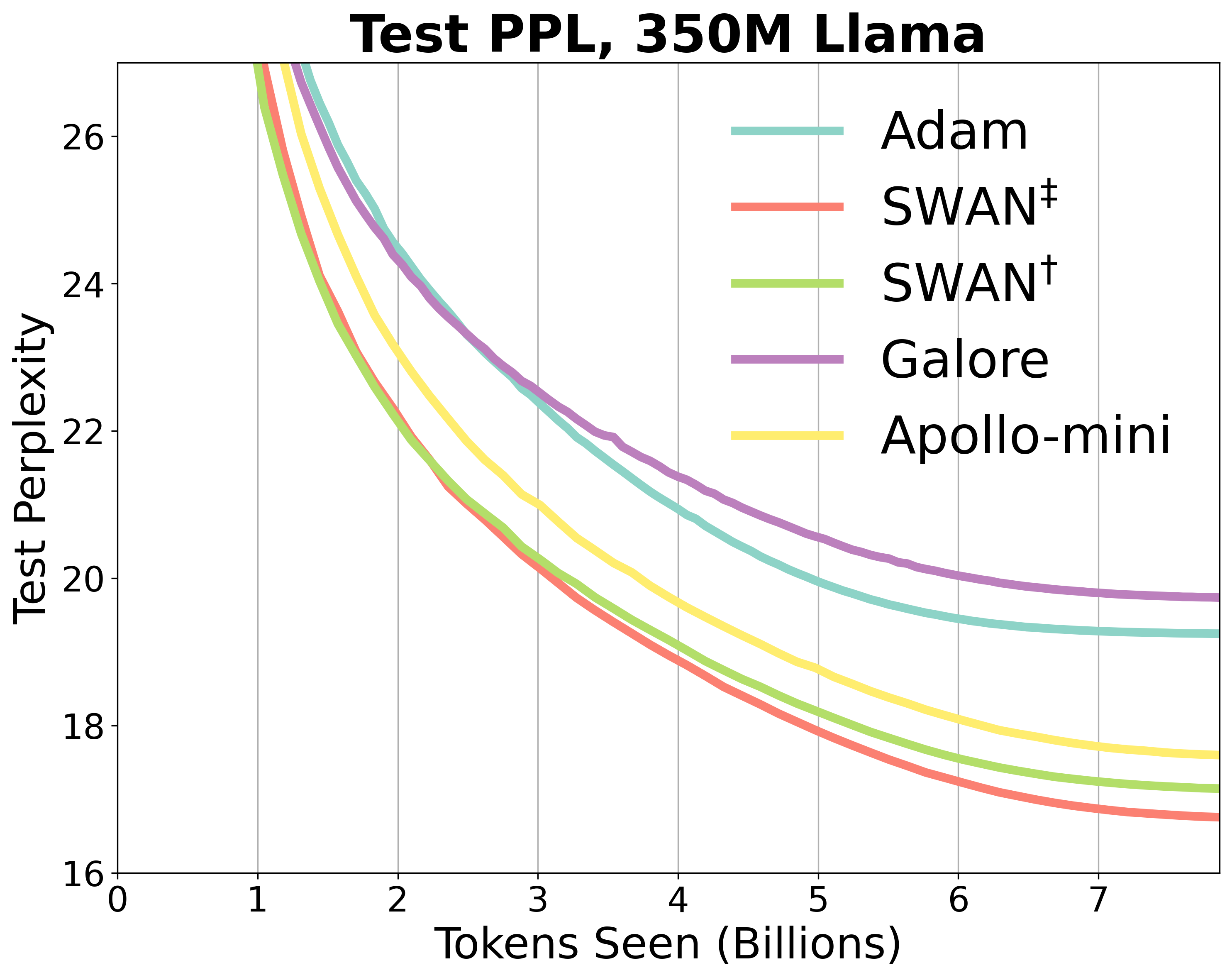} 
        \caption{350M}
    \end{subfigure}
    \hfill
    \begin{subfigure}[b]{0.45\textwidth}
        \centering \includegraphics[width=\textwidth]{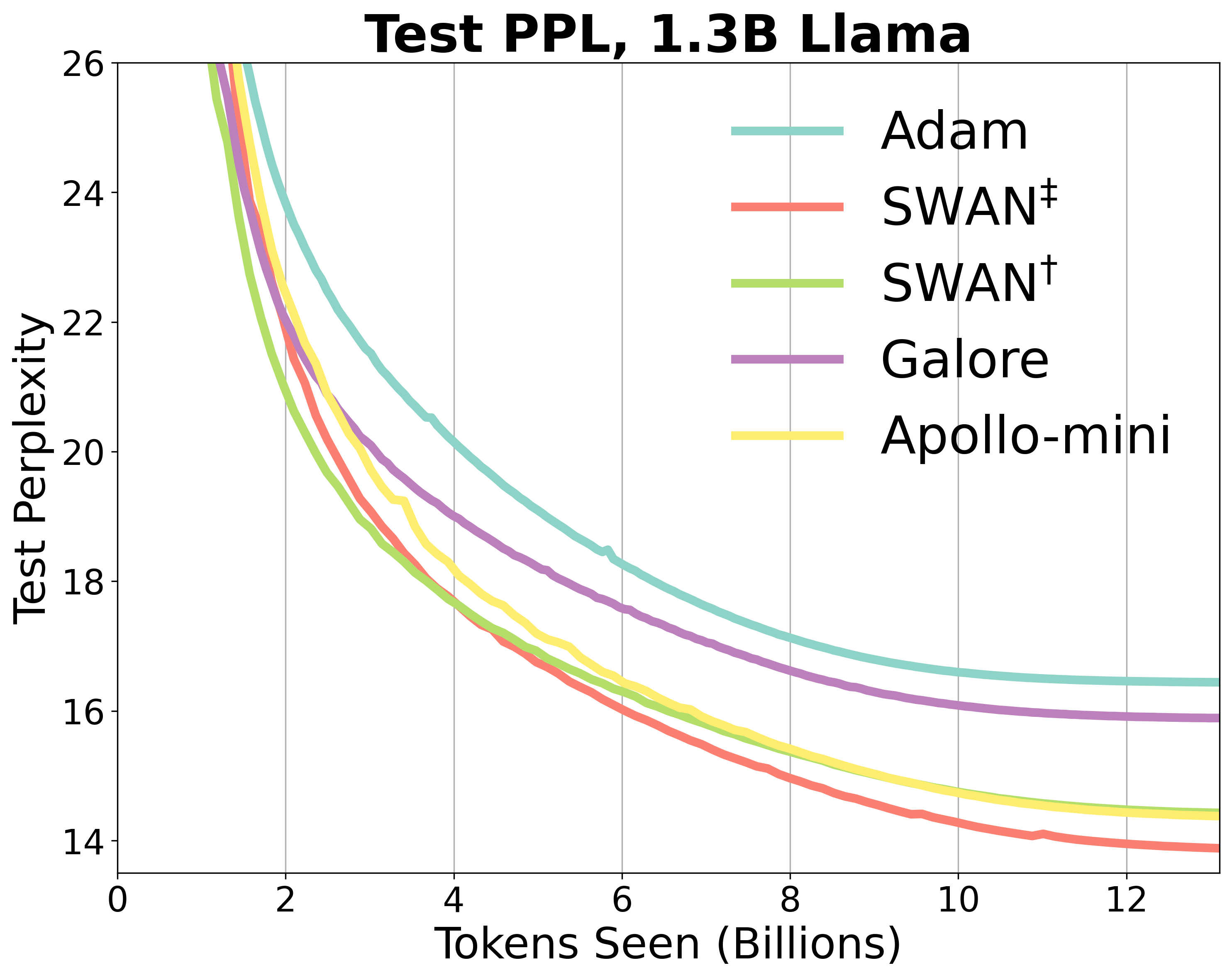} 
        \caption{1.3B }
    \end{subfigure}
    \caption{ Comparison of convergence rate of different methods on LLM pretraining tasks. The training curves of Adam, Galore and Apollo-mini are reproduced according to the opensource code of \cite{Zhao2024GaLoreML, zhu2024apollo}. We further compare to their official results in \Cref{tab: llm}.}
    \label{fig: llm}
\end{figure}


\paragraph{\name{} optimizer} We consider three variants: \textbf{\name{}-0}, which aims to show-case the robustness and effectiveness of our method out-of-the-box, with almost no tuning. It uses naive NS-iteration for whitening, disabled learning rate warmup, and use similar learning rates optimized for Adam. \textbf{\name{}}$^\dag$, the vanilla version of our method, in which we enabled learning rate warmup, and allowed the use of optimized learning rates that largely differ from Adam. Finally, \textbf{\name{}$^\ddag$}, the strongest and most efficient variant that employs the proposed NSDS scheme for fast whitening (\cref{sec: practical}). For all \name{} variants, we adopt a \emph{lazy-tuning approach} (hyperparameters are set without extensive search) to reduce the possibility of unfair performance distortion. Notably, for \name{}$^\ddag$, we share the same hyperparameters across all model sizes.  Full details can be found in \Cref{app: code}.

\paragraph{Baselines} We focus on comparing \name{} with \textbf{Adam} \citep{adam} and other memory-efficient optimizers. All baseline results are directly quoted from corresponding papers as they all share the same standard setup. The baselines include \textbf{Adam} \citep{adam}, which is a standard choice for training large models; and \textbf{Galore} \cite{Zhao2024GaLoreML}, a memory-efficient Adam variant with low-rank gradient projections. We also consider \textbf{Low-Rank} \citep{kamalakara2022exploring}, a low-rank factorization approach ($\rmW = \rmB\rmA$); and \textbf{LoRA} \cite{hu2021lora}, which applies the LoRA method for pre-training as in \cite{Zhao2024GaLoreML}. Additionally, we include \textbf{ReLoRA} \cite{Lialin2023ReLoRAHT}, a full-rank extension of LoRA with parameter merging, and \textbf{Momentum$+$$\gradwhite$}, which applies Newton-Schulz $\gradwhite$ on top of momentum instead of $\gradnorm$; this is equivalent to \textbf{Muon} \cite{jordan2024muon} with Nesterov acceleration turned off. Finally, we compare with \textbf{Apollo-mini} and \textbf{Apollo} \citep{zhu2024apollo}. Full details can be found in \Cref{app: code}.


\paragraph{Fair Comparison} We follow \citet{kaddour2024no} by allocating the same total training budget (in terms of  tokens) for all methods. By the end of training, all methods decay to the same target learning rate, which is $10\%$ of the initial learning rates. This approach eliminates an often over-looked evaluation bias, where optimizers with earlier learning rate cool-downs tend to have an unfair advantage due to artificially fast convergence caused by learning rate decay \cite{kaddour2024no}.

\paragraph{Results} As shown in \Cref{tab: llm} and \Cref{fig: preview}, all \name{} variants achieve strong performance compared with baselines, requiring the lowest memory consumptions comparable to SGD. Across all models, \name{}-0 surpasses the performance of the Adam and Momentum-$\gradwhite$ (muon-like) baseline in terms of validation perplexities. \name{}$^\dag$ further delivers a comparable performance to Apollo series; and finally, \name{}$^\ddag$ with NSDS scheme delivers the strongest, SOTA-level performance under this setup. Notably, on the 350M and 1.3B models, all \name{} variants reaches at least 2$\times$ speedup (in steps or tokens) relative to Adam (\Cref{fig: preview} and \Cref{fig: speedup_scaling} (a)). Finally, we observe that in general, \name{}$^\ddag$ has a slower convergence speed in the early stage than \name{}$^\dag$. However, it offers strong tail convergence in return and surpasses \name{}$^\dag$ in terms of final performance.

\subsection{Memory Efficiency and Throughput Analysis} \label{sec: memory_exp}


\begin{table}[]
\centering
\caption{Raw and effective throughput analysis, under different model parallelization (MP) settings. }
\label{tab: throughput}
\begin{tabular}{l|l}
\hline
      Method   & Raw / eff. throughput (1B) \\ \hline
Adam & 117872 / 117872 (tokens/s)          \\
\midrule
\name{}$^\dag$ w/o MP                     &     $\:$$\:$58600 / 117200 (tokens/s)       \\
\name{}$^\ddag$ w/o MP     & \textbf{107808} / \textbf{258739}  (tokens/s)    \\
\midrule
\name{}$^\dag$ w/ MP              & 114160 / 228320 (tokens/s)         \\
\name{}$^\ddag$ w/ MP     & \textbf{115872} / \textbf{278092}  (tokens/s)    \\ \hline
\end{tabular}
\end{table}

\paragraph{Memory Footprint} We compare \name{}, Adam, and Galore on a single A100 GPU. Unlike \cite{Zhao2024GaLoreML, zhu2024apollo}, which report layer-wise training memory usage, we measure total end-to-end memory consumption under full-model training using a batch-size of $1$ for 1.3B, 7B, and 13B models. As shown in \Cref{fig: preview} (c), \name{}’s memory usage is on par with SGD, providing nearly a 50\% reduction in total memory. If we further incorporate the per-layer training technique of  \citep{Zhao2024GaLoreML, lv2023full}, \name{} further achieves $\approx 70\%$ total memory reduction on top of $\approx 100\%$ reduction on optimizer states. This underlines the benefit of the stateless design.


\paragraph{Effective Throughput} We assess throughput when training a 1.3B LLama model on 8 A100 GPUs with a batch size of 130K. All gradient processing are done in BF16. We use two metrics: \emph{raw throughput}: number of tokens processed per second. \emph{Effective throughput}: raw throughput adjusted by \name{}'s token efficiency relative to Adam. These metrics evaluate the impact of different $\gradwhite$ schemes on training speed, and also account for the fact that some optimizers make more effective use of training tokens. As shown in \Cref{tab: throughput}, \name{}$^\dag$ with naive NS $\gradwhite$ requires model parallelization (where NS of different tensors are distributed to different GPUs) to achieve competitive throughput. Without model parallelization, the throughput of \name{}$^\dag$ is 50 \% lower than Adam. With the proposed NSDS scheme, \name{}$^\ddag$ achieves an raw throughput comparable to Adam, \emph{without} any need for model parallelization. Consequently, \name{}$^\ddag$ exhibits a 2 $\times$ higher effective throughput than Adam.


\subsection{Is the Speed-up Multiplicative or Additive?}\label{sec: speedup}

A key question regarding speedup factors is whether the improvement over Adam is \textit{multiplicative} or \textit{additive}. A multiplicative speedup implies that the optimizer’s relative advantage remains proportionally consistent over time, while an additive speedup suggests a less desired constant step advantage. To investigate this,  we take the \name{}-0 setup as an example, and address this question using two plots, the \textit{speed-up ratio comparison} and the \textit{perplexity comparison} (\Cref{fig: speedup_scaling}), across model sizes.

\begin{figure}[t!]
    \centering
    \begin{subfigure}[b]{0.49\textwidth}
        \centering
        \includegraphics[width=\textwidth]{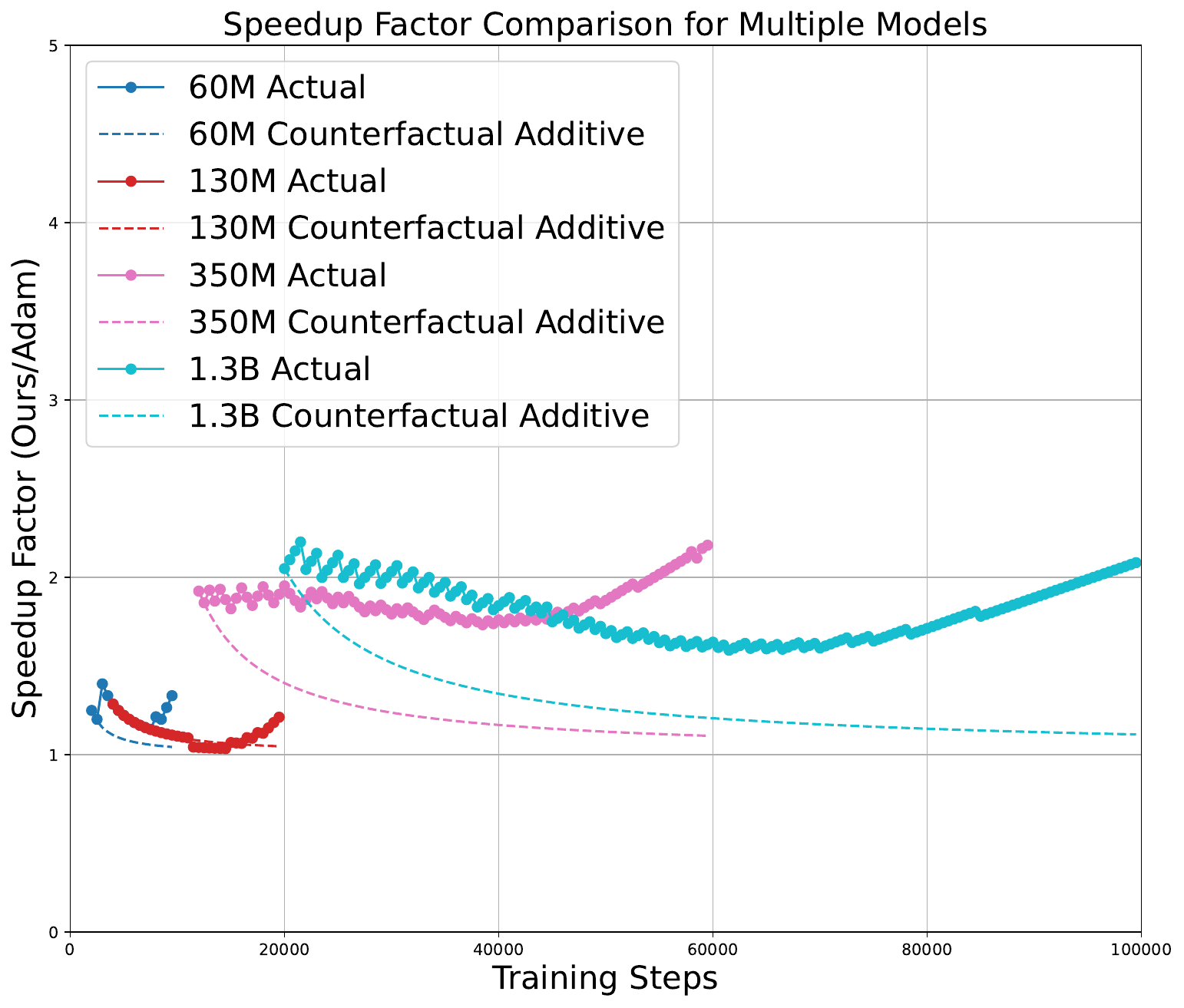} 
        \caption{Speedup factor of \name{} vs Adam training steps}
    \end{subfigure}
    \hfill
    \begin{subfigure}[b]{0.49\textwidth}
        \centering
        \includegraphics[width=\textwidth]{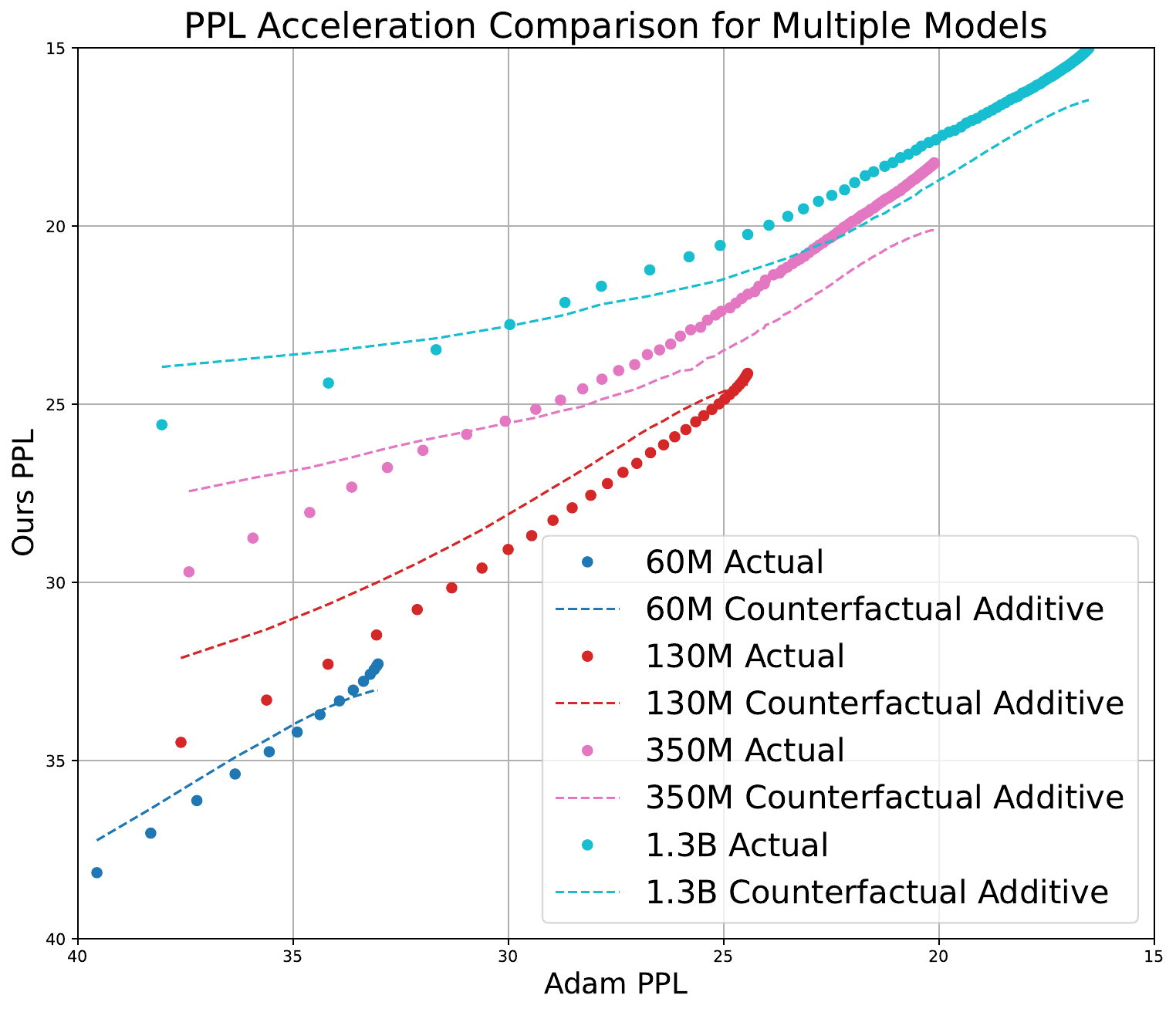} 
        \caption{Ours PPL vs Adam PPL scatter plots}
    \end{subfigure}
    \caption{Comparative analysis of \name{} and Adam optimizers: speedup ratios and perplexity metrics across various model sizes. \textbf{(a)} shows how \name{} reduces the number of training steps needed to achieve the same evaluation perplexity as Adam for models ranging from 60M to 1.3B parameters. A speedup ratio greater than one indicates that \name{} reaches target PPL values faster than Adam. \textbf{(b)} presents a direct comparison of perplexity scores between \name{} and Adam. In both plots, we also provide counterfactual additive curves (dashed lines) modeling baselines corresponding to constant step advantages. Together, these plots highlight the nature of \name{}'s speedup over Adam across different model scales.   }
    \label{fig: speedup_scaling}
\end{figure}

\paragraph{Speedup Ratio Definition}
We define the \textit{speedup ratio} $R(P)$ for a given perplexity ($\text{PPL}$) threshold $P$ as the ratio of the number of training steps Adam requires to reach a specific evaluation perplexity ($\text{PPL}$) to the number of required steps for \name{}:
\begin{equation}
    R(P) = \frac{S_{\text{Adam}}(P)}{S_{\text{\name{}}}(P)},
\end{equation}
where $S_{\text{Adam}}(P)$/$S_{\text{\name{}}}(P)$ are the training steps required by Adam/\name{} to reach perplexity $P$.

\paragraph{Counterfactual Additive Curve Estimation}
To test whether the speedup is additive, we compute a \textit{counterfactual additive curve} by assuming \name{} gains a fixed step advantage $\Delta$ over Adam in early training (approximately the first 10\%--20\% of total steps):
\begin{equation}
    \Delta = \frac{1}{N} \sum_{i=1}^{N} \left(S_{\text{Adam}}(P_i) - S_{\text{\name{}}}(P_i)\right),
\end{equation}
where $N$ is the number of $\text{PPL}$ thresholds considered. We then use $\Delta$ to define the counterfactual additive speedup ratio:
\begin{equation}
    R_{\text{additive}}(P) = \frac{S_{\text{Adam}}(P)}{S_{\text{Adam}}(P) - \Delta},
\end{equation}
and the counterfactual additive perplexity estimate:
\begin{equation}
    \text{PPL}_{\text{additive}}(S) = \text{PPL}_{\text{Adam}}\left(S + \Delta\right).
\end{equation}
This represents the expected perplexity of \name{} if it only outpaces Adam by fixed $\Delta$ steps.

\paragraph{Results}
\Cref{fig: speedup_scaling} compares \name{}’s actual performance against the counterfactual additive curves. If \name{}’s actual curves exceed these additive estimates, it indicates a tendency towards a multiplicative speedup, instead of the additive advantage. We summarize the observations across model sizes as: 1) for smaller models (60M and 130M), the actual speedup trajectories align closely with the additive baseline, indicating a primarily additive speedup; 2) for larger models (350M and 1.3B), the actual curves rise noticeably above the additive estimates, suggesting a multiplicative speedup. This indicates that \name{} yields increasing efficiency gains as model size grows.

\subsection{Ablation Studies} \label{sec: ablation}

We take \name{}-0 from the \name{} series as an example and conduct the following ablation studies. All \name{} optimizer mentioned in this section specifically refers to the \name{}-0 setting.

\paragraph{Effect of $\gradnorm$ and $\gradwhite$ on Performance}
We consider six ablation settings: (1) \name{}(full), (2) \name{} ($\gradnorm$ only), (3) \name{} ($\gradwhite$ only), (4) Adam (full), (5) Adam (momentum only), and (6) Adam (second moment only). As shown in \Cref{fig: ablation} (a), both $\gradnorm$ and $\gradwhite$ contribute to \name{}’s final performance. Removing either results in performance degradation. Similarly, Adam also requires all moments for optimal performance.

\begin{figure}[H]
    \centering
    \begin{subfigure}[b]{0.325\textwidth}
        \centering
        \includegraphics[width=\textwidth]{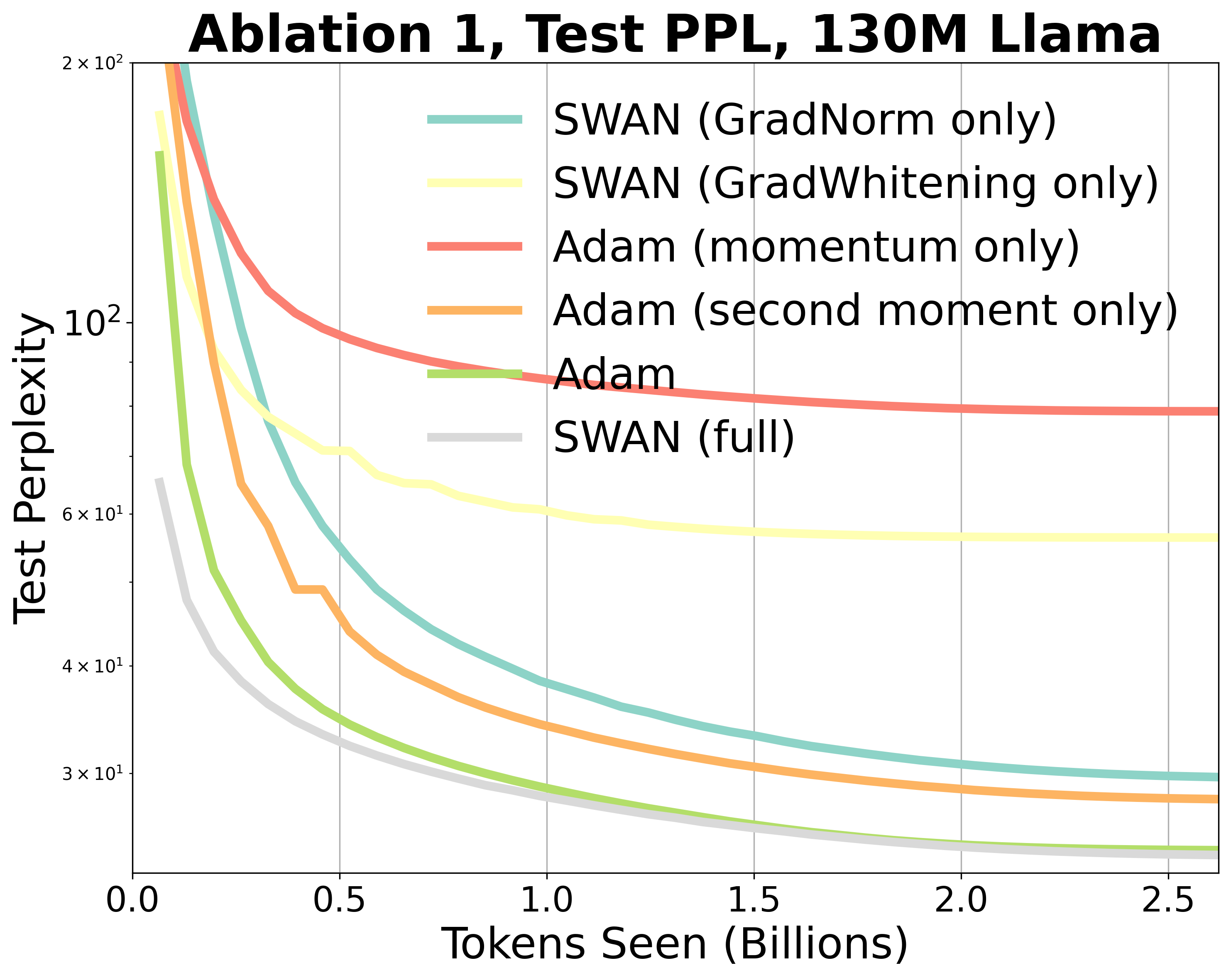} 
        \caption{}
    \end{subfigure}
    \hfill
    \begin{subfigure}[b]{0.325\textwidth}
        \centering
        \includegraphics[width=\textwidth]{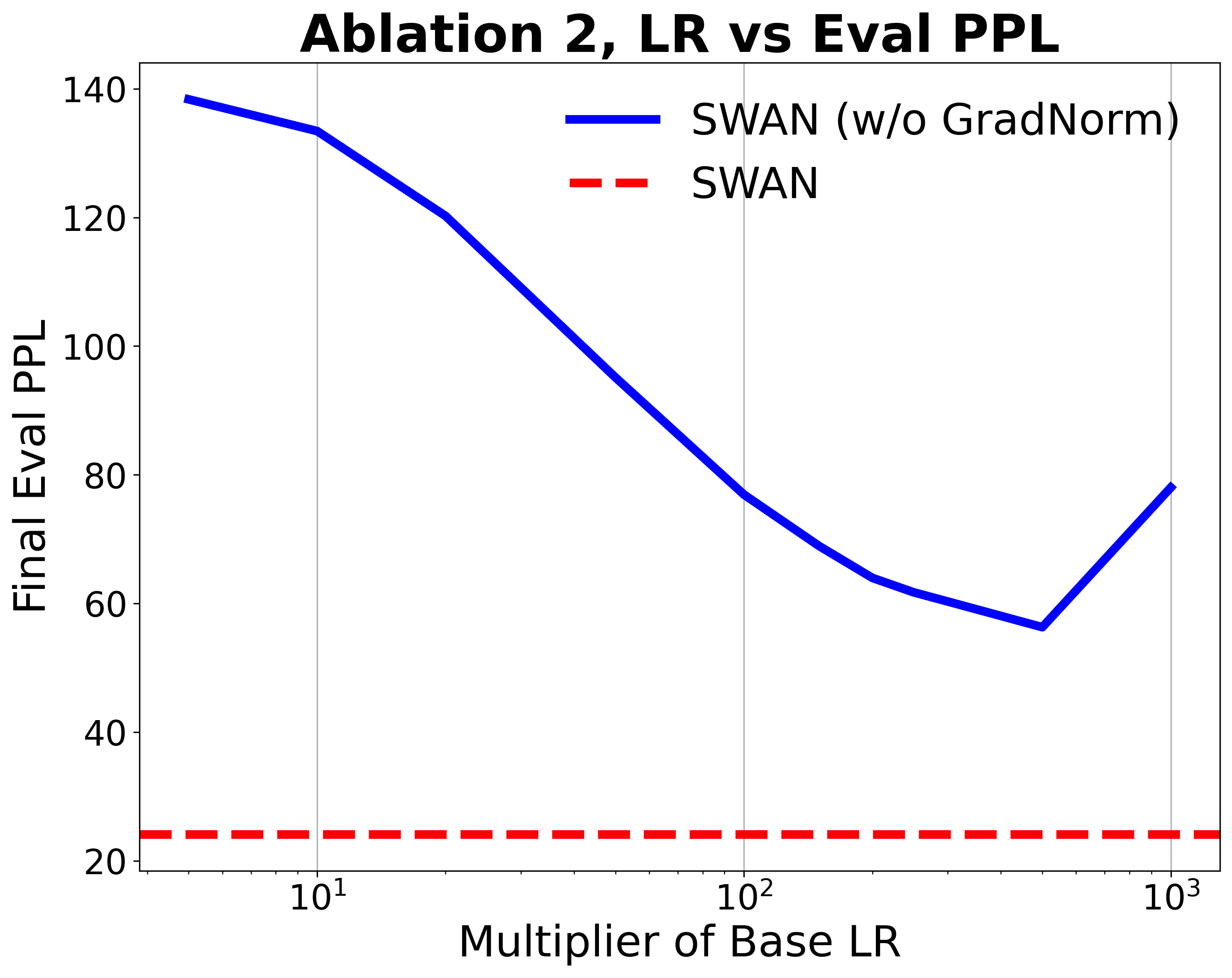} 
        \caption{}
    \end{subfigure}
    \hfill
    \begin{subfigure}[b]{0.325\textwidth}
        \centering
        \includegraphics[width=\textwidth]{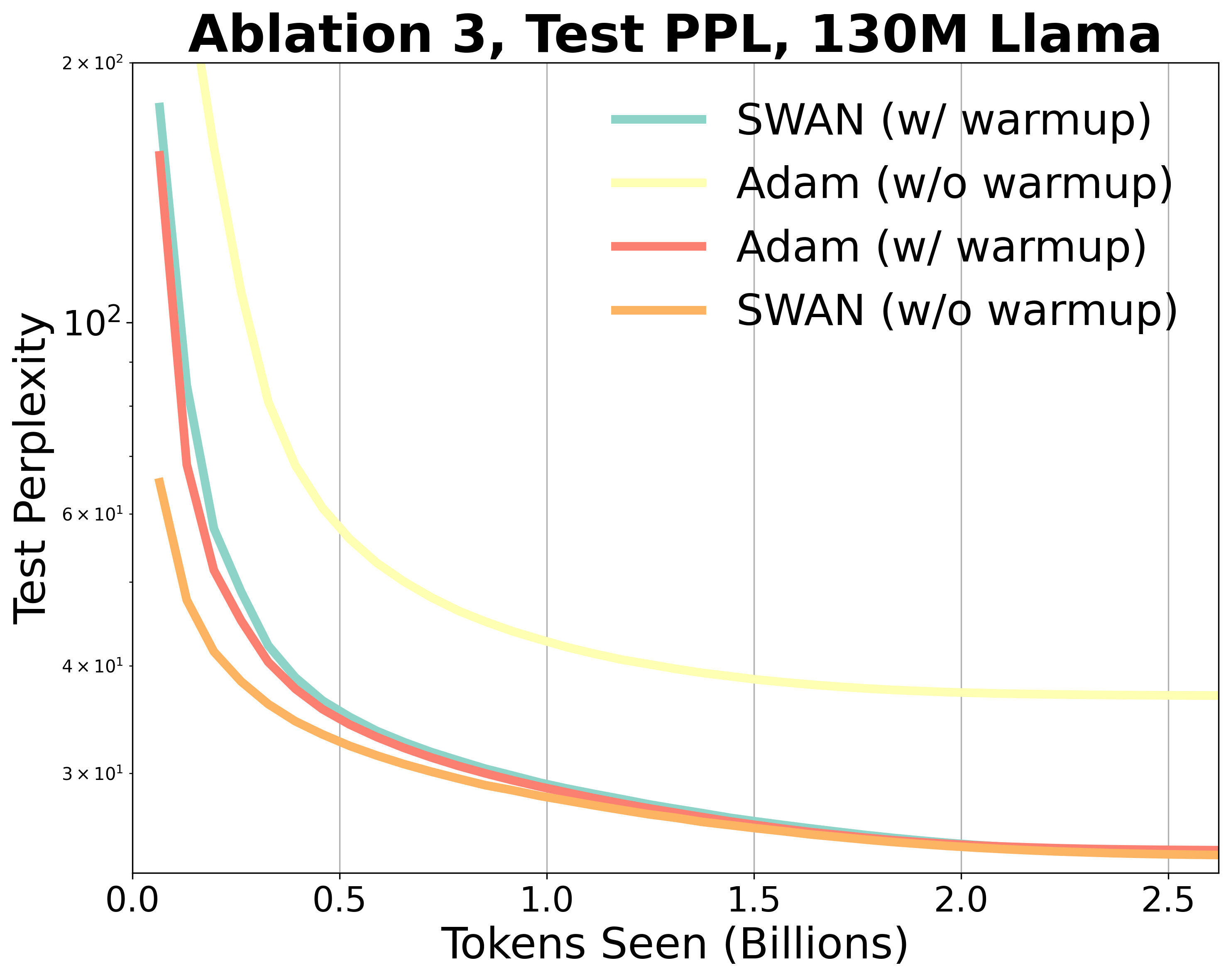} 
        \caption{}
    \end{subfigure}
    \caption{ Ablation studies on 130M model. (a) Ablation on the contribution of each components in \name{} and Adam. (b) Ablation on removing $\gradnorm$ and compensate with larger learning rates. (c) Ablation on the effect of learning rate warm-ups. }
    \label{fig: ablation}
\end{figure}

\paragraph{Does \name{} Succeed Only by Increasing Effective Learning Rates though $\gradnorm$?}
To answer this question, we remove the $\gradnorm$ operator from \name{} and run a learning rate sweep. Starting with the default learning rate of full \name{}, we apply multipliers from $1$ to $10^3$. In \Cref{fig: ablation} (b), the blue line shows the final validation perplexity for \name{} without $\gradnorm$ at different learning rates multipliers, while the dashed red line represents the performance of full \name{} at the default learning rate. Although raising the learning rate can improve the performance of \name{} without $\gradnorm$, the gap to full \name{} remains large. 
This indicates that $\gradnorm$’s gradient noise stabilization is essential and cannot be replaced simply by increasing the learning rate.

\paragraph{How does warm-up affect the performance?}
\Cref{sec: llm_exp} shows that \name{} can train with no warm-up phase even under a relatively large learning rate (0.001). Here, we compare Adam and \name{} with and without warm-ups. As seen in \Cref{fig: ablation} (c), \name{} without the warm-up phase gives better performance, and it still outperforms Adam under Adam’s own warm-up schedule. On the other hand, Adam's performance decreases drastically without a proper warm-up. These findings suggest that \name{} is more robust to warm-up schedule and can train effectively with or without it.

\section{Open Questions and Extensions}
In \Cref{sec: analysis} we analyzed the design choice of $\gradnorm$ and $\gradwhite$ from the perspective of LLM dynamics. However, the \emph{compositional effect} of these two operators is still an open question, which is not discussed in this work. Specifically, in \Cref{app: NS} we empirically show that when comparing with standard, standalone NS processing of the gradients, both $\gradwhite$ and (NSDS)-$\gradwhite$ of \name{} introduces significant changes, rotating the update by a relatively large magnitude. This indicates that additional efforts are needed to further explain the novel dynamics of \name{}, as well as the question of ``what makes a good preprocessing chain for gradients''.  

After the v1 version of this paper on arXiv, we posted a follow-up work of \name{} \citep{scetbon2025gradientmultinormalizationstatelessscalable}, in which we have shown that \name{} is a special case of a general framework called multi-normalized gradient descent (MNGD). MNGD aims at normalizing gradients according to multiple norms, generalizing the steepest descent viewpoint of~\cite{bernstein2024old} that recasts popular first-order optimizers as normalization of gradients under a single norm.
  
\section{Conclusion}  
We introduced \name{}, a stateless optimizer for LLM training that combines two well-known operators applied to raw gradients, $\gradnorm$ and $\gradwhite$, to stabilize the stochasticity of gradient distribution and neutralizing the local geometry of loss landscape, respectively. Through theoretical analysis and empirical evidence, we showed that \name{} reduces memory usage while achieving on-par or even better performance compared to Adam on LLaMA pre-training tasks. Notably, \name{} achieves $2\times$ speedups compared to Adam in terms of tokens processed when training 350M- and 1.3B-parameter models. These findings serve as a proof-of-concept that the stateless approach has the potential to serve as a practical and efficient alternative to other optimizers that require tracking internal states. Future work may explore other design choices for stateless optimizers and further expand \name{}'s applicability to other complex training regimes beyond standard LLM pre-training.


\clearpage
\newpage
\bibliography{bibliography}
\bibliographystyle{iclr2025_conference}

\newpage

\appendix

\setlength{\cftbeforesecskip}{1em}
\setlength{\cftbeforesubsecskip}{0.5em}
\setlength{\cftbeforesubsubsecskip}{0.5em}


\clearpage

\section{Desired properties of adaptive optimizers}
\label{app: desired property}
There is a rich literature on understanding adaptive methods' inner workings and unreasonable effectiveness. Using Adam as an example, we first summarize from the literature below the key desired properties of stateful adaptive optimizers that contribute to their empirical success: \emph{gradient smoothing}, \emph{gradient invariance}, and \emph{gradient whitening}. Then we discuss how these understandings will leads to the design of \emph{stateless} adaptive optimizers.

\paragraph{Gradient Smoothing.} Under the stochastic optimization setting, mini-batch sampling introduces heterogeneous distribution shift on the gradient distribution: $\rmG^{(t)} =  \E[\rmG^{(t)}] + \mepsilon^{(t)}$, where $\mepsilon^{(t)}$ is time-heterogeneous noise induced by mini-batch sampling. 
While $\mepsilon^{(t)}$ helps SGD escapes local optima \citep{jastrzkebski2017three, zhu2018anisotropic}, 
the \emph{covariate shift} of $\mepsilon^{(t)}$ over time also present challenges to learning as the model needs to adjust and compensate for this shift, especially under the emergence of heavy tailed gradient distributions \citep{zhang2020adaptive} \footnote{Such shift cannot be removed by forward covariate-shift reduction architectures such Layer Norm, as it is only invariant to global scaling and re-centering, such as $\rmW^{(t)} = \delta \rmW^{(t)} + \bm{\gamma} \mathbf{1}^\top$ for some scalar $\delta$ and incoming vector shift $\bm{\gamma}$ \citep{Ba2016LayerN}.}. 
%
Following this viewpoint, it has been proven that momentum reduces the influence of noises for SGD \citep{cutkosky2020momentum, crawshaw2022robustness}. Therefore we hypothesis that the first moment estimate $\vm^{(t)}$ of Adam also effectively stabilizes gradient distribution and reduces effect of $\mepsilon^{(t)}$. This smoothing stabilizes the variance caused by noisy stochastic gradients across time.

\paragraph{Gradient Invariance.} More recently it has also been identified \citep{kunstner2023noise, kunstner2024heavy} that the major factor contributing to the performance gap between SGD and Adam might lie in Adam's \emph{Sign-descent}-like nature \citep{bernstein2018signsgd, crawshaw2022robustness, lion}. Intuitively, Adam without bias correction under $\beta_1 = 0$ and $\beta_2 = 0$ is equivalent to signed gradient descent ($\Delta \rmW = \text{sign} (\rmG)$). Indeed, the performance of Adam can be closely reproduced \citep{kunstner2023noise, crawshaw2022robustness} or even surpassed \citep{lion} by variants of signed descent with momentum. Apart from sign-based methods, evidence on performance boost using gradient clipping/normalization was also discussed in the context of understanding Adam \citep{zhang2020adaptive}. Therefore, we hypothesize that one of the key properties of Adam is that it offers \emph{invariance over certain transformations} on gradients. Particularly, the original Adam is invariant to diagonal rescaling of the gradients \citep{adam}; the signed gradient method is invariant to \emph{any} scaling that preserves the sign of gradients; and the clipped SGD variant is invariant to extreme gradient magnitude spikes. 

\paragraph{Gradient Whitening.} Finally, we argue that the empirical success of adaptive methods also lies in that they model the curvature by first-order information. This is realized by the second moment estimate $\vnu^{(t)}$, which approximates the diagonal of the Fisher information matrix \citep{adam, hwang2024fadam}; helping to counteract local curvatures of the problem. Specifically, Adam computes a trailing estimation of the diagonal coefficients of the Fisher matrix $\rmF = \E[\vg \vg^\top ]$ by tracking $\hat{\rmF} =\text{diag}(\rmF) =\text{diag}[\E[\vg^2]]$, where $\vg = \mathtt{vec}(\rmG)$ is the vectorized gradient. Interestingly, instead of preconditioning the first moment as $ \hat{\rmF}^{-1} \mathtt{vec} (\vm)$, Adam uses a whitening-like preconditioned update $\hat{\rmF}^{-\frac{1}{2}} \mathtt{vec} (\vm)$, suggesting an \emph{element-wise} approximate whitening of the gradient. 
%
%
It has been shown that such element-wise whitening leads to diagonal approximation to inverse Hessian $\hat{\rmF}^{-\frac{1}{2}} \approx \text{diag}(\rmH^{-1})$ \citep{molybog2023theory}. Recent empirical studies show that Adam biases optimization trajectories towards regions where the condition number of Hessian is low \citep{jiang2024does}. Therefore, we hypothesize that Adam approximately whitens the gradients element-wise, leading to well-conditioned regions.

\section{Acceleration of Newton-Schulz iteration via diagonal substitution} \label{app: NS}

\subsection{Algorithnms}

Computing $\gradwhite$ exactly can be expensive, as it involves solving the matrix square-root inverse. One option is to directly apply the Newton-Schulz variant of decorrelated batch normalization \citep{song2022fast, li2018towards, huang2019iterative}, which allows a more GPU-friendly estimation. This is given by \citep{song2022fast, li2018towards}:
\begin{align*}
  \begin{cases}
   \rmY_{k+1} = \frac{1}{2} \rmY_{k} (3 \rmI - \rmZ_k \rmY_k) \\
   \rmZ_{k+1} = \frac{1}{2} (3 \rmI - \rmZ_k \rmY_k) \rmZ_k
  \end{cases}
\end{align*}
where $\rmY_0 = \rmG \rmG^\top $ 
, $\rmZ_0 = \rmI$. At convergence, $\mathtt{GradWhitening}(\rmG) = \rmZ \rmG$ (\Cref{alg:optimizer2}). However, estimating $(\rmG\rmG^\top)^{-1/2}$ with NS requires $\mathcal{O}(m^3)$ (assuming $m < n$) complexity. Here, we propose a heuristic scheme that has $\mathcal{O}(m^2)$ complexity to estimate square-root inverse:
\begin{align*}
  \begin{cases}
   \rmY_{k+1} = \frac{1}{2} \rmY_{k} \text{Diag}(3 \rmI - \rmZ_k \text{Diag}(\rmY_k))) \\
   \rmZ_{k+1} = \frac{1}{2} (3 \rmI - \text{Diag}(\rmZ_k) \rmY_k) \text{Diag}(\rmZ_k) 
  \end{cases}
\end{align*}
where $\text{Diag}(\cdot)$ returns a diagonal matrix that has the same diagonal elements as the input matrix. Basically, whenever we encounter matrix multiplication in NS iterations, we replace one of them by its diagonal approximation. We refer to this as the \emph{NS with diagonal substitution} (NSDS) scheme. 

Note that in the above standard presentation we have fixed that both NS and NSDS uses coefficients $= 0.5$ for $\rmY$ updates and $\rmZ$ updates, respectively. In practice we may further tune these coefficients to compensate for short number of iterations (usually under 10).

\subsection{Experiment: LLM Gradient Condition Number Reduction}

\paragraph{Setup} In this synthetic experiment, we assessed the effectiveness of two whitening methods, Newton-Schulz and the proposed Newton-Schulz with diagonal substitution (NSDS) scheme, on gradient matrices obtained from LLM training. As discussed in \Cref{sec: update_rules}, the exact $\gradwhite$ operator results in matrices with optimal condition number ($=1$); therefore, we hereby investigate the matrix condition numbers of the processed gradients obtained from different methods. Specifically, both methods use 5 NS iterations with NS step size optimized. We train a 130M LLama model following the architecture setting of \Cref{sec: llm_exp} on randomly generated sequences for 1000 steps, and take the MLP weights of a middle layer (we take the fifth layer without loss of generality) and use different methods to whiten the gradient matrices. We consider three methods: standard NS; NSDS; and \name{} with NSDS (that is, composing $\gradnorm$ with NSDS-$\gradwhite$).  At each training step, the condition number reduction ratio of different method was calculated for both whitening methods (the higher the better). Note that for all methods, the gradients have been pre-normalized by its norm.

\paragraph{Results} Results are shown in \Cref{fig: condition_number} \textbf{(a)}. We notice that NSDS alone (orange curve) is not sufficient to reach a good condition number reduction ratio. However, when combined with $\gradnorm$ (i.e., SWAN with NSDS, the green curve), its performance started to catch up and even outperform the standard NS method after 500 training steps. This show the effectiveness of the proposed scheme. One potential caveat that we spot is that the condition number produced by SWAN with NSDS is more noisy than the standard NS iteration; which might lead to improvements that will be addressed in future work. 

\paragraph{The significance of $\gradnorm$ and NSDS} The results above highlight the importance of $\gradnorm$. A potential question is whether $\gradnorm$, when followed by $\gradwhite$, is merely a no-op that rescales the initial location of the NS iteration for better convergence. In fact, for all methods considered in \Cref{fig: condition_number}, the gradients have been pre-normalized by their (global) norm before being fed into each method. This re-scaling applied to all methods provides a negative answer to the question. To clarify further, we estimate the cosine similarities between the processed gradients produced by different pairs of methods. The results, shown in \Cref{fig: condition_number} \textbf{(b)}, reveal the following:  As the training iterations increase, the cosine similarity score between SWAN-NSDS and NS (denoted as $\text{cos}(\text{SWAN-NSDS}, \text{NS})$) monotonically decreases to small values, indicating a near-orthogonal relationship. Both $\text{cos}(\text{SWAN}, \text{NS})$ and $\text{cos}(\text{SWAN}, \text{SWAN-NSDS})$ decrease over time, suggesting that both $\gradnorm$ and NSDS contribute to the orthogonality of $\text{cos}(\text{SWAN-NSDS}, \text{NS})$. This demonstrates that the changes introduced by both $\gradwhite$ and NSDS-$\gradwhite$ are significant, rotating the update by a relatively large magnitude. This might also explain the observation in \Cref{sec: llm_exp} that SWAN with NSDS behaves differently from other variants, showing slower early convergence but stronger long-term convergence.

\begin{figure}[h]
    \begin{subfigure}[b]{0.46\textwidth}
    \centering
    \includegraphics[width=0.99\linewidth]{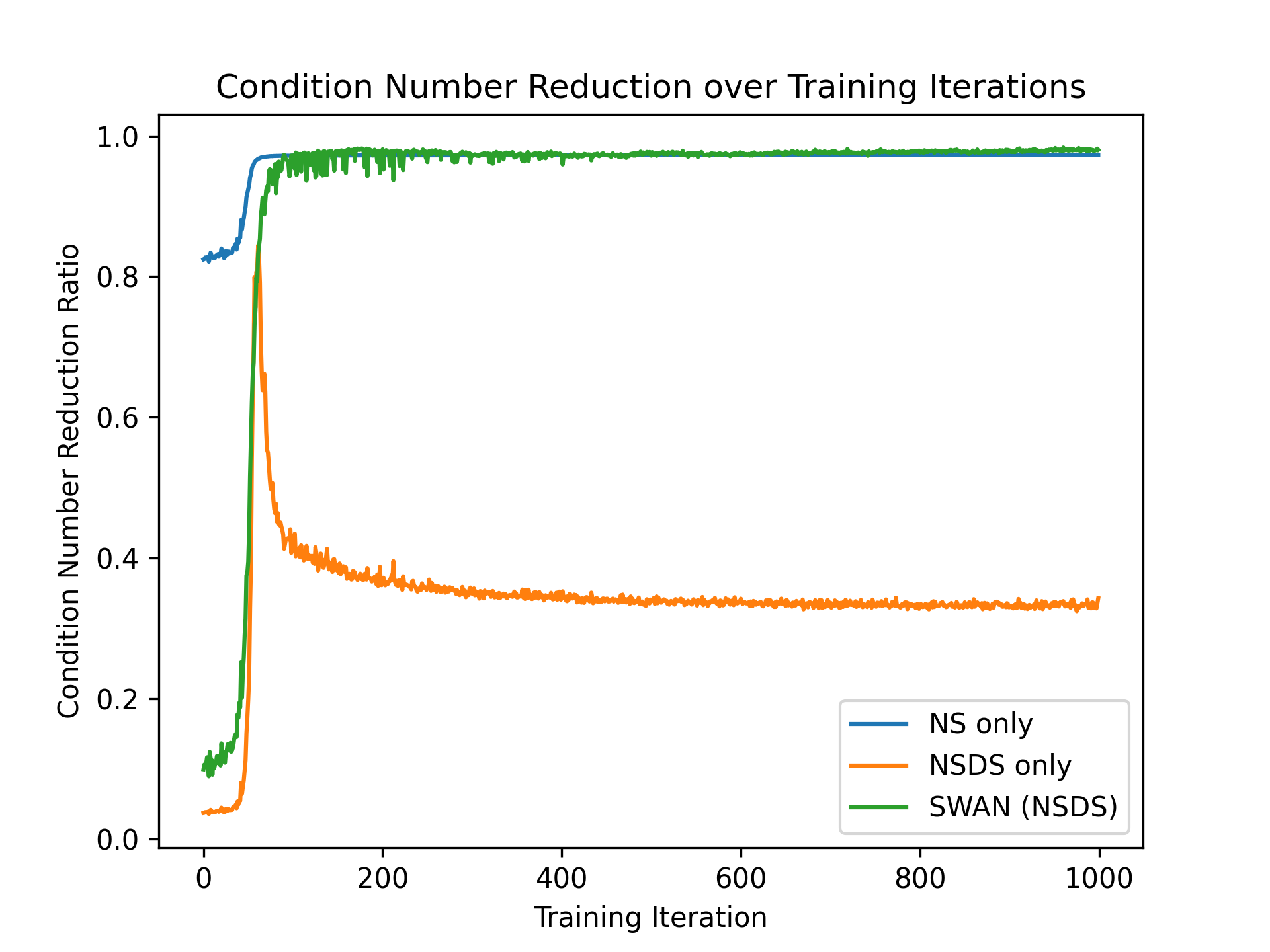}
    \caption{}
    \end{subfigure}
    \hfill
     \begin{subfigure}[b]{0.46\textwidth}
    \centering
    \includegraphics[width=0.99\linewidth]{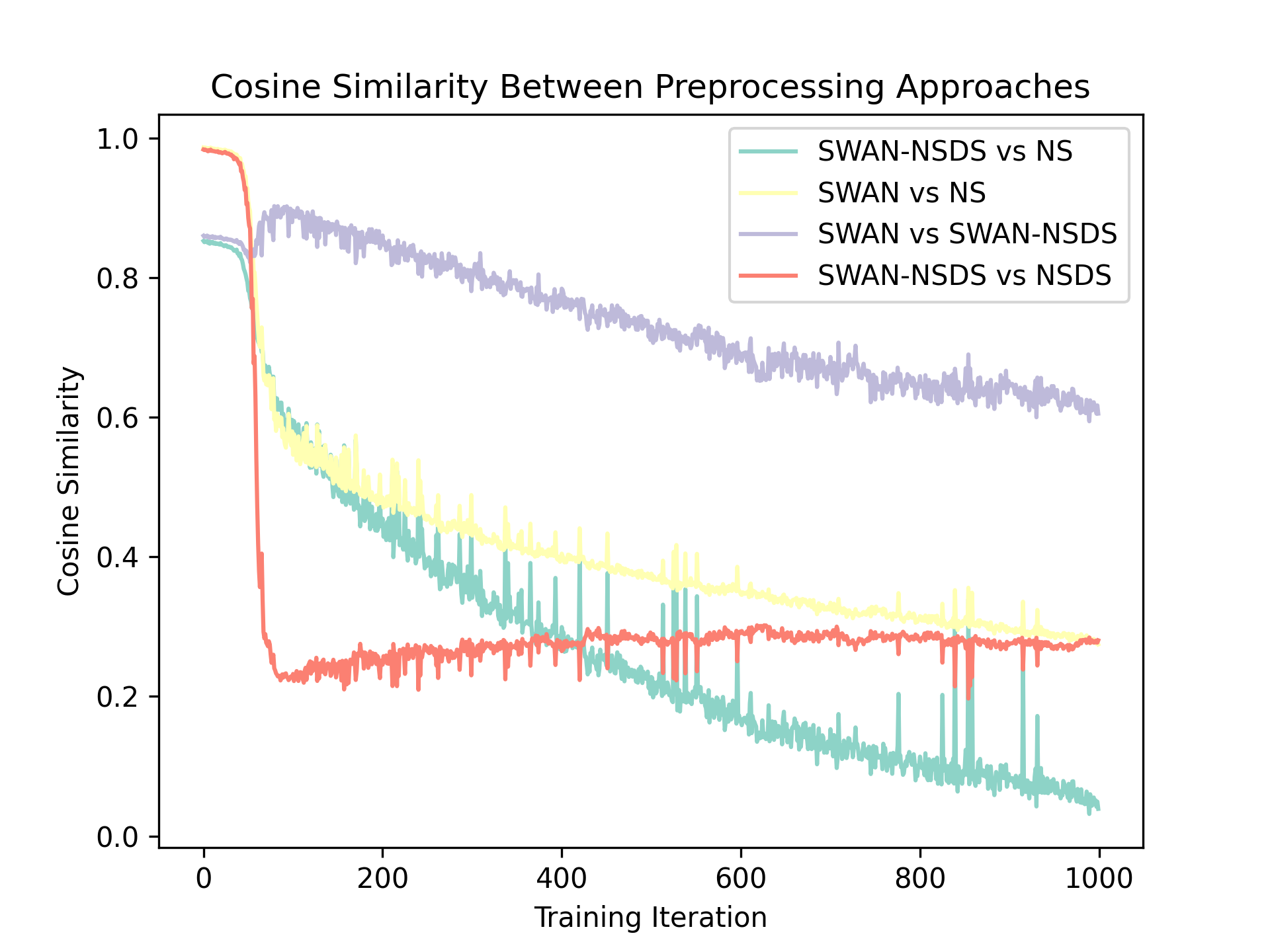}
    \caption{}
    \end{subfigure}
    \caption{Comparison between different whitening schemes. (a) Performance comparison by condition number reduction ratio (higher the better). (b) Cosine similarities between the whitened matrix produced by different preprocessing schemes, respectively.}
    \label{fig: condition_number}
\end{figure}

\subsection{Ablation: effect of NSDS iterations} \label{app: ablation_nsds}

We examine the impact of the number of iterations of our Newton Schulz with Diagonal Substitution (NSDS) scheme. With \name{}$^\ddag$, we compare the test PPL performance on 130M model. Results are shown in \Cref{tab: NSDS_iterations}.  As we can see, the improvement brought by additional NSDS iteration is marginal; hence in this paper, we only apply 2 iterations of NSDS.

\begin{table}[H]
\centering
\caption{Effect of NSDS iterations on test PPL. Results are obtain using FP32 precision.}
\label{tab: NSDS_iterations}
\begin{tabular}{l|l}
\hline
      \# NSDS iterations   & Test PPL \\ \hline
1 & - (LLM loss diverge)         \\      
2 & 22.63         \\
5  & 22.62     \\
10             & 22.61         \\
\hline
\end{tabular}
\end{table}

\subsection{Ablation: Precision} \label{app: precision_ablation}

We compare the performance of different precisions of $\gradwhite$ with NSDS scheme. As shown in \Cref{tab: NSDS_BF16}, on 130M model (2 step NSDS iterations) ablation we show that BF16 can be used without major performance degrade compared to FP32.

\begin{table}[H]
\centering
\caption{BF16 vs FP 32 NSDS}
\label{tab: NSDS_BF16}
\begin{tabular}{l|l}
\hline
      NSDS Precision  & Test PPL on 130M \\ \hline
BF16 & 22.61         \\      
FP32 & 22.63         \\

\hline
\end{tabular}
\end{table}

\section{Additional Analysis} \label{app: discussion}

\subsection{Analyzing the GradWhitening Pt. II: Robustness Against Local Curvature} \label{sec: analysis2}

In this section, we present main results regarding the convergence rate of the $\gradwhite$ method, understand its implications, and compare it with the lower bounds of GD and Adam. 

First, for simplicity, we focus on the following quadratic problem:
\begin{equation}
    \mathcal{L}(\rmW) = \frac{1}{2} \text{Tr}(\rmW^\top  \rmH \rmW) - \text{Tr}(\rmC^\top  \rmW), \label{eq: quadratic}
\end{equation}
where $ \rmW \in \mathbb{R}^{m \times n} $ 
%
%
is the parameter matrix, $ \rmH \in \mathbb{R}^{m \times m} $ is a positive definite matrix, and $ \rmC \in \mathbb{R}^{m \times n} $ is a constant matrix. 

For simplicity and without loss of generality, we assume $ \rmC = 0 $. This is because minimizing $  \mathcal{L}(\rmW) = \frac{1}{2} \text{Tr}(\rmW^\top  \rmH \rmW) - \text{Tr}(\rmC^\top  \rmW) $ is equivalent to minimizing $ \mathcal{L}(\rmW) = \frac{1}{2} \text{Tr} [(\rmW - \rmW^*)^\top  \rmH (\rmW - \rmW^*)] $, where $ \rmW^* = \rmH^{-1} \rmC $. By defining $ \rmZ = \rmW - \rmW^* $, the problem reduces to minimizing $ \mathcal{L}(\rmZ) = \frac{1}{2} \text{Tr}(\rmZ^\top  \rmH \rmZ)$.

\paragraph{Remark} Most results in this note can be easily extended to any loss function that are either i) strongly convex;  or ii) has twice differentiable functions and Lipschitz continuous Hessian, by considering their the second order approximation around $\rmW^*$.

Next, to understand the effect of $\gradwhite$, we will examine the gradient flow dynamics induced by  $\gradwhite$. Consider the $\gradwhite$-modified gradient descent: 
\begin{equation}
    \Delta \rmW^{(t)} = - \eta \gradwhite(\rmG^{(t)})  \label{eq: gradwhite_modified_gd}
\end{equation}
its exact convergence rate is given by the result as below:

\begin{theorem}[\textbf{Contraction factor of $\gradwhite$}]
    \label{thm_sve_optimal} Consider the  quadratic loss function \Cref{eq: quadratic}. Assume the initialization distribution of $ \rmW^0 $ assigns zero probability to any set of zero Lebesgue measure in $ \mathbb{R}^{m \times n}$.  Let our update rule be:
    $$ \rmW_{\text{whitened}}^{(t+1)} = \rmW_{\text{whitened}}^{(t)} - \eta \gradwhite(\rmG^{(t)}) $$
    where the learning rate is $\eta$. Then, with probability 1, we have:
    \begin{itemize}
        \item The optimal dynamic learning rate to achieve the fastest convergence is given by
        \begin{equation}
            {\eta^{(t)}}^* = \frac{ \|\rmH \rmW_{\text{whitened}}^{(t)}
            \|_1 }{\text{Tr}[\rmH]}. \label{eq_sve_lr}
        \end{equation}
        where $\|\rmH \rmW_{\text{whitened}}^{(t)}
            \|_1$ denotes the Schatten $p$-norm with $p=1$ (i.e., sum of singular values).
        \item Under ${\eta^{(t)}}^*$, the contraction factor of loss function at $t$ is given by:
        \begin{equation}
        \frac{ \mathcal{L}(\rmW_{\text{whitened}}^{(t+1)}) - \mathcal{L}^*}{\mathcal{L}(\rmW_{\text{whitened}}^{(t)}) - \mathcal{L}^*}
        = 1 - \frac{\|\rmH \rmW_{\text{whitened}}^{(t)}\|_1 ^2}{\text{Tr}[(\rmW_{\text{whitened}}^{(t)})^\top \rmH \rmW_{\text{whitened}}^{(t)}] \text{Tr}[\rmH]} 
        \label{eq_sve_rate}
    \end{equation}
    \item Furthermore, if we additionally enforce $ \rmW^0 \sim V^{m \times n}(\mathbb {R})$, i.e., initialized as an element in Steifel manifold. Then we have 
    \begin{equation}
        \frac{ \mathcal{L}(\rmW_{\text{whitened}}^{t=1}) - \mathcal{L}^*}{\mathcal{L}(\rmW^0) - \mathcal{L}^*}
        = 0
    \end{equation}
    \end{itemize}
That is, $\gradwhite$ solves the optimization problem ($\Cref{eq: quadratic}$) with 1 step iteration. 
\end{theorem}

\Cref{thm_sve_optimal} has the following key implications.

\paragraph{Convergence rate is condition number agnositc} Unlike the convergence rates of GD and Adam presented in \citet{zhang2024transformers}, as well as  Theorem \ref{cor_gd_lower_bd_general} and Corollary \ref{cor_adam_lower_bd_general} in Appendix, the optimal convergence rate (\ref{eq_sve_rate}) of $\gradwhite$ no longer explicitly depends on the condition number $\kappa$ of $H$. In fact, consider a lower bound $\frac{\|\rmH \rmW_{\text{whitened}}^{(t)}
    \|_1^2}{\text{Tr}[(\rmW_{\text{whitened}}^{(t)})^\top \rmH \rmW_{\text{whitened}}^{(t)}] \text{Tr}[\rmH]} \geq \frac{\text{Tr}[H\rmW_{\text{whitened}}^{(t)}]^2}{\text{Tr}[(\rmW_{\text{whitened}}^{(t)})^\top \rmH \rmW_{\text{whitened}}^{(t)}] \text{Tr}[\rmH]}$, since trace of $H$ appear both in the nominator and denominator, we expect that 
    to be more robust to ill-conditioned problems. For example, consider the specific initialization $\rmW_{\text{whitened}}^{(t)} = cI$, it is straightforward to show that  $\frac{\|\rmH \rmW_{\text{whitened}}^{(t)}
    \|_1^2}{\text{Tr}[(\rmW_{\text{whitened}}^{(t)})^\top \rmH \rmW_{\text{whitened}}^{(t)}] \text{Tr}[\rmH]} \geq \frac{\text{Tr}[H\rmW_{\text{whitened}}^{(t)}]^2}{\text{Tr}[(\rmW_{\text{whitened}}^{(t)})^\top \rmH \rmW_{\text{whitened}}^{(t)}] \text{Tr}[\rmH]} \perp \kappa $, which is completely disentangled from the condition number. Hence $\frac{\|\rmH \rmW_{\text{whitened}}^{(t)}
    \|_1^2}{\text{Tr}[(\rmW_{\text{whitened}}^{(t)})^\top \rmH \rmW_{\text{whitened}}^{(t)}] \text{Tr}[\rmH]}$ would not shrink as $\kappa \xrightarrow{}\infty$.  See Proposition \ref{prop_robust} for less extreme situations.

    \paragraph{Superlinear convergence with Stiefel manifold initialization} Theorem \ref{thm_sve_optimal} suggests that if $\rmW_{\text{whitened}}^{(t)}$ is initialized in the Stiefel manifold, then $\gradwhite$ reaches superlinear convergence rate (= Newton's method), while being cheaper. In fact, it is straightforward to verify that $\gradwhite$ reaches optimal solution with 1 step update. This implies $\gradwhite$ is theoretically the optimal optimization algorithm if $\rmW$ is initialized in the Stiefel manifold. 

    \paragraph{Estimation and interpretation of optimal learning rate} Compared to the optimal dynamic learning rate of gradient descent $G = \frac{G^\top G}{G^\top HG}$, the optimal learning rate ${\eta^{(t)}}^*$ of $\gradwhite$ is much easier to compute. $\frac{\text{Tr}[H \rmW_{\text{whitened}}^{(t)}]}{\text{Tr}[\rmH]}$ can be seen as balancing the average gradient magnitude against the average curvature. A higher trace of gradient ($\rmH \rmW_{\text{whitened}}^{(t)}$) (strong gradients) relative to $\rmH$ (steep curvature) suggests a larger learning rate, promoting faster updates. Conversely, a higher trace of $\rmH$ would imply a smaller learning rate to ensure stable convergence in regions with high curvature.

Next, we show that the convergence speed of $\gradwhite$ update is indeed robust to the condition number of local curvature.
\begin{proposition}[\textbf{Robustness of $\gradwhite$ update convergence rate against the condition number of local Hessian}] \label{prop_robust} Consider the quantity:
$$Q := \frac{\text{Tr} [\rmH \rmW_{\text{whitened}}^t]^2 }{\text{Tr}[(\rmW_{\text{whitened}}^{(t)})^T \rmH \rmW_{\text{whitened}}^{(t)}] \text{Tr}[\rmH]}\label{eq_sve_rate_lb}$$

Assume: i) ,$\rmW_{\text{whitened}}^{(t)} \neq \rmW^*$; and ii) the norm of $\rmH$ is bounded. Then, there exist some finite positive constant $c$, such that 
$$Q > c$$
This holds even if $\kappa \xrightarrow{ } +\infty$, where $\kappa$ is the condition number of $\rmH$.
    
\end{proposition}

Below, we provide comparison between $\gradwhite$ modified gradient descent and Adam. We only consider non-Stiefel initialization for $\gradwhite$, since with non-Stiefel initialization $\gradwhite$ is optimal according to \Cref{thm_sve_optimal}. Our results below shows that, for poor conditioned problems $\gradwhite$ with a properly chosen single global learning rate always outperforms Adam even with \emph{optimally tuned sub-group learning rates}, in terms of convergence speed. 

\begin{proposition}[\textbf{$\gradwhite$ with single lr vs Adam with tuned group lr}] \label{thm_sve_dam_group}  Consider the optimization problem \Cref{eq: quadratic}. Assume  $ \rmH $ is block-diagonal, i.e.,
    $
    \rmH = \text{diag}(\rmH_1, \rmH_2, \dots, \rmH_L),
    $
    where each $ \rmH_l \in \mathbb{R}^{m_l \times m_l} $ is a positive definite matrix for $ l = 1, 2, \dots, L $, and $ \sum_{l=1}^L m_l = m $.  Assuming for $\gradwhite$ we use one global learning rate for all parameters; and for Adam, we use the optimally chosen group learning rate $\eta_l$ and initial condition $w_0$ for each block $\rmH_l$. 

Assume either if i) certain regularity conditions are met (see proof in Appendix), or ii), if $\rmH$ is poorly-conditioned (its condition number is large enough). Then: regardless of its initialization, $\gradwhite$ with a properly chosen learning rate will still have a strictly better convergence speed (i.e., smaller contraction factor) across all blocks $l \in [L]$ than Adam ($\beta_1=0, \beta_2 = 1$) under optimal group-wise learning rates and initial condition.

\end{proposition}

\paragraph{Remark} As pointed out by \cite{zhang2024transformers} and \cite{da2020general}, Adam with $\beta_2 <1$ will have issues with convergence, which will not be completely removed even with lr decay. Therefore, we will not discuss the case of $\beta_2 <1$ to avoid the complication.  

\subsection{Numerical Verification of \Cref{prop_llm_hessian_main}}

Given a STB, we consider the following standard full-batch learning dynamics \citep{tian2023joma}. Define the conditional expectation $\mathbb{E}_{q=m}[\cdot] := \mathbb{E}[\cdot | q = m]$. Consider the dynamics of the weight matrix $W$ and the attention logits $z_q$, if we train the model with a batch of inputs that always end up with query $q[i] = m$. The weight update for $W$ and $z_q$ are given by the following noisy updates:
\begin{equation}
    \dot{\rmW}^{(t)} = \E_{q=m} \left[  \vf^{(t)} (\rmG_{\vh} \odot \vh'^{(t)})^\top  \right], \quad \dot{\vz}_m^{(t)} = \mathbb{E}_{q=m} \left[ \left( \frac{\partial \vb}{\partial \vz_m^{(t)}} \right)^\top  \rmU_C^\top  \vg_f^{(t)} \right] , \label{eq: stb_dynamics}
\end{equation}
Where $\vf^{(t)} = \left( \rmU_C \left( \exp(\vz_q^{(t)}) \odot \vx \right) + \vu_q \right)$, $(\vh^{(t)})' = \phi'((\rmW^{(t)})^\top  \vf^{(t)})$ is the derivative of the current activation, $\rmG_{\vh}^{(t)} = \nabla_{\vh^{(t)}} \mathcal{L}$ is the gradient of the loss function $\mathcal{L}$ with respect to the hidden activation $\vh^{(t)}$, and $\vg_{\vf^{(t)}}^{(t)} = \sum_k \vg_{\vh_k^{(t)}}^{(t)} (\vh^{(t)}_k)'  \vw_k^{(t)}$ is the sum of the gradients with respect to the attention logits. Here, $\vw_k^{(t)}$ is the $k$-th column of $\rmW^{(t)}$,  $\vg_{h_k^{(t)}}^{(t)}[i]$ be the backpropagated gradient sent to node $k$ at sample $i$.

Then, we numerically solving the STB ODE with $n = 12, M_C = 10$ in \Cref{eq: stb_dynamics}. During all training steps, we analytically track the evolution of Hessian of $rmW$. Results are shown in \Cref{fig_stb_hessian}. As predicted by \Cref{prop_llm_hessian_main}, we see very similar structures across the diagonal blocks of the Hessian.

\begin{figure}[tb]
    \centering
\includegraphics[width=0.9\textwidth]{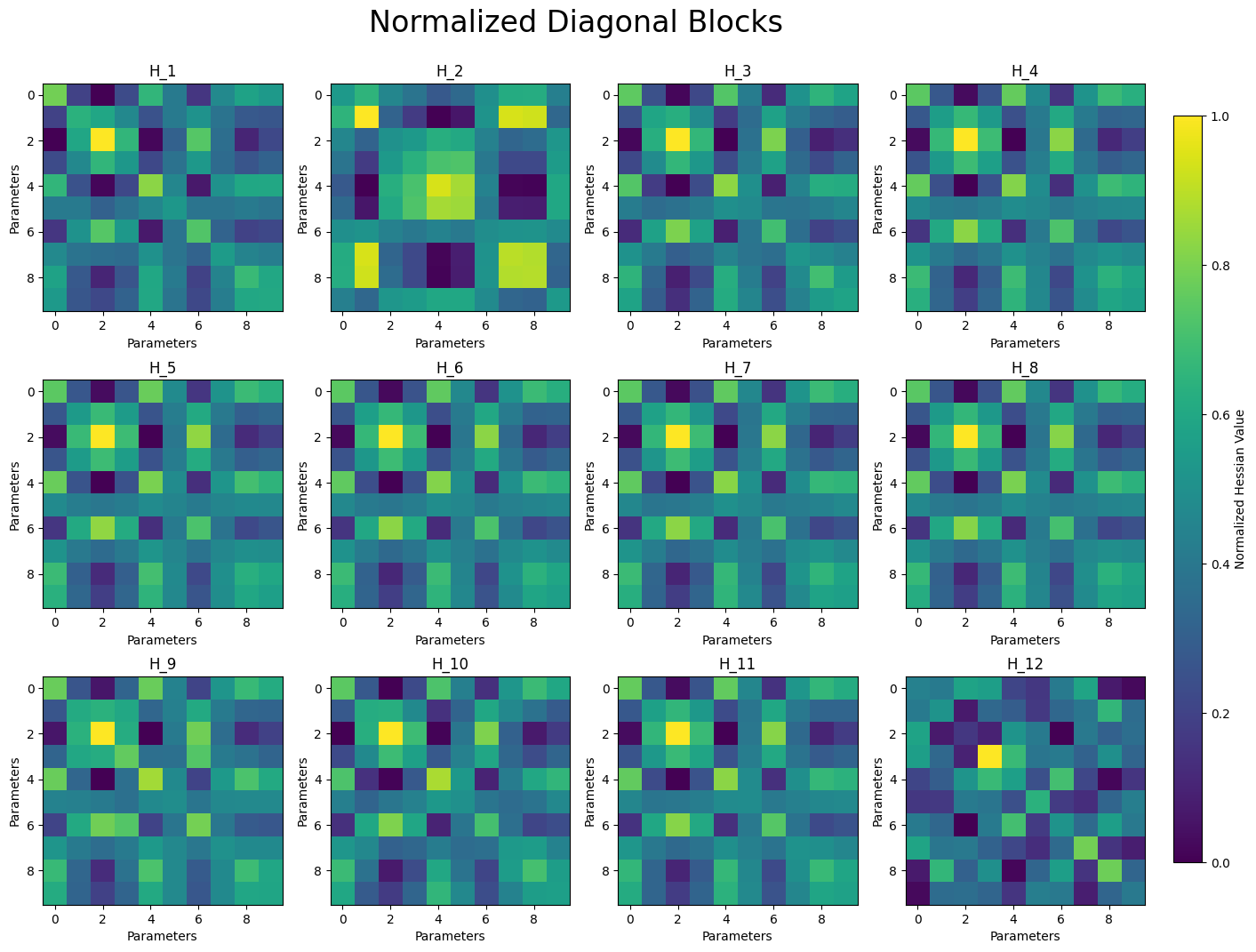}
    \caption{Normalized Hessian Blocks of size $M_C \times M_C$ along the diagonal direction of the Hessian, obtained from numerically solving the STB ODE (with $n = 12, M_C = 10$) (\ref{def: STB}) given by the full-batch dynamics (i.e., removing noise in \Cref{eq: stb_dynamics}). During all training steps, we analytically track the evolution of Hessian. As predicted by \Cref{prop_llm_hessian_main}, we see very similar structures across the diagonal blocks of the Hessian. }
    \label{fig_stb_hessian}
\end{figure}

\section{Proof of \Cref{thm: stability}}

\begin{proof}

We first consider the noiseless, full batch dynamics. Define $\rmV \in \mathbb{R}^{M_C \times n}$ as $\rmV := \rmU_C^\top\rmW$. Then following Theorem 2 in \cite{tian2023joma}, each column of $\rmV$ satisfies the following differential equation:
    \begin{equation}
        \dot{\rmV}_{[:, j]} = \exp(\rmV_{[:, j]}^2/2 + C) \odot  \mathbb{E}_q[g_{h_j} \vx]
        \label{eq: joma_ode}
    \end{equation}

    The corresponding dynamics of attention score is given by:
    \begin{equation}
        \vz_q = \frac{1}{2}\sum_j \rmV_{[:, j]}^2. \label{eq: joma}
    \end{equation}
    Without loss of generality, in this proof we only consider $C = 0$.

    Now, following the argument of Lemma B.6 of \cite{Zhao2024GaLoreML}, we reparameterize the dynamics \emph{row-wise}. For this, consider instead
    \[
    \rmV = \begin{bmatrix}
        \vu_1^\top \\ 
        \vu_2^\top \\
        \vdots\\
        \vu_{M_C}^\top
    \end{bmatrix}
    \]
    Then, \eqref{eq: joma_ode}  becomes:
    \begin{equation}
        \dot{\vu}_i = [\exp({\vu_i}^2 ) \cdot \mathbf{1}] \vmu_i
    \end{equation}
    where $\vmu_i \in \mathbb{R}^{n \times 1}$ is given by $[\vmu_i]_j := \mathbb{E}_q[g_{h_j} x_i]$. Therefore, it is clear that $\vu_i$ always move along the direction of $\vmu_i$ due to the stationary back-propagated gradient assumption. Hence, $\dot{\vu}_i = \alpha_i(t) \vmu_i$ for some scalar dynamics $\alpha_i(t)$.

    Next, consider the mini-batch version of the dynamics. In this case, the packpropagated gradient term $[\vmu_i]_j := \mathbb{E}_q[g_{h_j} x_i]$ is corrupted by some i.i.d. mini-batch noise $\mxi$. The noisy row-wise dynamics now becomes:
    \begin{equation}
       \dot{\vu}_i = \alpha_i(t) ( \vmu_i +  \mxi_i)
    \end{equation}

    Therefore, after row-wise standardization, the new dynamics becomes


    \begin{align*}
       \dot{\tilde{\vu}}_i & = \frac{ \alpha_i(t) (\mu_i + \mxi_i)  }{ \alpha_i(t) (\frac{1}{n}\sum_j (\mu_{ij} + \xi_{ij} )^2  ) } \\
       & = \frac{  (\mu_i + \mxi_i)  }{ (\frac{1}{n}\sum_j (\mu_{ij} + \xi_{ij} )^2  ) } 
    \end{align*}

    Therefore, the normalized noisy gradient $\dot{\tilde{\vu}}_i$ no longer depend on the time-variant component $\alpha(t)$. Hence, we have proved:
    $$
    \Cov[\tilde{\rmG}_{\rmU_{C}^\top\rmW}[i, :]^{(t_1)} ] = \Cov[\tilde{\rmG}_{\rmU_{C}^\top\rmW}[i, :]^{(t_2)}] \quad \text{for all } t_1, t_2, \text{and} \, i.
    $$
    The corresponding result for $ \tilde{\rmG}_{\vz_q}^{(t)} := \gradnorm( \frac{ \partial \gL_{\rmW, \vz_q}(\top\vx^{(t)}) } {\vz_q}) $ can be trivially derived due to \Cref{eq: joma}.

\end{proof}

\section{Proof of \Cref{thm_sve_optimal}}

\begin{proof}  
    We first show that $\nabla \mathcal{L}(\rmW^{(0)}) = \rmH \rmW^{(0)}$ (and hence $\nabla \mathcal{L}(\rmW_{\text{whitened}}^{(t)})$ with $t \neq \infty$) are non-zero with probability 1 under Assumption of the theorem. Given $\nabla \mathcal{L}(\rmW^{(0)}) = \rmH \rmW^{(0)}$, the set of matrices $\rmW^{(0)}$ such that $\text{Tr}(\rmH \rmW^{(0)}) = 0$ forms a hyperplane in the space of $d \times d$ matrices. Specifically, it is defined by the linear equation: $\text{Tr}(\rmH \rmW^{(0)}) = 0$. Since $\rmH$ is positive definite, at least one entry of $\rmH$ is non-zero. Thus, the hyperplane $\text{Tr}(\rmH \rmW^{(0)}) = 0$ has zero Lebesgue measure in the space of $d \times d$ matrices. Given that $\rmW^{(0)}$ is sampled from a continuous distribution, the probability that $\text{Tr}(\rmH \rmW^{(0)}) = 0$ is zero. Therefore, $\nabla \mathcal{L}(\rmW^{(0)}) \neq 0$ (and hence $\nabla \mathcal{L}(\rmW_{\text{whitened}}^{(t)})$ with $t \neq \infty$) with probability 1.  
      
    Next, we define the cost-to-go as:  
    $$  
    \mathcal{L}(\rmW^{(t)}) - \mathcal{L}^* = \frac{1}{2} \text{Tr}\left[(\rmW^{(t)})^\top \rmH \rmW^{(t)}\right],  
    $$  
      
    and the per-step improvement is (since $\mathcal{L}^* = 0$ under $ \rmW = 0 $, ):  
    $$  
    \mathcal{L}(\rmW^{(t)}) - \mathcal{L}(\rmW^{(t+1)}) = \frac{1}{2} \text{Tr}\left[(\rmW^{(t)})^\top \rmH \rmW^{(t)}\right] - \frac{1}{2} \text{Tr}\left[(\rmW^{(t+1)})^\top \rmH \rmW^{(t+1)}\right].  
    $$  
      
    Substituting the update rule $\rmW^{(t+1)} = \rmW^{(t)} - \eta \gradwhite(\rmG_{\text{whitened},l}) = \rmW^{(t)} - \eta \rmU \rmV^\top$, we get:  
    $$  
    \mathcal{L}(\rmW^{(t)}) - \mathcal{L}(\rmW^{(t+1)}) = \frac{1}{2} \text{Tr}\left[(\rmW^{(t)})^\top \rmH \rmW^{(t)}\right] - \frac{1}{2} \text{Tr}\left[(\rmW^{(t)} - \eta \rmU \rmV^\top)^\top \rmH (\rmW^{(t)} - \eta \rmU \rmV^\top)\right].  
    $$  
      
    Expanding the right-hand side, we have  
    $$  
    \mathcal{L}(\rmW^{(t)}) - \mathcal{L}(\rmW^{(t+1)}) = \eta \text{Tr}\left[(\rmW^{(t)})^\top \rmH \rmU \rmV^\top\right] - \frac{\eta^2}{2} \text{Tr}\left[(\rmU \rmV^\top)^\top \rmH (\rmU \rmV^\top)\right].  
    $$  
      
    Now, noticing that $\rmG = \rmH \rmW^{(t)} = \rmU \Sigma \rmV^\top$, we have:  
    $$  
    \text{Tr}\left[(\rmW^{(t)})^\top \rmH \rmU \rmV^\top\right] = \text{Tr}\left[(\rmH \rmW^{(t)})^\top \rmU \rmV^\top\right] = \text{Tr}\left[(\rmU \Sigma \rmV^\top)^\top \rmU \rmV^\top\right] = \text{Tr}\left[\rmV \Sigma \rmU^\top \rmU \rmV^\top\right] = \text{Tr}\left[\rmV \Sigma \rmV^\top\right].  
    $$  
      
    Since $\rmV$ is orthogonal, $\rmV^\top \rmV = \rmI$, and $\Sigma$ is diagonal, we obtain:  
    $$  
    \text{Tr}\left[(\rmW^{(t)})^\top \rmH \rmU \rmV^\top\right] = \text{Tr}(\Sigma) = \|\rmH \rmW_{\text{whitened}}^{(t)}\|_1.  
    $$  
      
    Similarly:  
    $$  
    \text{Tr}\left[(\rmU \rmV^\top)^\top \rmH (\rmU \rmV^\top)\right] = \text{Tr}\left[\rmV \rmU^\top \rmH \rmU \rmV^\top\right] = \text{Tr}\left[\rmV \Lambda \rmV^\top\right] = \text{Tr}(\rmH),  
    $$  
      
    where $\Lambda$ is the eigenvalue matrix of $\rmH$. Given those intermediate results, we have:  
    \begin{align*}  
        \frac{\mathcal{L}(\rmW^{(t+1)}) - \mathcal{L}^*}{\mathcal{L}(\rmW^{(t)}) - \mathcal{L}^*} &= 1 - \frac{\mathcal{L}(\rmW^{(t)}) - \mathcal{L}(\rmW^{(t+1)})}{\mathcal{L}(\rmW^{(t)}) - \mathcal{L}^*} \\  
        &= 1 - \frac{\eta \|\rmH \rmW_{\text{whitened}}^{(t)}\|_1 - \frac{\eta^2}{2} \text{Tr}(\rmH)}{\frac{1}{2} \text{Tr}\left[(\rmW^{(t)})^\top \rmH \rmW^{(t)}\right]}.  
    \end{align*}  
      
    Noticing that this is a quadratic function of $\eta$ and the second order coefficient is positive, it is straightforward to verify via the quadratic formula that the optimal learning rate is given by  
      
    \begin{equation*}  
        \eta_t^* = \frac{\|\rmH \rmW_{\text{whitened}}^{(t)}\|_1}{\text{Tr}(\rmH)}. \label{eq_sver}  
    \end{equation*}  
  
    Under which the optimal contraction factor is given by  
  
    \begin{equation*}  
        \frac{\mathcal{L}(\rmW_{\text{whitened}}^{(t+1)}) - \mathcal{L}^*}{\mathcal{L}(\rmW_{\text{whitened}}^{(t)}) - \mathcal{L}^*}  
        = 1 - \frac{\|\rmH \rmW_{\text{whitened}}^{(t)}\|_1^2}{\text{Tr}\left[(\rmW_{\text{whitened}}^{(t)})^\top \rmH \rmW_{\text{whitened}}^{(t)}\right] \text{Tr}(\rmH)}. 
    \end{equation*}  
  
    Finally, if we additionally enforce $\rmW^{(0)} \sim V^{m \times n}(\mathbb{R})$, i.e., we can parameterize $\rmW^{(0)} = \rmO$ where $\rmO$ is orthogonal, then it is trivial to verify that $\gradwhite$ reaches the optimal solution with a 1-step update. To see this, consider the $\gradwhite$ update:  
    $$  
    \rmW_{\text{whitened}}^{(1)} = \rmO - \eta^* \gradwhite(\rmH \rmO) = \rmO - \frac{\|\rmH \rmO\|_1}{\text{Tr}(\rmH)} \gradwhite(\rmH \rmO),  
    $$  
      
    noticing that $\frac{\|\rmH \rmO\|_1}{\text{Tr}(\rmH)} = 1$, and $\gradwhite(\rmH \rmO) = \mathcal{P}(\rmQ \Lambda \rmQ^\top \rmO) = \rmQ \rmQ^\top \rmO = \rmO$. Hence:  
    $$  
    \rmW_{\text{whitened}}^{(1)} = \rmO - \eta^* \gradwhite(\rmH \rmO) = \rmO - \rmO = \mathbf{0} = \rmW^*.  
    $$  
      
    Hence, the proof is complete.  
\end{proof}  


\section{Proof of \Cref{prop_robust}}

\begin{proof}  
    Since $\rmW_{\text{whitened}}^{(t)} \neq \rmW^*$, the square of the trace of the gradient $\text{Tr}[\rmH \rmW_{\text{whitened}}^{(t)}]^2$ must exceed some positive constant $C_G$, that is, $\text{Tr}[\rmH \rmW_{\text{whitened}}^{(t)}]^2 > C_G$.  
      
    On the other hand, because:  
    \begin{enumerate}  
        \item The quadratic loss term $\text{Tr}[(\rmW_{\text{whitened}}^{(t)})^\top \rmH \rmW_{\text{whitened}}^{(t)}]$ is upper-bounded on $\mathbb{R}^{n \times n}$, and  
        \item $\text{Tr}[(\rmW_{\text{whitened}}^{(t)})^\top \rmH \rmW_{\text{whitened}}^{(t)}] \neq 0$ (due to $\rmW_{\text{whitened}}^{(t)} \neq \rmW^*$),  
    \end{enumerate}  
    we have that there exists a positive number $0 < C_{\mathcal{L}}$ such that   
    $$  
    0 < \text{Tr}[(\rmW_{\text{whitened}}^{(t)})^\top \rmH \rmW_{\text{whitened}}^{(t)}] < C_{\mathcal{L}}.  
    $$  
      
    Finally, since the norm of $\rmH$ is upper bounded, its trace must also be upper bounded by some constant $C_H$. Therefore, putting everything together, we have:  
    $$  
    Q > \frac{C_G^2}{C_{\mathcal{L}} C_H}.  
    $$  
      
    This inequality holds even as the condition number $\kappa \xrightarrow{} +\infty$.  
\end{proof}

\section{Proof of \Cref{thm_sve_dam_group}}

To prove \Cref{thm_sve_dam_group}, we first generalize existing work on the convergence rate lower bound (via contraction factor) of gradient descent and Adam (we only consider $\beta_2 = 1$) under the same setting:

\begin{theorem}
 [Contraction factor lower bound for gradient descent, generalized based on \citet{zhang2024transformers}]
\label{cor_gd_lower_bd_general}
Consider the optimization problem in \Cref{eq: quadratic}. Let $\rmW_{GD}^t$ be the output of GD after $t$ steps.  Then, for any step size $ \eta $, there exists an initial condition such that the following lower bound on the contraction rate holds:
$$
\mathcal{L}(\rmW_{GD}^{t+1}) - \mathcal{L}^* \geq \left(1 - \frac{2}{\kappa + 1}\right) \left(\mathcal{L}(\rmW_{GD}^t) - \mathcal{L}^*\right),
$$
where $ \mathcal{L}^* = \mathcal{L}(\rmW^*) $. Furthermore, under optimal $\eta = \frac{2}{\lambda_1 + \lambda_m}$, the bound becomes tight regardless of the settings of $\rmH$, where $\lambda_1$ and $\lambda_m$ are the largest and smallest eigen values of $\rmH$, respectively.
\end{theorem}

\begin{proof}  
    The proposition 1 in \citet{zhang2024transformers} has shown that the lower bound holds for diagonal positive definite Hessian $\rmH$. To show that the lower bound holds for a general positive definite Hessian $\rmH$ we will reformulate the problem to align with the setup in diagonal case (Prposition 1 of \citet{zhang2024transformers}).
  
    First, for any positive definite Hessian $\rmH$, we can perform an eigen decomposition $\rmH = \rmU \rmS \rmU^\top$, where $\rmU$ is an orthogonal matrix and $\rmS$ is a diagonal matrix containing the eigenvalues of $\rmH$. Define a change of variables $\rmZ = \rmU^\top \rmW$. Then, the optimization problem becomes  
    $$  
    \mathcal{L}(\rmZ) = \frac{1}{2} \text{Tr}(\rmZ^\top \rmS \rmZ),  
    $$  
    which reduces the problem to the case of a diagonal $\rmH$ with condition number $\kappa = \frac{\lambda_1}{\lambda_m}$, where $\lambda_1$ and $\lambda_m$ are the largest and smallest eigenvalues of $\rmH$, respectively.  
  
    Thus, by applying Proposition 1 of \citet{zhang2024transformers} to this transformed problem, we conclude that there exists initial point such that the lower bound on the contraction rate  
    $$  
    \mathcal{L}(\rmW_{GD}^{(t+1)}) - \mathcal{L}^* \geq \left(1 - \frac{2}{\kappa + 1}\right) \left(\mathcal{L}(\rmW_{GD}^{(t)}) - \mathcal{L}^*\right)  
    $$  
    holds for the transformed variables $\rmZ$ and, equivalently, for the original variables $\rmW$ since the condition number is preserved under orthogonal transformations.  
  
    Therefore, the lower bound for gradient descent applies to any general positive definite Hessian $\rmH$ provided the condition number $\kappa$ remains unchanged.  
  
    Finally, under the optimal step size $\eta = \frac{2}{\lambda_1 + \lambda_m}$, the bound becomes tight regardless of the settings of $\rmH$. This is achieved by selecting $\eta$ to minimize the contraction factor, aligning with well-known results regarding the optimal convergence rate of gradient descent on quadratic objectives \citep{nesterov2013introductory}.  
  
    This completes the proof of \Cref{cor_gd_lower_bd_general}.  
\end{proof}

\begin{corollary} [Lower bound on Adam ($\beta_2 = 1$)] \label{cor_adam_lower_bd_general} Consider the optimization problem in \Cref{eq: quadratic}. Assume the weight initialization $ \rmW^0 $ assigned zero probability to any set of zero Lebesgue measure in $\mathbb{R}^{m \times n}$.  Let $ \rmW_{\text{Adam}}^{(t)} $ be the parameter after $ t $ iterations of Adam with hyperparameters $ \beta_1 = 0 $ and $ \beta_2 = 1 $.  Then, for any step size $ \eta $, the following lower bound on the contraction rate holds:
$$
\mathcal{L}(\rmW_{Adam}^{(t+1)}) - \mathcal{L}^* \geq \left(1 - \frac{2}{\kappa'(\rmW^0) + 1}\right) \left(\mathcal{L}(\rmW_{Adam}^{(t)}) - \mathcal{L}^*\right),
$$
where $\kappa'(\rmW^0)$ is the $\rmW^0$-dependent condition number of the preconditioned Hessian $\text{diag} (|\rmH \rmW^0|^{-1}) \rmH$, and  $ \mathcal{L}^* = \mathcal{L}(\rmW^*) $.
\end{corollary}

\begin{proof}  
    The update rule of Adam with $\beta_1 = 0$ and $\beta_2 = 1$ is given by \citet{zhang2024transformers}:  
    $$  
    \rmW_{\text{Adam}}^{(t+1)} = \rmW_{\text{Adam}}^{(t)} + \eta \, \text{diag}(|\rmH \rmW^{(0)}|^{-1}) \, \rmH \rmW_{\text{Adam}}^{(t)}.  
    $$  
    This can be interpreted as gradient descent with a preconditioned Hessian matrix $\text{diag}(|\rmH \rmW^{(0)}|^{-1}) \rmH$. By applying \Cref{cor_gd_lower_bd_general}, we conclude that the contraction rate for Adam under these settings satisfies the lower bound:  
    $$  
    \mathcal{L}(\rmW_{\text{Adam}}^{(t+1)}) - \mathcal{L}^* \geq \left(1 - \frac{2}{\kappa + 1}\right) \left(\mathcal{L}(\rmW_{\text{Adam}}^{(t)}) - \mathcal{L}^*\right),  
    $$  
    where $\kappa$ is the condition number of the Hessian matrix $\rmH$.  
      
    Therefore, the proof is complete.  
\end{proof}  

Next, We extend our \Cref{thm_sve_optimal} to block-diagonal Hessian case to prepare for discussions on group learning rates when comparing to Adam.

\begin{corollary}[\textbf{Upper Bound Convergence Rate of \name}]  
    \label{thm_mod_gd}  
    Consider the same quadratic loss function $\mathcal{L}(\rmW) = \frac{1}{2} \text{Tr}(\rmW^\top \rmH \rmW)$ with $\rmH$ being block-diagonal. That is, $
    \rmH = \text{diag}(\rmH_1, \rmH_2, \dots, \rmH_L),$
    where each $ H_l \in \mathbb{R}^{m_l \times n_l} $ is a positive definite matrix for $ l = 1, 2, \dots, L $, and $ \sum_{l=1}^L m_l =  m$ and $ \sum_{l=1}^L n_l =  n$. Assume the initialization distribution of $\rmW^{(0)}$ assignes zero probability to any zero measure set in $\mathbb{R}^{m \times n}$. Let $\rmW_{\text{whitened}}^{(t)}$ be the parameter matrix after $t$ iterations of the \name optimizer defined in \Cref{thm_sve_optimal}, with learning rate $\eta$. Then, under the conditions that \footnote{Note that according to Proposition \ref{prop_robust}, this is always achievable when $\rmH_l$ is poorly-conditioned ($\frac{\lambda_{l, d_l}}{\lambda_{l, 1}}$ is small enough). The lower interval above converges to zero, and one can simply pick e.g., $\eta = \min_{l \in [L]} \frac{\text{Tr}(\rmH_l \rmW_l^{(t)})}{\text{Tr}(\rmH_l)}$.}:  
    $$  
    \|\rmH_l \rmW_l^{(t)}\|_1^2 - \text{Tr}(\rmH_l) \cdot \text{Tr}\left((\rmW_l^{(t)})^\top \rmH_l \rmW_l^{(t)}\right) \cdot \frac{2 \lambda_{l, m_l}}{\lambda_{l, 1} + \lambda_{l, m_l}} > 0,  
    $$  
    where $\lambda_{l, 1}$ and $\lambda_{l, m_l}$ are the largest and smallest singular value of $\rmH_l$, respectively; then there exists a proper learning rate $\eta$ such that:  
    %
    %
    with probability 1, the loss satisfies:  
    \begin{equation}  
        \mathcal{L}(\rmW_{\text{whitened}}^{(t+1)}) - \mathcal{L}^* < \max_{l \in [L]} \left(1 - \frac{2}{\kappa_l + 1}\right) \left(\mathcal{L}(\rmW_{\text{whitened}}^{(t)}) - \mathcal{L}^*\right),  
    \end{equation}  
    where $\kappa_l$ is the condition number of $\rmH_l$.  
\end{corollary}

\begin{proof}  
  
    Applying the arguments in the proof of Theorem \ref{thm_sve_optimal} to each block $l$, we have (for simplicity, we will drop the subscript ``whitened'' when there is no confusion):  
    \begin{align*}  
        \frac{\mathcal{L}(\rmW_l^{(t+1)}) - \mathcal{L}^*}{\mathcal{L}(\rmW_l^{(t)}) - \mathcal{L}^*} &= 1 - \frac{\mathcal{L}(\rmW_l^{(t)}) - \mathcal{L}(\rmW_l^{(t+1)})}{\mathcal{L}(\rmW_l^{(t)}) - \mathcal{L}^*} \\  
        &= 1 - \frac{\eta \|\rmH_l \rmW_l^{(t)}\|_1 - \frac{\eta^2}{2} \text{Tr}(\rmH_l)}{\frac{1}{2} \text{Tr}\left[(\rmW_l^{(t)})^\top \rmH_l \rmW_l^{(t)}\right]}.  
    \end{align*}  
      
    It is straightforward to verify via the quadratic formula that if one chooses $\eta$ satisfying:  
    \begin{align*}  
       & \frac{\|\rmH_l \rmW_l^{(t)}\|_1 - \sqrt{\|\rmH_l \rmW_l^{(t)}\|_1^2 - \text{Tr}(\rmH_l) \cdot \text{Tr}((\rmW_l^{(t)})^\top \rmH_l \rmW_l^{(t)}) \cdot \frac{2 \lambda_{l, m_l}}{\lambda_{l, 1} + \lambda_{l, m_l}}}}{\text{Tr}(\rmH_l)} < \eta \\  
       & < \frac{\|\rmH_l \rmW_l^{(t)}\|_1 + \sqrt{\|\rmH_l \rmW_l^{(t)}\|_1^2 - \text{Tr}(\rmH_l) \cdot \text{Tr}((\rmW_l^{(t)})^\top \rmH_l \rmW_l^{(t)}) \cdot \frac{2 \lambda_{l, m_l}}{\lambda_{l, 1} + \lambda_{l, m_l}}}}{\text{Tr}(\rmH_l)},  
    \end{align*}  
      
    then we have:  
    \begin{align*}  
          \frac{\frac{1}{2} \text{Tr}\left[(\rmW_l^{(t)})^\top \rmH_l \rmW_l^{(t)}\right]}{\eta \|\rmH_l \rmW_l^{(t)}\|_1 - \frac{\eta^2}{2} \text{Tr}(\rmH_l)} < \frac{\kappa_l + 1}{2}.  
    \end{align*}  
  
    Rearranging, we obtain:  
    $$  
    \frac{\mathcal{L}(\rmW_l^{(t+1)}) - \mathcal{L}^*}{\mathcal{L}(\rmW_l^{(t)}) - \mathcal{L}^*} < 1 - \frac{2}{\kappa_l + 1}.  
    $$  
      
    Since $\rmH$ is block-diagonal, the updates for each block $l$ are independent. Summing the loss over all blocks, we obtain:  
    $$  
    \mathcal{L}(\rmW_{\text{whitened}}^{(t+1)}) - \mathcal{L}^* = \sum_{l=1}^L \left(\mathcal{L}(\rmW_l^{(t+1)}) - \mathcal{L}^*\right) < \sum_{l=1}^L \left(1 - \frac{2}{\kappa_l + 1}\right) \left(\mathcal{L}(\rmW_l^{(t)}) - \mathcal{L}^*\right).  
    $$  
    Taking the maximum contraction factor across all blocks:  
    \begin{align*}
    \mathcal{L}(\rmW_{\text{whitened}}^{(t+1)}) - \mathcal{L}^* <& \max_{l \in [L]} \left(1 - \frac{2}{\kappa_l + 1}\right) \sum_{l=1}^L \left(\mathcal{L}(\rmW_l^{(t)}) - \mathcal{L}^*\right)\\
    =& \max_{l \in [L]} \left(1 - \frac{2}{\kappa_l + 1}\right) \left(\mathcal{L}(\rmW_{\text{whitened}}^{(t)}) - \mathcal{L}^*\right).  
    \end{align*} 
      
    Thus, the overall contraction factor for the \name optimizer is:  
    $$  
    \rho_{\text{\name}} = \max_{l \in [L]} \left(1 - \frac{2}{\kappa_l + 1}\right).  
    $$  
      
\end{proof}

Finally, we are ready to prove \Cref{thm_sve_dam_group}:

\setcounter{proposition}{2}

{  

\renewcommand\theproposition{2} 
\begin{proposition}[\textbf{$\gradwhite$ with single lr vs Adam with tuned group lr}]  
Consider the optimization problem \Cref{eq: quadratic}. Assume  $ \rmH $ is block-diagonal, i.e.,    $    \rmH = \text{diag}(\rmH_1, \rmH_2, \dots, \rmH_L),    $    where each $ \rmH_l \in \mathbb{R}^{m_l \times m_l} $ is a positive definite matrix for $ l = 1, 2, \dots, L $, and $ \sum_{l=1}^L m_l = m $.  Assuming for $\gradwhite$ we use one global learning rate for all parameters; and for Adam, we use the optimally chosen group learning rate $\eta_l$ and initial condition $w_0$ for each block $\rmH_l$.   
Assume either if i) certain regularity conditions are met (see proof in Appendix), or ii), if $\rmH$ is poorly-conditioned (its condition number is large enough). Then: regardless of its initialization, $\gradwhite$ with a properly chosen learning rate will still have a strictly better convergence speed (i.e., smaller contraction factor) across all blocks $l \in [L]$ than Adam ($\beta_1=0, \beta_2 = 1$) under optimal group-wise learning rates and initial condition.  
\end{proposition}  
}  

\begin{proof} 
For simplicity, we will drop the subscript ``whitened'' when there is no confusion. Let $\kappa'_l(\rmW_l^{(0)})$ denote the $ \rmW_l^{(0)} $-dependent condition number of the $ l $-th block preconditioned Hessian $\text{diag}(\left|\rmH_l \rmW_l^{(0)}\right|^{-1}) \rmH_l$. Let $\lambda_{l, m_l}(\rmW_l^{(0)})$ and $\lambda_{l, 1}(\rmW_l^{(0)})$ be the smallest and largest eigenvalues of $\text{diag}(\left|\rmH_l \rmW_l^{(0)}\right|^{-1}) \rmH_l$, respectively. Then,  
  
    \paragraph{Case 1:} Under the conditions that:  
    \begin{enumerate}  
        \item \textbf{Existence of roots:} $\forall l \in [L]$,   
        $$  
        \|\rmH_l \rmW_l^{(t)}\|_1^2 - \text{Tr}(\rmH_l) \cdot \text{Tr}\left((\rmW_l^{(t)})^\top \rmH_l \rmW_l^{(t)}\right) \cdot \frac{2 \lambda_{l, m_l}(\rmW_l^{(0)})}{\lambda_{l, 1}(\rmW_l^{(0)}) + \lambda_{l, m_l}(\rmW_l^{(0)})} > 0,  
        $$  
        and  
        \item \textbf{Overlap condition:}  
        \begin{align*}  
            & \min_{l \in [L]} \frac{\text{Tr}(\rmH_l \rmW_l^{(t)}) + \sqrt{\text{Tr}(\rmH_l \rmW_l^{(t)})^2 - \text{Tr}(\rmH_l) \cdot \text{Tr}\left((\rmW_l^{(t)})^\top \rmH_l \rmW_l^{(t)}\right) \cdot \frac{2 \lambda_{l, m_l}(\rmW_l^{(0)})}{\lambda_{l, 1}(\rmW_l^{(0)}) + \lambda_{l, m_l}(\rmW_l^{(0)})}}}{\text{Tr}(\rmH_l)} \\  
            & > \max_{l \in [L]} \frac{\text{Tr}(\rmH_l \rmW_l^{(t)}) - \sqrt{\text{Tr}(\rmH_l \rmW_l^{(t)})^2 - \text{Tr}(\rmH_l) \cdot \text{Tr}\left((\rmW_l^{(t)})^\top \rmH_l \rmW_l^{(t)}\right) \cdot \frac{2 \lambda_{l, m_l}(\rmW_l^{(0)})}{\lambda_{l, 1}(\rmH_l^{(0)}) + \lambda_{l, m_l}(\rmW_l^{(0)})}}}{\text{Tr}(\rmH_l)}.  
        \end{align*}  
    \end{enumerate}  
    Then, there exists a global learning rate $\eta$, such that for all $l \in [L]$,  
    \begin{align*}  
        \frac{\mathcal{L}(\rmW_{\text{whitened}}^{(t+1)})_l - \mathcal{L}^*_l}{\mathcal{L}(\rmW_{\text{whitened}}^{(t)})_l - \mathcal{L}^*_l} < 1 - \frac{2}{\kappa'_l(\rmW_l^{(0)}) + 1} \leq \frac{\mathcal{L}(\rmW_{\text{Adam}}^{(t+1)})_l - \mathcal{L}^*_l}{\mathcal{L}(\rmW_{\text{Adam}}^{(t)})_l - \mathcal{L}^*_l}.  
    \end{align*}  
  
    \paragraph{Case 2:} If $\rmH$ is poorly-conditioned, i.e., $\frac{\lambda_{l, m_l}(\rmW_l^{(0)})}{\lambda_{l, 1}(\rmW_l^{(0)})} \xrightarrow{} 0$, then \Cref{prop_robust} asserts that the following term  
    $$  
    \max_{l \in [L]} \frac{\text{Tr}(\rmH_l \rmW_l^{(t)}) - \sqrt{\text{Tr}(\rmH_l \rmW_l^{(t)})^2 - \text{Tr}(\rmH_l) \cdot \text{Tr}\left((\rmW_l^{(t)})^\top \rmH_l \rmW_l^{(t)}\right) \cdot \frac{2 \lambda_{l, m_l}(\rmW_l^{(0)})}{\lambda_{l, 1}(\rmW_l^{(0)}) + \lambda_{l, m_l}(\rmW_l^{(0)})}}}{\text{Tr}(\rmH_l)} \xrightarrow{} 0,  
    $$  
    and one can simply choose, for example, $\eta = \min_{l \in [L]} \frac{\text{Tr}(\rmH_l \rmW_l^{(t)})}{\text{Tr}(\rmH_l)}$. Under this choice of $\eta$, we still have  
    \begin{align*}  
        \frac{\mathcal{L}(\rmW_{\text{whitened}}^{(t+1)})_l - \mathcal{L}^*_l}{\mathcal{L}(\rmW_{\text{whitened}}^{(t)})_l - \mathcal{L}^*_l} < 1 - \frac{2}{\kappa'_l(\rmW_l^{(0)}) + 1} \leq \frac{\mathcal{L}(\rmW_{\text{Adam}}^{(t+1)})_l - \mathcal{L}^*_l}{\mathcal{L}(\rmW_{\text{Adam}}^{(t)})_l - \mathcal{L}^*_l}  
    \end{align*}  
    for all $l \in [L]$.  
\end{proof}

\section{Proof of \Cref{prop_llm_hessian_main}}

\begin{proof}  
    First, define $\rmV \in \mathbb{R}^{M_C \times n}$ as $\rmV := \rmU_C^\top \rmW$, and consider the Hessian with respect to $\rmV$ instead of $\rmW$. Notice that although the loss function $\mathcal{L}$ is unknown, its first-order derivatives are known. Specifically, they are given by:  
    $$  
    \frac{\partial \mathcal{L}}{\partial v_{lk}} = \dot{v}_{lk} = \mathbb{E}_{q=m} \left[ g_{h_k} x_l \right] e^{\frac{1}{2} \sum_{s} v_{ls}^2}.  
    $$  
      
    Therefore, the second-order derivatives, i.e., the Hessian matrix $\rmH(\rmV)$, are:  
    $$  
    \rmH(\rmV)_{lk,l'k'} = \frac{\partial^2 \mathcal{L}}{\partial v_{lk} \partial v_{l'k'}} = \frac{\partial \left[ \mathbb{E}_{q=m} \left[ g_{h_k} x_l \right] e^{\frac{1}{2} \sum_{s} v_{ls}^2} \right]}{\partial v_{l'k'}} = \dot{v}_{lk} v_{l'k'} \delta_{ll'},  
    $$  
    where $\delta_{ll'}$ is the Kronecker delta, which is $1$ if $l = l'$ and $0$ otherwise.  
      
    Based on Lemma B.6 in \citet{Zhao2024GaLoreML}, as $t \to \infty$, there exists an index subset $O_l \subset \{1, \dots, M_C\}$ such that:  
    $$  
    v_{l^* k} \gg v_{lk}, \quad \dot{v}_{l^* k} \gg \dot{v}_{lk}, \quad \forall l^* \in O_l, \; l \notin O_l, \; \forall k.  
    $$  
    Consequently,  
    $$  
    \rmH(\rmV)_{l^* k, l^* k'} \gg \rmH(\rmV)_{lk, l'k'}, \quad \forall l^* = l'^* \in O_l, \; l, l' \notin O_l, \; \forall k, k'.  
    $$  
      
    After normalization, as $t \to \infty$, we have  
    $$  
    \frac{\rmH(\rmV)_{lk,l'k'}}{\sum_{l,l'} \rmH(\rmV)_{lk,l'k'}} \xrightarrow{t \to \infty}   
    \begin{cases}  
        1, & \text{if } l = l' \in O_l, \\  
        0, & \text{otherwise}.  
    \end{cases}  
    $$  
      
    Reverting back to the $\rmW$ space, we have  
    $$  
    \rmH(\rmW) = (\rmI_K \otimes \rmU_C) \rmH(\rmV) (\rmI_K \otimes \rmU_C)^\top,  
    $$  
    where $\otimes$ denotes the Kronecker product and $\rmI_K$ is the identity matrix of appropriate dimensions.  
      
    Therefore, for all $1 \leq s, s' \leq d$ and $1 \leq k, k' \leq n$, we obtain  
    $$  
    \frac{\rmH(\rmW)_{sk,s'k'}}{\sum_{s,s'} \rmH(\rmW)_{sk,s'k'}} = \frac{\rmH(\rmW)_{sk',s'k'}}{\sum_{s,s'} \rmH(\rmW)_{sk',s'k'}} \quad \text{as } t \to \infty.  
    $$  
    This holds for all $1 \leq s, s' \leq d$ and $1 \leq k, k' \leq n$.  
\end{proof}

\section{Empirical verification of theoretical insights from \Cref{sec: analysis} regarding $\gradnorm$ and $\gradwhite$} \label{app: verification}

\subsubsection{Does $\gradnorm$ Stabilize Gradient Distributions of SGD?}  \label{sec: gradnorm_exp}

\begin{figure}[t!]
    \centering
    \begin{subfigure}[b]{0.49\textwidth}
        \centering
        \includegraphics[width=\textwidth]{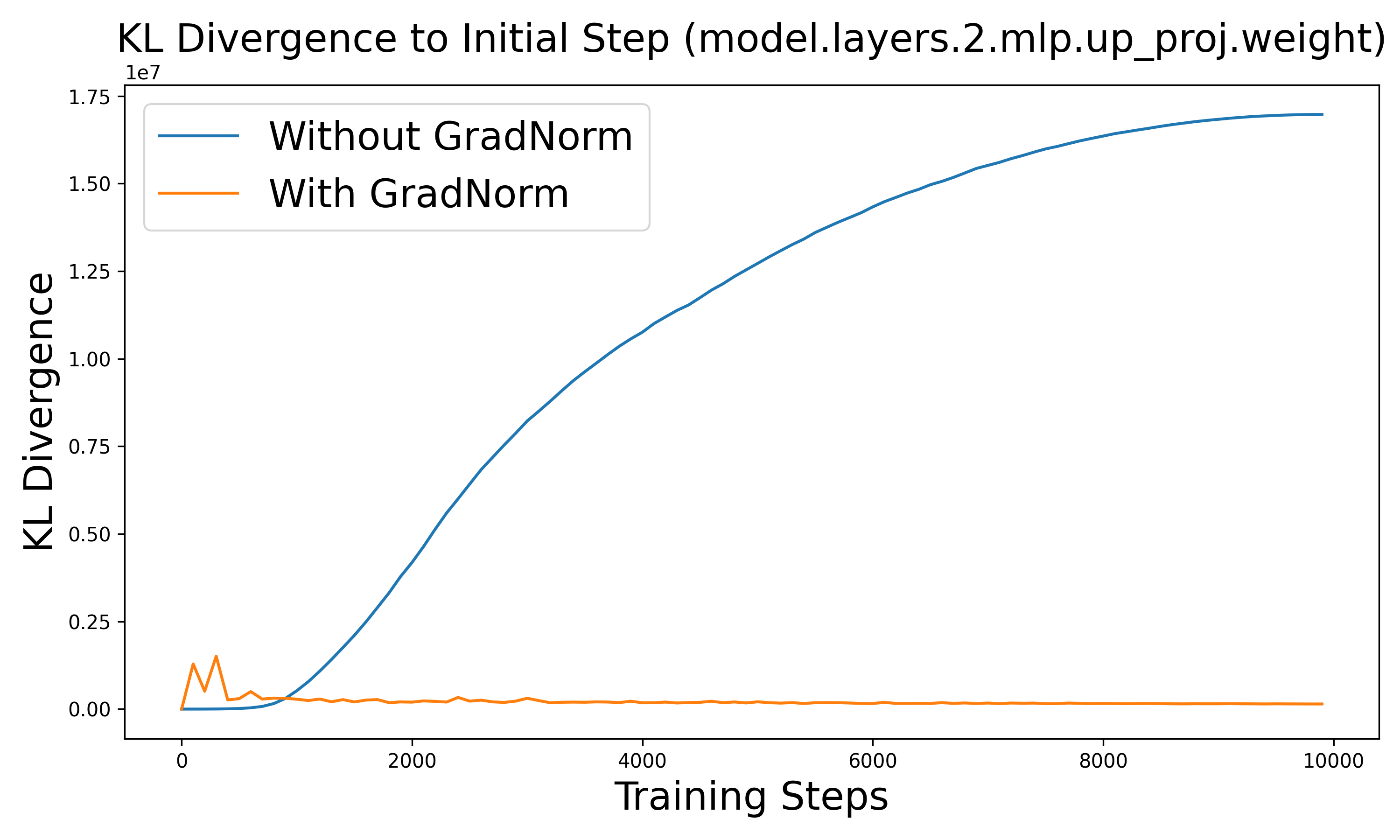} 
        \caption{}
        \label{fig:sub1}
    \end{subfigure}
    \hfill
    \begin{subfigure}[b]{0.49\textwidth}
        \centering
        \includegraphics[width=\textwidth]{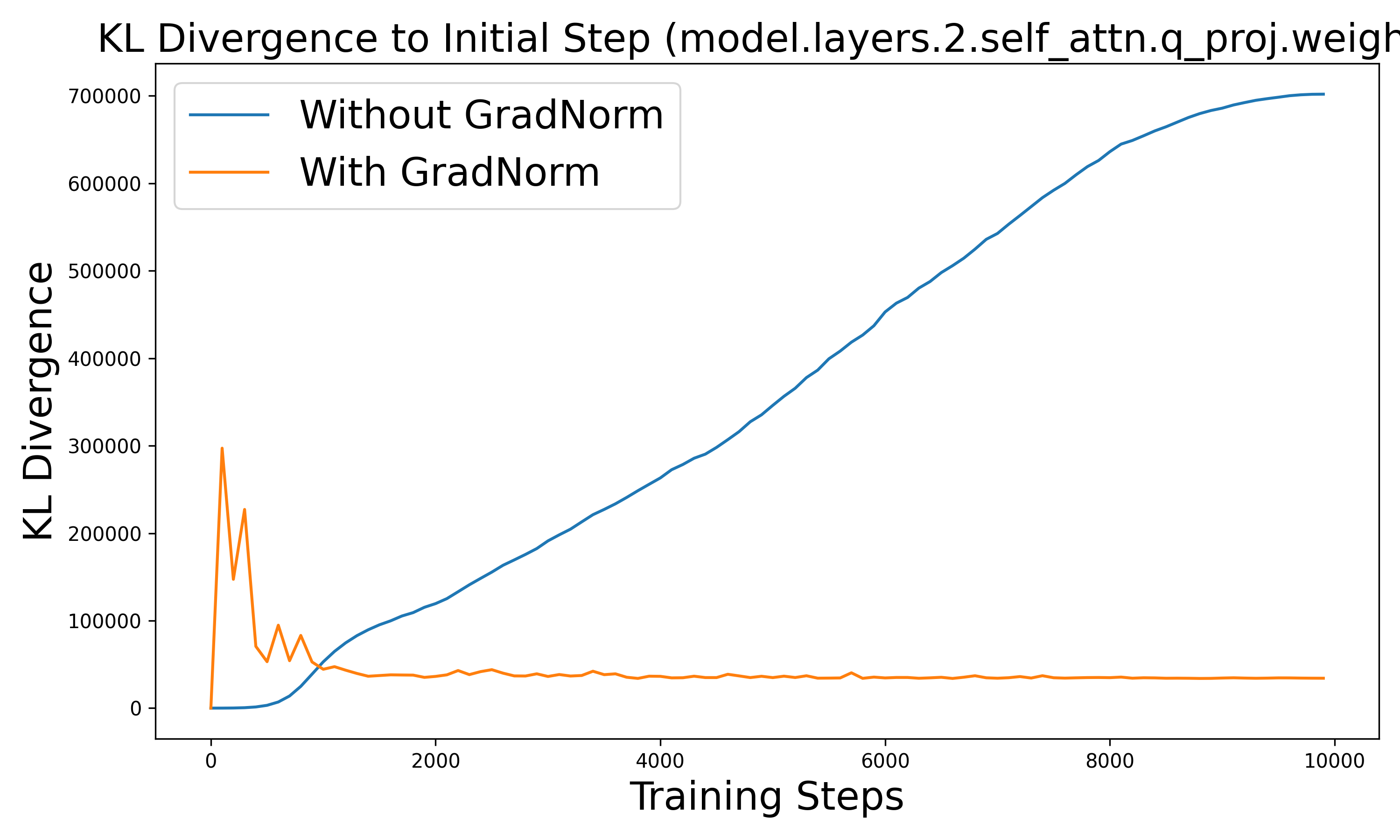} 
        \caption{}
        \label{fig:sub2}
    \end{subfigure}
    \caption{ KL divergence comparison of gradient distributions against initial gradient distribution across training Steps. We use the projection weights in attention and MLP modules of the second layer as an example. The plots compare standard training with $\gradnorm$-augmented training. Lower KL divergence values indicate greater stability in gradient distributions. 
    }
    \label{fig: gradnorm_toy}
\end{figure}

To examine whether $\gradnorm$ stabilizes the distribution of the stochastic gradients as suggested by \Cref{thm: stability}, we conduct controlled experiments using a scaled-down LLaMA-based model (about 10 million parameters)~\citep{Lialin2023ReLoRAHT}, trained on the C4 dataset. Our goal is to measure how $\gradnorm$ affects the distribution of stochastic gradients over multiple training steps.
Specifically, we employ a small-scale LLaMA-based model with approximately 10 million parameters ~\citep{Lialin2023ReLoRAHT}. Training is conducted on the C4 dataset~\citep{2020t5}.

\paragraph{Baselines} We compare:
\begin{itemize}
    \item \textbf{Standard training}: This uses an SGD optimizer with a learning rate of $5 \times 10^{-4}$ and a linear learning rate scheduler, including a 10\% warm-up of total training steps (10,000 steps).
    \item \textbf{$\gradnorm$-processed training}: This applies $\gradnorm$ to pre-process the stochastic gradient before the parameter update. All other settings match the standard training baseline.
\end{itemize}

 \paragraph{Methodology} We measure gradient statistics in the presence of mini-batch noise. At step $t=0$, we sample 16 additional mini-batches (batch size 64 each) and compute the mean and standard deviation of the corresponding raw or $\gradnorm$ gradients in each batch, and obtain the approximated initial gradient distribution.
 After each training step of baseline methods, we perform the same procedure and calculate the Kullback-Leibler divergence between the resulting gradient distributions and the initial gradient distributions. This process tracks how the gradient distribution changes over time.

\paragraph{Results} \Cref{fig: gradnorm_toy} shows the KL divergence of gradient distributions for standard and $\gradnorm$-augmented training, relative to the corresponding initial approximated distributions. Apart from early spikes, $\gradnorm$ reduces fluctuations in the gradient distribution throughout training. 

\subsubsection{Does $\gradwhite$ Counteracts Local Curvature and Provide Fast Convergence on ill-conditioned problems?} \label{sec: whitening_exp}

This subsection evaluates the optimization performance of gradient descent when combined with $\gradwhite$. We use three classic problem settings:

\begin{figure}[t!]
    \centering
    \begin{subfigure}[b]{0.49\textwidth}
        \centering
        \includegraphics[width=\textwidth]{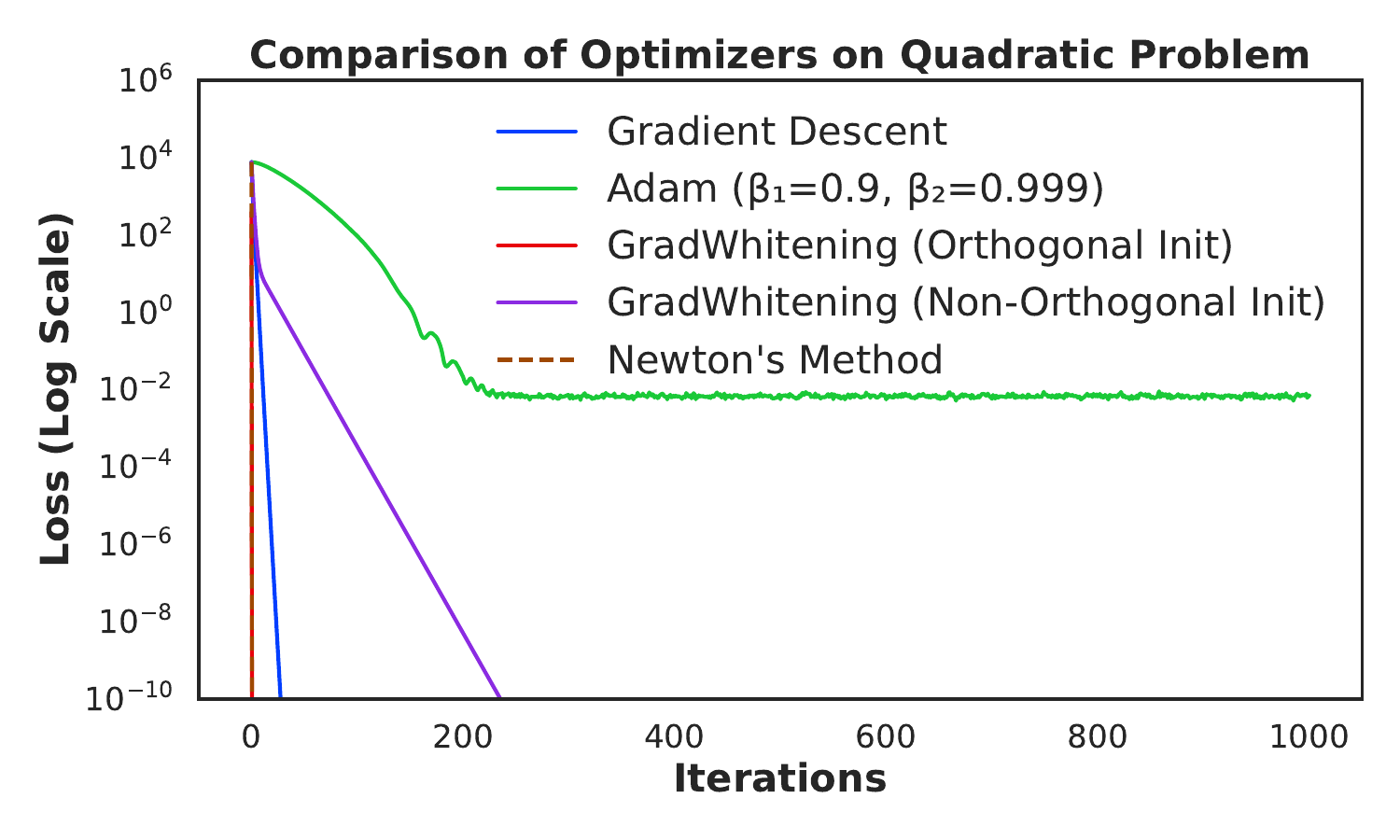} 
        \caption{Quadratic optimization, well-conditioned}
    \end{subfigure}
    \hfill
    \begin{subfigure}[b]{0.49\textwidth}
        \centering
        \includegraphics[width=\textwidth]{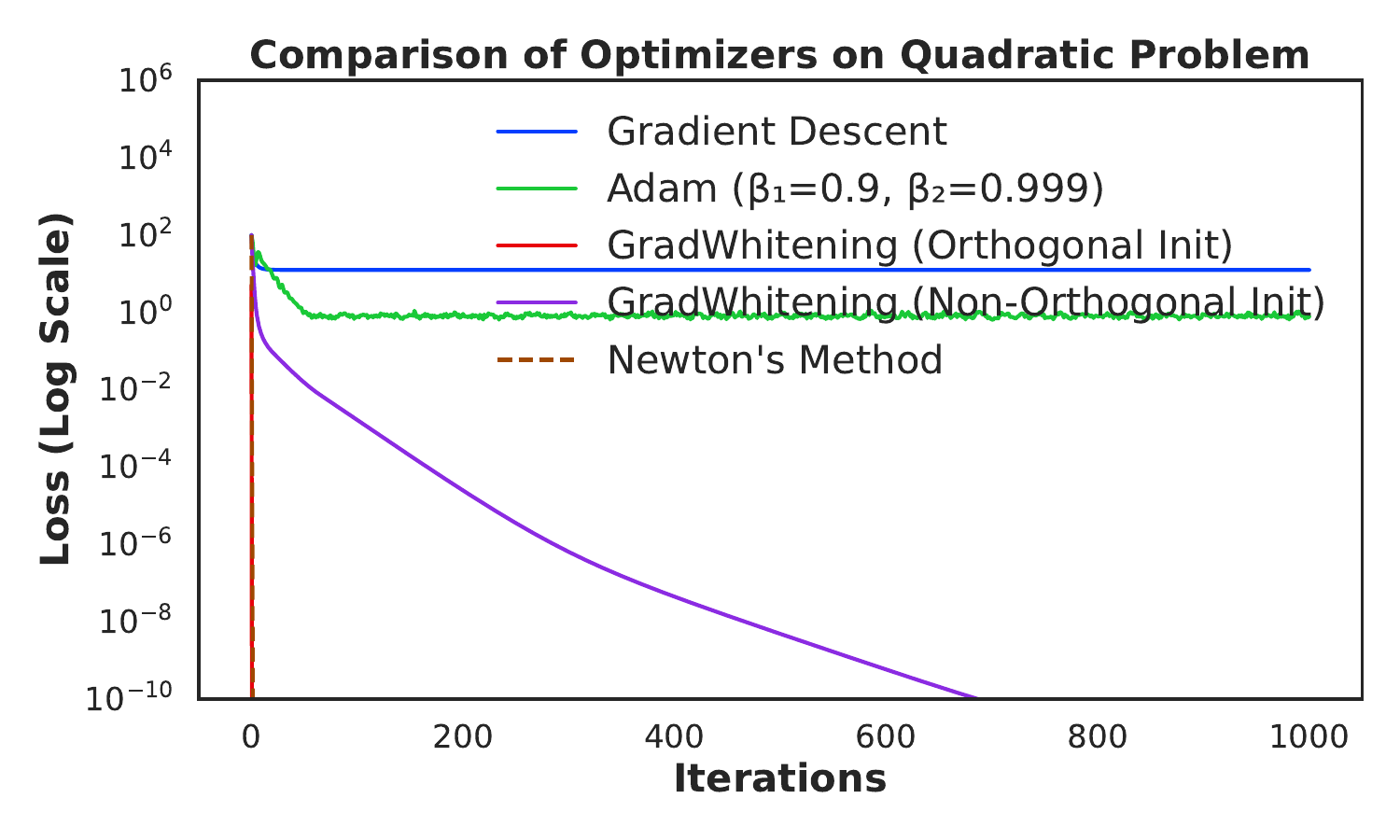} 
        \caption{Quadratic optimization, ill-conditioned}
    \end{subfigure}
    \hfill
    \vskip\baselineskip
    \begin{subfigure}[b]{0.49\textwidth}
        \centering
        \includegraphics[width=\textwidth]{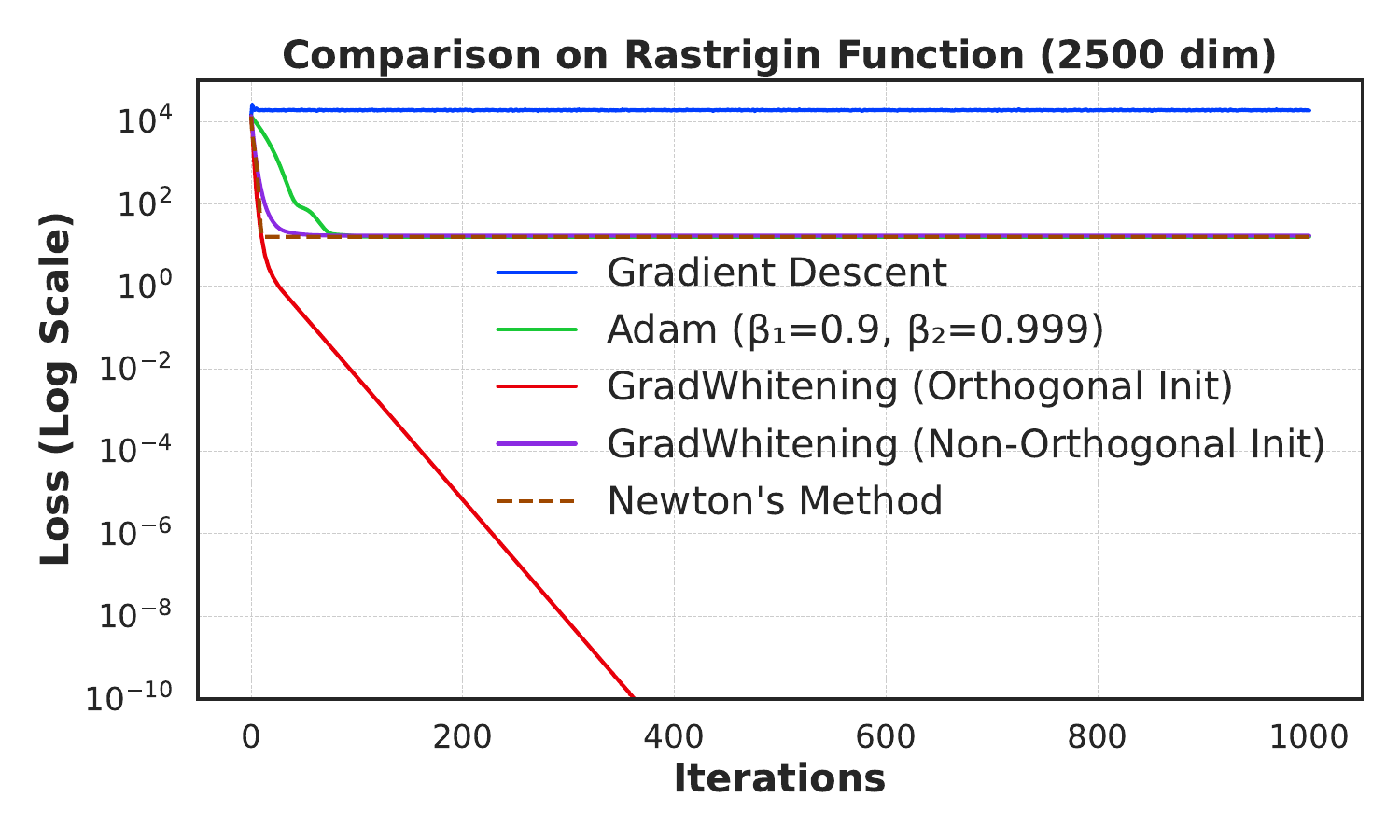} 
        \caption{Rastrigin function optimization}
    \end{subfigure}
    \hfill
    \begin{subfigure}[b]{0.49\textwidth}
        \centering
        \includegraphics[width=\textwidth]{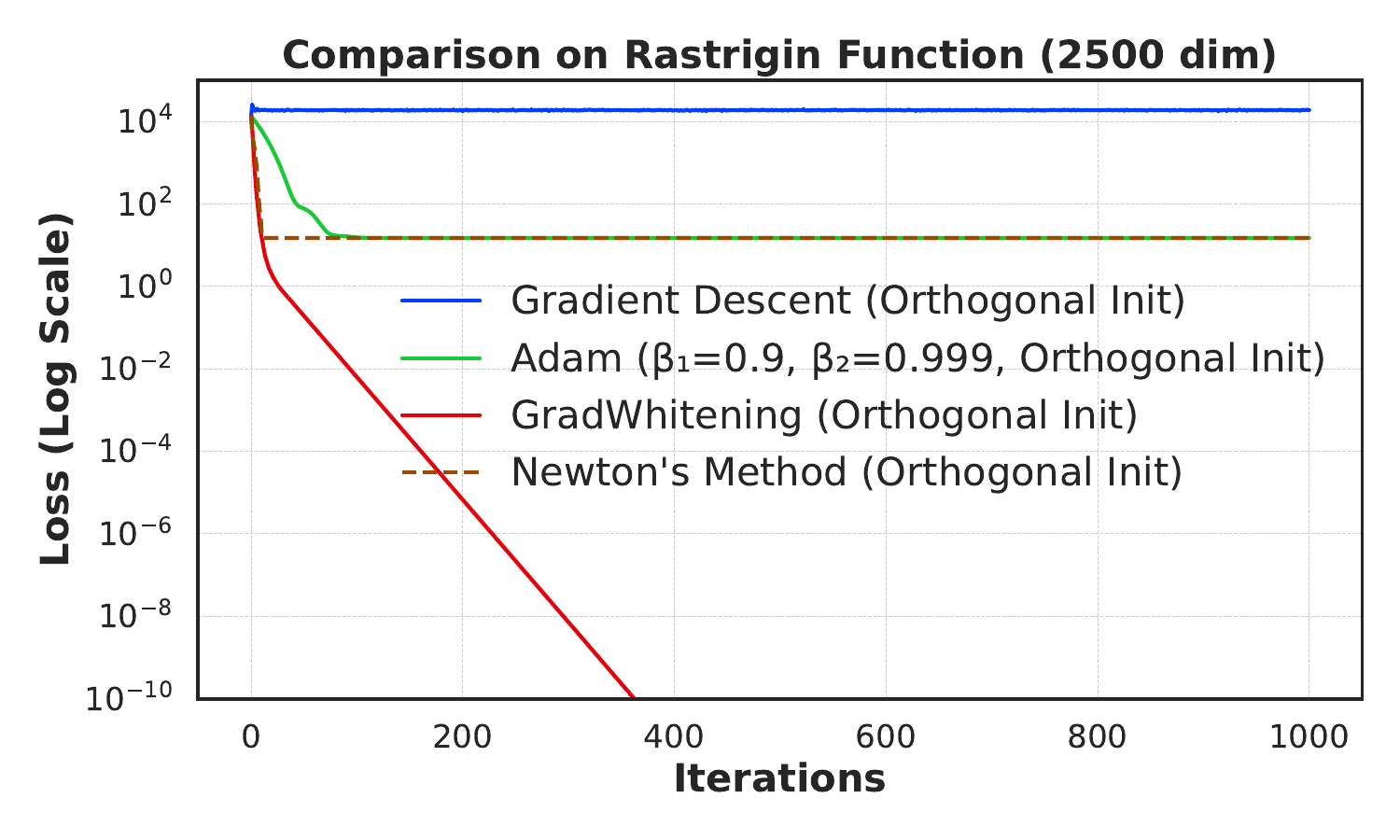} 
        \caption{Effect of initialization}
    \end{subfigure}
    \caption{ Comparison of convergence rate of different methods on quadratic and non-convex optimization problems. (a): 2500-dimensional quadratic optimization with well-conditioned $\rmH$. (b): 2500-dimensional quadratic optimization with ill-conditioned $\rmH$ (c): 2500-dimensional Rastrigin function optimization. (d): 2500-dimensional Rastrigin function optimization, but forcing all methods to use the same orthogonal initial location.}
    \label{fig: gradwhite_toy}
\end{figure}

\begin{itemize}
    \item \textbf{High-dimensional quadratic optimization.} A quadratic problem of the form in \Cref{eq: quadratic}, where $\rmW \in \mathbb{R}^{50 \times 50}$. 
    \item \textbf{Ill-conditioned quadratic optimization.} Same setup as above, but with a deliberately chosen ill-conditioned $\rmH$. 
    \item \textbf{Non-convex optimization with multiple local optima.} We use the multivariate Rastrigin function:
    \[
      f(\rmW) = m^2 \rmA + \frac{1}{2}\text{Tr}[\rmW^\top \rmW] - \rmA \sum_{ij} \text{cos}(2\pi W_{ij}),
    \]
    where $\rmW$ is an $m \times m$ matrix and $m = 50$. This function has $10^{m^2}$ possible local optima.
\end{itemize}

\paragraph{Baselines} We compare five methods on all three problems: gradient descent (GD) with the theoretical optimal learning rate in \Cref{cor_gd_lower_bd_general}, Adam with $\beta_1 = 0.9$, $\beta_2 = 0.999$ and a hand-tuned learning rate, Newton's method with a tuned learning rate, and two $\gradwhite$-based variants (with and without orthogonal initialization). This is to verify, under orthogonal initialization, $\gradwhite$-processed GD behaves similarly to Newton’s method, as discussed in \Cref{sec: gradwhite_discuss} and \Cref{thm_sve_optimal}. All methods share the same initialization, except for the orthogonal $\gradwhite$ variant, which projects the initial parameters onto an orthogonal matrix.

\paragraph{Results} From \Cref{fig: gradwhite_toy} (a)--(c), we summarizes the following outcomes:

\begin{itemize}
    \item \textbf{Quadratic problems (\Cref{fig: gradwhite_toy} (a) and (b)).} 
    $\gradwhite$ with orthogonal initialization and Newton's method converge to optimum in one step, aligning with the theoretical predictions in \Cref{sec: gradwhite_discuss} and \Cref{thm_sve_optimal}.
    
    \item \textbf{Well- vs.\ ill-conditioned cases (\Cref{fig: gradwhite_toy} (a) and (b)).} 
    In the well-conditioned setting (a), standard GD outperforms both Adam and $\gradwhite$ (non-orthogonal initialization). In the ill-conditioned setting (b), $\gradwhite$ (non-orthogonal initialization) outperforms GD by a large margin, while GD experiences slow convergence.
     
    \item \textbf{Comparison with Adam (\Cref{fig: gradwhite_toy} (a)--(c)).}
    In all three settings, $\gradwhite$ with non-orthogonal initialization consistently outperforms Adam, consistent with \Cref{thm_sve_dam_group}.
   
    \item \textbf{Rastrigin function (\Cref{fig: gradwhite_toy} (c)).}
    On this non-convex problem, $\gradwhite$ performs comparably to Newton's method, with or without orthogonal initialization. 

    \item \textbf{Effect of initialization (\Cref{fig: gradwhite_toy} (d))} Furthermore, we force all methods to share the same orthogonalized initialization. As shown by the result, $\gradwhite$ GD with orthogonal initialization still consistently outperforms all baselines, confirming that this initialization is only beneficial to $\gradwhite$ GD among all baselines. 
\end{itemize}

\section{Implementation details}\label{app: code}


\paragraph{General setup} We describe the implementation setups for \name{} used in LLM pre-training tasks. To enable a more straightforward and comparable analysis, we simply replicate the setting of \cite{Zhao2024GaLoreML}, under exactly the same model configs and optimizer hyperparameter configs, whenever possible. This includes the same model architecture, tokenizer, batch size, context length, learning rate scheduler, learning rates, subspace scaling, etc.  

\paragraph{Precision} All baselines uses BF16 for model weights, gradients, and optimizer states storage. For all \name{} variants, we use BF16 for model weights and gradients. For the $\gradwhite$ step of \name{}-0 and \name{}$\dag$ we use FP32 to whiten the BF16 gradients and then convert it back to BF16. We found that this helps to improve trainining stability and performance. However, for the NSDS scheme of \name{}$\ddag$ we oberve that FP32 does not offer performance boost over BF16 (\Cref{app: precision_ablation}), therefore we stick with BF16.

\paragraph{Learning rate scheduling} we use exactly the same scheduler as in \cite{Zhao2024GaLoreML}, with the exception of \name{}-0, which does not require any learning rate warmup. Therefore, for \name{}-0, we directly start with maximum learning rate, and enter the learning rate decay phase, using the same decay parameters as \cite{Zhao2024GaLoreML}.

\paragraph{Reproducing baseline results} Most baseline results are cited from respective papers \citep{Zhao2024GaLoreML, zhu2024apollo} as we share the exact same setup. We also tried to reproduce their results using the same opensourced code, and generally obtain slightly worse results for Galore, Apollo and Adam for larger models (350M and 1B, see \Cref{tab: reproduced}). Therefore in the main paper we only compare with the official results. For the reproduced results of Adam, we specify the details below as it was not disclosed in \cite{Zhao2024GaLoreML}.  We use same learning rate tuning procedure as suggested by \cite{Zhao2024GaLoreML} (i.e., performing grid search over $\{0.01, 0.005, 0.001, 0.0005, 0.0001\}$). We found that the optimal learning rates for Adam is 0.001. The only exception is that for a model of size 1.3B: as we already know that a larger model requires smaller learning rates, we conduct a learning search for Adam over a smaller but more fine-grained grid of $\{ 0.001, 0.0007, 0.0005, 0.0003, 0.0001\}$. As a result, the optimal learning rate found for Adam on 1.3B is 0.0007. Finally, one baseline that does not exist in the literature is the Momentum + $\gradwhite$ (Muon-like optimizer without Nestrov acceleration) baseline; and we report our own results. We start from the default learning rates used by Muon \cite{jordan2024muon} and tuned them over a grid of ${0.01, 0.02, 0.03, 0.04, 0.05}$.

\begin{table*}[t]
    \centering
    \caption{\small{Reproduced results.  } }
    \label{tab: reproduced}
    \begin{tabular}{lcccc}
    \toprule
                & \textbf{130M} & \textbf{350M} & \textbf{1.3 B} \\
    \midrule
    Adam (reproduced) & 24.44 (0.75G) & 19.24 (2.05G) & 16.44 (7.48G) \\
    Apollo-mini (reproduced)  & 23.97 (0.43G) & 17.60 (0.93G) & 14.37 (2.98G) \\ 
    Galore (reproduced) & {24.67} (0.57G) & {19.74} (1.29G) & {15.89} (4.43G) \\
    \bottomrule
    $r$ of low-rank methods  & 256 & 256 & 512 \\
    Training Steps  & 20K & 60K & 100K \\ %
    \bottomrule
    \end{tabular}
\end{table*}

\paragraph{\name{} Settings and hyperparameters} Since \name{} utilizes matrix-level operations on gradients, it can only be applied to 2D parameters. Therefore, in our experiments, we only apply \name{} on all linear projection weights in transformer blocks. Similar to Galore \citep{Zhao2024GaLoreML}, the rest of the non-linear parameters still uses Adam as the default choice. Therefore, we follow the learning rate setup of Galore, where we fix some global learning rate across all model sizes and all modules. Then, for the linear projection modules where \name{} is applied, we simply apply a scaling factor $\alpha$ on top of the global learning rate.  For all \name{} variants, we adopt a \emph{lazy-tuning approach} (hyperparameters are set without extensive search), as detailed below. This helps to reduce the possibility of unfair performance distortion due to excessive tuning.   

\begin{itemize}
    \item \textbf{\name{}-0} uses naive NS-iteration for whitening, disabled learning rate warmup, and use similar learning rates optimized for Adam. We fix the global learning rate to be the same as Adam, and fix $\alpha=1$.  The only exception is the 1.3 B case. This is because we observe that the optimal learning rate of Adam under 1.3B becomes smaller than 0.001, hence we also reduce the learning rate on \name{}, where we used $\alpha = 0.3$, resulting an effective learning rate of $0.0003$.  To summarize the hyperparameter of \textbf{\name{}-0} is set to be similar to Adam, without any tuning. This is to demonstrate the robustness of \name{} series optimizers and their capability to work out-of-the-box as a replacement for Adam.
    \item \textbf{\name{}}$^\dag$, is the vanilla version of our method, in which we enabled learning rate warmup, and allowed the use of optimized learning rates that largely differ from Adam. We notice that \name{} allows larger learning rates than Adam. We use a global learning rate of 0.02, as well as the scaling factor $\alpha = 0.05$. This is selected by simply searching the learning rate over a constraint grid $\{0.01, 0.02, 0.05\}$, and then setting $\alpha = 0.05$ such that the effective learning rate is scaled back to 0.001. There is no guarantee that this heuristic rule is optimal; but we found that this usually does not make a run fail (e.g., with loss divergence).

    \item Finally, \textbf{\name{}$^\ddag$}, the most efficient version of \name{} that employs the proposed NSDS scheme for fast whitening (\cref{sec: practical}). Similar to \textbf{\name{}}$^\dag$, we use the same global learning rate of 0.02, as well as the scaling factor $\alpha = 0.05$ across all model sizes. We suspect with more careful tuning, its performance can be significantly improved; however, this is out of the scope of the paper.
\end{itemize}

The configurations of $\gradwhite$ is discussed next.


\paragraph{Implementation of $\gradwhite$} For $\gradwhite$, before N-S iteration, we further normalize its input matrix by its Frobenius norm. This is applied in all our ablation studies as well, regardless of whether the gradient is processed by $\gradnorm$. For our proposed Newton-Schulz with Diagonal Substitution (NSDS) used in \name$^\ddag$, we only run it for 2 steps across all model sizes. We found that using a step size $\beta \neq 0.5$ in the $\gradwhite$ operator (\Cref{alg:optimizer2}) can improve its convergence. We set $\beta = 0.4$. We found that NSDS is generally robust to $\beta$ as long as it is not too large; and our specific choice of parameters usually already gives satisfactory performance in LLM pretraining. For the naive N-S iteration used in \name{}-0 and \name{}$^\dag$, we run 10 steps which is usually sufficient. We set $\beta = 0.8$. The naive NS is run in FP32 precision, on top of BF16 gradients and weights; while the NSDS is run in BF16.

\paragraph{Computational Overhead.} Below, we discuss the computational overhead of Naive Newton-Schulz. In practice, we only run the naive Newton-Schulz for $\leq$ 10 iterations, which corresponds to $\leq$ 50  matrix multiplications. These matrix multiplications are in general GPU friendly, hence for the task of training LLMs, the batch size is the more dominant factor for compute. For example QWen 14B \citep{bai2023qwen} has 4M batch size vs a a model dimension $d_{\text{model}}$  = 5120, DeepSeek \citep{Bi2024DeepSeekLS} 67B has 6M batch size vs $d_{\text{model}}$ = 8192, and LLama 3 \citep{llama3} 405B has 4-16M batch size vs $d_{\text{model}}$ = 16384. To estimate the computational overhead,  assuming the N-S iteration involves approximately $ 50 \times d_{\text{model}}^3 $ FLOPs. In contrast, the primary training cost scales with the batch size and is proportional to $ \text{batch\_size} \times d_{\text{model}}^2 $ FLOPs. In those examples, the estimated computational overhead of Newton–Schulz is typically below $ \leq 7\%$. A similar estimation has been given in \citep{jackson2023isometric} ($\leq 5 \%$), as well as in \citep{jordan2024muon} ($\leq 1 \%$). 

However, note that whether the above analysis hold for large scale, distributed LLM training is still unclear. This is due to a few factors. First, the N-S iteration is after all a $\mathcal{O}(m^3$ complexity operation and might need to scale as model size increases. Second, the distributed computation of N-S steps needs extra non-trivial effort, which brings additional infrastructure challenges. This motivates us to propose the NSDS scheme, which enables Adam-level throughput without performance compromise.

\end{document}